\documentclass{article}
\usepackage{geometry}
\usepackage{amsmath, amsfonts, bm}
\usepackage{amsthm,graphicx}
\usepackage[dvipsnames]{xcolor}
\usepackage{mathtools}

\usepackage[unicode=true,
 bookmarks=false,
 breaklinks=false,pdfborder={0 0 1},colorlinks=false]
 {hyperref}
\hypersetup{
 colorlinks,citecolor=blue,filecolor=blue,linkcolor=blue,urlcolor=blue}

\allowdisplaybreaks

\usepackage{natbib}
\usepackage{float}
\usepackage{multirow}
\usepackage{footnote}
\usepackage{dsfont}
\usepackage{booktabs}

\usepackage[linesnumbered,ruled,vlined]{algorithm2e}
\usepackage{algorithmic}

\definecolor{yxc}{RGB}{255,0,0}
\definecolor{yjc}{RGB}{190,0,255}
\definecolor{dacong}{RGB}{10,103,68}
\definecolor{cc}{RGB}{1,11,111}

 \usepackage[normalem]{ulem}

\usepackage{dsfont}
\DeclareMathOperator{\ind}{\mathds{1}}

\newcommand{\Exs}{\mathbb{E}}

\newcommand{\soft}[1]{{#1}_{\tau}}
\newcommand{\prob}{\mathop{{P}}}
\newcommand{\ex}[2]{\mathbb{E}_{#1}\left[#2\right]}
\newcommand{\exlim}[2]{\mathop\mathbb{E}\limits_{#1}\left[#2\right]}
\newcommand{\exlimbig}[2]{\mathop\mathbb{E}\limits_{#1}\big[#2\big]}
\newcommand{\exlimBig}[2]{\mathop\mathbb{E}\limits_{#1}\Big[#2\Big]}

\newcommand{\V}{V} 
\newcommand{\A}{A} 
\newcommand{\Q}{Q} 
\newcommand{\plcy}{\pi} 
\newcommand{\prm}{\theta} 
\newcommand{\s}{s} 
\newcommand{\ac}{a} 
\newcommand{\ssp}{\mathcal{S}} 
\newcommand{\asp}{\mathcal{A}} 
\newcommand{\real}{\mathbb{R}} 
\newcommand{\disct}{\gamma} 
\newcommand{\prn}[1]{\left({#1}\right)} 
\newcommand{\prnbig}[1]{\big({#1}\big)} 
\newcommand{\prnBig}[1]{\Big({#1}\Big)} 
\newcommand{\prnBigg}[1]{\Bigg({#1}\Bigg)} 
\newcommand{\brk}[1]{\left[{#1}\right]} 
\newcommand{\norm}[1]{\left\|{#1}\right\|} 
\newcommand{\normbig}[1]{\big\|{#1}\big\|} 
\newcommand{\abs}[1]{\left|{#1}\right|} 

\newcommand{\cS}{\mathcal{S}}
\newcommand{\cA}{\mathcal{A}}

\newcommand{\KL}{\mathsf{KL}}

\newcommand{\sP}{{P}} 
\newcommand{\idn}{{I}} 

\newcommand{\estsoftQ}{\widehat{Q}_\tau^{(t)}}
\newcommand{\exctplcy}[1]{{\breve{\plcy}}^{(#1)}}

\newtheorem{theorem}{Theorem}
\newtheorem{lemma}{Lemma}
\newtheorem{prop}{Proposition}
\newtheorem{remark}{Remark}

\title{Fast Global Convergence of Natural Policy Gradient Methods \\ with Entropy Regularization}
\author{Shicong Cen\thanks{Department of Electrical and Computer Engineering, Carnegie Mellon University; email: \texttt{shicongc@andrew.cmu.edu}. } \\
CMU    \\
	\and
	Chen Cheng\thanks{Department of Statistics, Stanford University; email: \texttt{chencheng@stanford.edu}.  } \\
Stanford    \\
	\and
	Yuxin Chen\thanks{Department of Electrical and Computer Engineering, Princeton University; email: \texttt{yuxin.chen@princeton.edu}. } \\
 Princeton  \\
 	\and
	Yuting Wei\thanks{Department of Statistics and Data Science, Carnegie Mellon University; email: \texttt{ytwei@cmu.edu}.}\\
	CMU\\
	\and
	Yuejie Chi\thanks{Department of Electrical and Computer Engineering, Carnegie Mellon University; email: \texttt{yuejiechi@cmu.edu}. }\\
	CMU\\
	}

\date{July 13, 2020; \quad Revised \today}

\begin{document}
	\maketitle
	
	\begin{abstract}

Natural policy gradient (NPG) methods  are among the most widely used policy optimization algorithms in contemporary reinforcement learning. 
This class of methods is often applied in conjunction with entropy regularization --- an algorithmic scheme that encourages exploration --- and is closely related to soft policy iteration and trust region policy optimization.  Despite the empirical success, the theoretical underpinnings for NPG methods remain limited even for the tabular setting.

This paper develops {\em non-asymptotic} convergence guarantees for entropy-regularized NPG methods under softmax parameterization, focusing on discounted Markov decision processes (MDPs). Assuming access to exact policy evaluation, we demonstrate that the algorithm converges linearly --- even quadratically once it enters a local region around the optimal policy --- when computing optimal value functions of the regularized MDP. 
Moreover, the algorithm is provably stable vis-\`a-vis inexactness of policy evaluation. Our convergence results accommodate a wide range of learning rates, and shed light upon the role of entropy regularization in enabling fast convergence. 
 

\end{abstract}


	\tableofcontents

	\section{Introduction}

Policy gradient (PG) methods and their variants \citep{williams1992simple,sutton2000policy,kakade2002natural,peters2008natural,konda2000actor},  which aim to optimize (parameterized) policies via gradient-type methods, lie at the heart of recent advances in reinforcement learning (RL) (e.g.~\citet{mnih2015human,schulman2015trust,silver2016mastering,schulman2017proximal}).  
Perhaps most appealing is their flexibility in adopting various kinds of policy parameterizations (e.g.~a class of policies parameterized via deep neural networks),  which makes them remarkably powerful and versatile in contemporary RL. 

As an important and widely used extension of PG methods, {\em natural policy gradient} (NPG) methods propose to employ natural policy gradients \citep{amari1998natural} as search directions, in order to achieve faster convergence than the update rules based on  policy gradients \citep{kakade2002natural,peters2008natural,bhatnagar2009natural,even2009online}. Informally speaking, NPG methods precondition the gradient directions by Fisher information matrices (which are the Hessians of a certain divergence metric), and fall under the category of quasi second-order policy optimization methods. In fact, a variety of mainstream RL algorithms, such as {\em trust region policy optimization} (TRPO) \citep{schulman2015trust} and {\em proximal policy optimization} (PPO) \citep{schulman2017proximal}, can be viewed as generalizations of NPG methods \citep{shani2019adaptive}. In this paper, we pursue in-depth theoretical understanding about this popular class of methods --- in conjunction with entropy regularization to be introduced momentarily.

\subsection{Background and motivation}

Despite the enormous empirical success, the theoretical underpinnings of policy gradient type methods have been limited even until recently, 
primarily due to the intrinsic non-concavity underlying the value maximization problem of interest \citep{bhandari2019global,agarwal2019optimality}. To further exacerbate the situation, an abundance of problem instances contain suboptimal policies residing in regions  with flat curvatures (namely, vanishingly small gradients and high-order derivatives) \citep{agarwal2019optimality}. Such plateaus in the optimization landscape could, in principle, be difficult to escape once entered, thereby necessitating a higher degree of exploration in order to accelerate policy optimization.

In practice, a strategy that has been frequently adopted to encourage exploration and improve convergence is to enforce entropy regularization \citep{williams1991function,peters2010relative,mnih2016asynchronous,duan2016benchmarking,haarnoja2017reinforcement,hazan2019provably,vieillard2020leverage,xiao2019maximum}. By inserting an additional penalty term to the objective function, this strategy penalizes policies that are not stochastic/exploratory enough, 
in the hope of preventing a policy optimization algorithm from being trapped in an undesired local region. Through empirical visualization, \citet{ahmed2019understanding} suggested that entropy regularization induces a smoother landscape that allows for the use of larger learning rates, and hence, faster convergence.   However, the theoretical support for regularization-based policy optimization remains highly inadequate.

Motivated by this, a very recent line of works set out to elucidate, in a theoretically sound manner, the efficiency of entropy-regularized policy gradient methods. Assuming access to exact policy gradients, \citet{agarwal2019optimality} and \citet{mei2020global} developed convergence guarantees for regularized PG methods (with relative entropy regularization considered in \citet{agarwal2019optimality} and entropy regularization in \citet{mei2020global}). Encouragingly, both papers suggested the positive role of regularization in guaranteeing faster convergence for the tabular setting. However, these works fell short of explaining the role of entropy regularization for other policy optimization algorithms like NPG methods, which we seek to understand in this paper.

\subsection{This paper}

Inspired by recent theoretical progress towards understanding PG methods \citep{agarwal2019optimality,bhandari2019global,mei2020global},
we aim to develop non-asymptotic convergence guarantees for entropy-regularized NPG methods in conjunction with softmax parameterization. We focus attention on studying tabular discounted Markov decision processes (MDPs), which is an important first step and a stepping stone towards demystifying the effectiveness of entropy-regularized policy optimization in more complex settings.

\paragraph{Settings.} Consider a $\gamma$-discounted infinite-horizon MDP with state space $\cS$ and action space $\cA$. 
Assuming availability of exact policy evaluation, the update rule of entropy-regularized NPG methods with softmax parameterization admits a simple update rule in the policy space (see Section~\ref{sec:models} for precise descriptions)
\begin{equation}
\label{eqn:NPG-regularized}
	\plcy^{(t+1)}(\ac|\s) ~\propto~
		 \big( \plcy^{(t)}(\ac|\s) \big)^{1- \frac{\eta\tau}{1-\disct}}\exp\Big( \frac{\eta{\soft{\Q}^{\plcy^{(t)}}(\s, \ac)}} { {1-\disct}} \Big)
\end{equation}
for any $(s,a)\in \cS\times \cA$, 
where $\tau>0$ is the regularization parameter, $0<\eta \leq \frac{1-\gamma}{\tau}$ is the learning rate (or stepsize), $\plcy^{(t)}$ indicates the $t$-th policy iterate, and $\soft{\Q}^{\plcy}$ is the soft Q-function under policy $\plcy$ (to be defined in \eqref{eq:defn-regularized-Q}). The update rule \eqref{eqn:NPG-regularized} is closely connected to several popular algorithms in practice. For instance, the {\em trust region policy optimization} (TRPO) algorithm \citep{schulman2015trust}, when instantiated in the tabular setting, can be viewed as implementing \eqref{eqn:NPG-regularized} with line search. In addition, by setting the learning rate as $\eta = \frac{1-\gamma}{\tau}$, the update rule \eqref{eqn:NPG-regularized} coincides with  {\em soft policy iteration} (SPI)  studied in \citet{haarnoja2017reinforcement}.

\paragraph{Our contributions.} The results of this paper deliver fully non-asymptotic convergence rates of entropy-regularized NPG methods without any hidden constants, which are previewed as follows (in an orderwise manner). The definition of $\epsilon$-optimality can be found in Table~\ref{tab:comparisons}.

\begin{itemize}
	\item
	\textbf{Linear convergence of exact entropy-regularized NPG methods.} We establish linear convergence of entropy-regularized NPG methods for finding the optimal policy of the entropy-regularized MDP, assuming access to exact policy evaluation. To yield an $\epsilon$-optimal policy for the regularized MDP (cf.~Table~\ref{tab:comparisons}), the algorithm \eqref{eqn:NPG-regularized} with a general learning rate $0< \eta \leq \frac{1-\gamma}{\tau}$ needs no more than an order of
\[
	\frac{1}{\eta\tau}  \log\left(\frac{1}{\epsilon}\right)   
\]
		iterations, where we hide the dependencies that are logarithmic on salient problem parameters (see Theorem~\ref{thm:npg_exact}). Some highlights of our convergence results are (i) their near dimension-free feature and (ii) their applicability to a wide range of learning rates (including small learning rates). 

\item 
\textbf{Linear convergence of approximate entropy-regularized NPG methods.}
We demonstrate the stability of the regularized NPG method with a general learning rate $0< \eta \leq \frac{1-\gamma}{\tau}$ even when the soft Q-functions of interest are only available approximately. This paves the way for future investigations that involve finite-sample analysis. Informally speaking, the algorithm exhibits the same convergence behavior as in the exact gradient case before an error floor is hit, where the error floor scales linearly in the entrywise error of the soft Q-function estimates (see Theorem~\ref{thm:npg_inexact}).

\item 
	\textbf{Quadratic convergence in the small-$\epsilon$ regime.} In the high-accuracy regime where the target level $\epsilon$ is {\em very small}, the algorithm \eqref{eqn:NPG-regularized} with $\eta = \frac{1-\gamma}{\tau}$ converges super-linearly, in the sense that the iteration complexity to reach $\epsilon$-accuracy for the regularized MDP is at most on the order of 
\begin{equation*}
	 \log  \log \left( \frac{1}{\epsilon} \right) ,
\end{equation*}
{\em after} entering a small local neighborhood surrounding the optimal policy. Here, we again hide the dependencies that are logarithmic on salient problem parameters (see Theorem~\ref{thm:spi_sp}).

\end{itemize}
%

\paragraph{Comparisons with prior art.} 
\citet{agarwal2019optimality} proved that unregularized NPG methods with softmax parameterization attain an $\epsilon$-accuracy within $O(1/\epsilon)$ iterations. In contrast, our results assert that $O(\log(1/\epsilon))$ iterations suffice with the assistance of entropy regularization, which hints at the potential benefit of entropy regularization in accelerating the convergence of NPG methods.   Shortly after the initial posting of our paper, \citet{bhandari2020note} posted a note that proves linear convergence of unregularized NPG methods with exact line search, by exploiting a clever connection to policy iteration.  Their convergence rate is governed by a quantity $\min_{s\in \cS}\rho(s)$, resulting in an iteration complexity at least $|\cS|$ times larger than ours. 
In comparison, our results cover a broad range of fixed learning rates (including small stepsizes that are of particular interest in practice), and accommodate the scenario with inexact gradient evaluation.   See Table~\ref{tab:comparisons} for a  quantitative comparison. Moreover, we note that the entropy-regularized NPG method with general learning rates is closely related to TRPO in the tabular setting (see \citet{shani2019adaptive}).
The recent work \citet{shani2019adaptive} demonstrated that TRPO converges with an iteration complexity $O(1/\epsilon)$ in entropy-regularized MDPs. The analysis therein is inspired by the mirror descent theory in generic optimization literature, which characterizes sublinear convergence  under properly decaying stepsizes and accommodates various choices of divergence metrics. 
		In comparison, our analysis strengthens the performance guarantees by carefully exploiting properties specific to the current version of the NPG method. In particular, we identify the delicate interplay between the crucial operational quantities $Q_\tau^\star - Q_\tau^{(t)}$  and $Q_\tau^\star - \tau \log \xi^{(t)}$ (to be defined later), and invoke the linear system theory to establish appealing contraction,  which allow for the use of more aggressive constant stepsizes and hence improved convergence.

\newcommand{\topsepremove}{\aboverulesep = 0mm \belowrulesep = 0mm} \topsepremove

\begin{table}[ht]

\begin{center}
	{
\begin{tabular}{c|c|c|c}
\toprule \hline
\multirow{2}{*}{paper} & \multirow{2}{*}{iteration complexity} & \multirow{2}{*}{regularization} & \multirow{2}{*}{learning rates} \tabularnewline
 &  &  &  \tabularnewline	
\toprule 
	\multirow{2}{*}{\citet{agarwal2019optimality}} & \multirow{2}{*}{$\frac{2}{(1-\gamma)^{2}\epsilon}+\frac{2}{\eta\epsilon}$} & \multirow{2}{*}{unregularized} & \multirow{2}{*}{constant: $(0,\infty)$}  \tabularnewline
	 &  &  &    \tabularnewline
\hline 
	\multirow{2}{*}{\citet{bhandari2020note}}	& \multirow{2}{*}{$\frac{1}{(1-\gamma)\min_{s\in\mathcal{S}}\rho(s)}\log\big(\frac{1}{\epsilon}\big)$} & \multirow{2}{*}{unregularized} & \multirow{2}{*}{exact line search} \tabularnewline
	&  &  &   \tabularnewline
\hline 
	\multirow{2}{*}{\textbf{this work}} & \multirow{2}{*}{$\frac{1}{1-\gamma}\log\big(\frac{1}{\epsilon}\big)$} & \multirow{2}{*}{regularized} & \multirow{2}{*}{constant: $\frac{1-\gamma}{\tau}$ $\vphantom{\frac{1}{1}}$}   \tabularnewline
	&  &  &     \tabularnewline
\hline 
	\multirow{2}{*}{\textbf{this work}} & \multirow{2}{*}{$\frac{1}{\eta\tau}\log\big(\frac{1}{\epsilon}\big)$} & \multirow{2}{*}{regularized} & \multirow{2}{*}{constant: $\big(0,\frac{1-\gamma}{\tau}\big)$} $\vphantom{\frac{1}{1}}$   \tabularnewline
	&  &  &     \tabularnewline
\toprule \hline
\end{tabular}
	}
\end{center}
	\caption{The iteration complexities of NPG methods  to reach $\epsilon$-accuracy {\em in terms of optimization error}, where the unregularized (resp.~regularized) version is given by \eqref{eq:NPG-original-0} (cf.~\eqref{eq:NPG-original}) with $\eta$ the learning rate.  
	We assume exact gradient evaluation and softmax parameterization, and hide the dependencies that are logarithmic on problem parameters. Here, $\epsilon$-accuracy or $\epsilon$-optimality for the unregularized (resp.~regularized) case mean $V^{\star}(s) -V^{\pi^{(t)}} (s) \leq\epsilon$ (resp.~$V_{\tau}^{\star}(s)-V_{\tau}^{\pi^{(t)}}(s)\leq\epsilon$) holds simultaneously for all $s\in \cS$;  
	  $\rho$ denotes the initial state distribution, which clearly obeys $\frac{1}{\min_{s\in \cS} \rho(s)}\geq |\cS|$.    \label{tab:comparisons}
}

\end{table}

It  is also helpful to compare our results with the state-of-the-art theory for PG methods with softmax parameterization \citep{agarwal2019optimality,mei2020global}. Specifically, \citet{agarwal2019optimality} established the asymptotic convergence of unregularized PG methods with softmax parameterization, while an iteration complexity of $O(1/\epsilon)$  was recently pinned down by \citet{mei2020global}. In the presence of entropy regularization, \citet{agarwal2019optimality} showed that PG with relative entropy regularization and softmax parameterization enjoys an iteration complexity of $O(1/\epsilon^2)$, while \citet{mei2020global} showed that the entropy-regularized softmax PG method converges linearly in $O(\log(1/\epsilon))$ iterations. However, the dependencies of the iteration complexity in \citet{mei2020global} on other salient parameters like $|\cS|$, $|\cA|$ and $\frac{1}{1-\gamma}$ are not fully specified. Very recently, \citet{li2021softmax} delivered a negative message demonstrating that  these dependencies can be highly pessimistic; in fact, one can find an MDP instance which takes softmax PG methods 
(super)-exponential time (in terms of $|\cS|$ and $\frac{1}{1-\gamma}$) to converge. In contrast, the bounds derived in the current paper are fully non-asymptotic, delineating clear dependencies on all salient problem parameters, which clearly demonstrate the algorithmic advantages of NPG methods. Fig.~\ref{fig:pg_vs_npg} depicts the policy paths of PG and NPG methods with entropy regularization for a simple bandit problem with three actions. It is evident from the plots that the NPG method follows a more direct path to the global optimum compared to the PG counterpart and hence converges faster. In addition, both algorithms converge more rapidly as the regularization parameter $\tau$ increases.
 

\begin{figure}[ht]
\begin{center}
\begin{tabular}{ccc}
\hspace{-2ex}
\includegraphics[width=0.33\textwidth]{{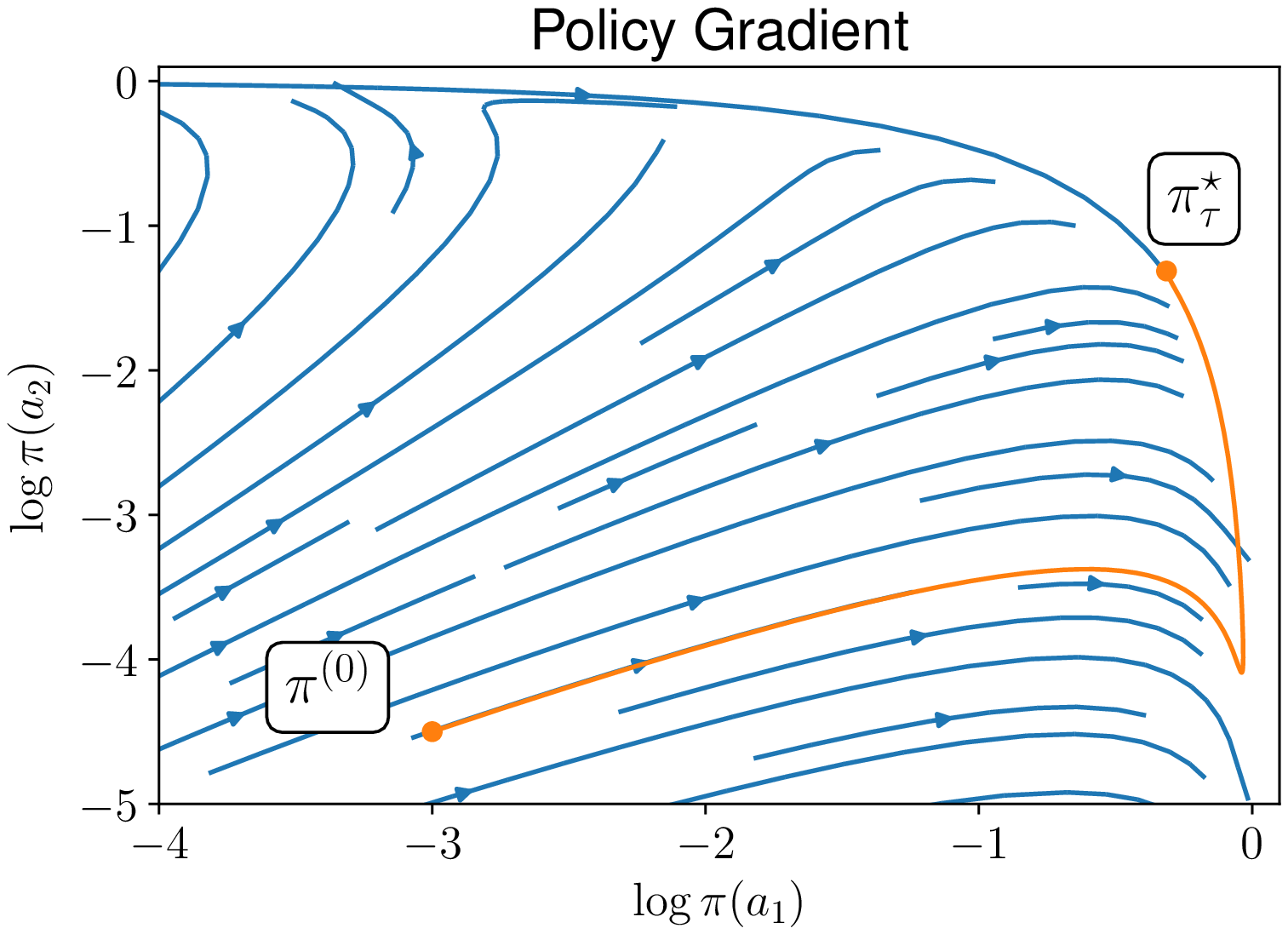}} & \hspace{-2ex}\includegraphics[width=0.33\textwidth]{{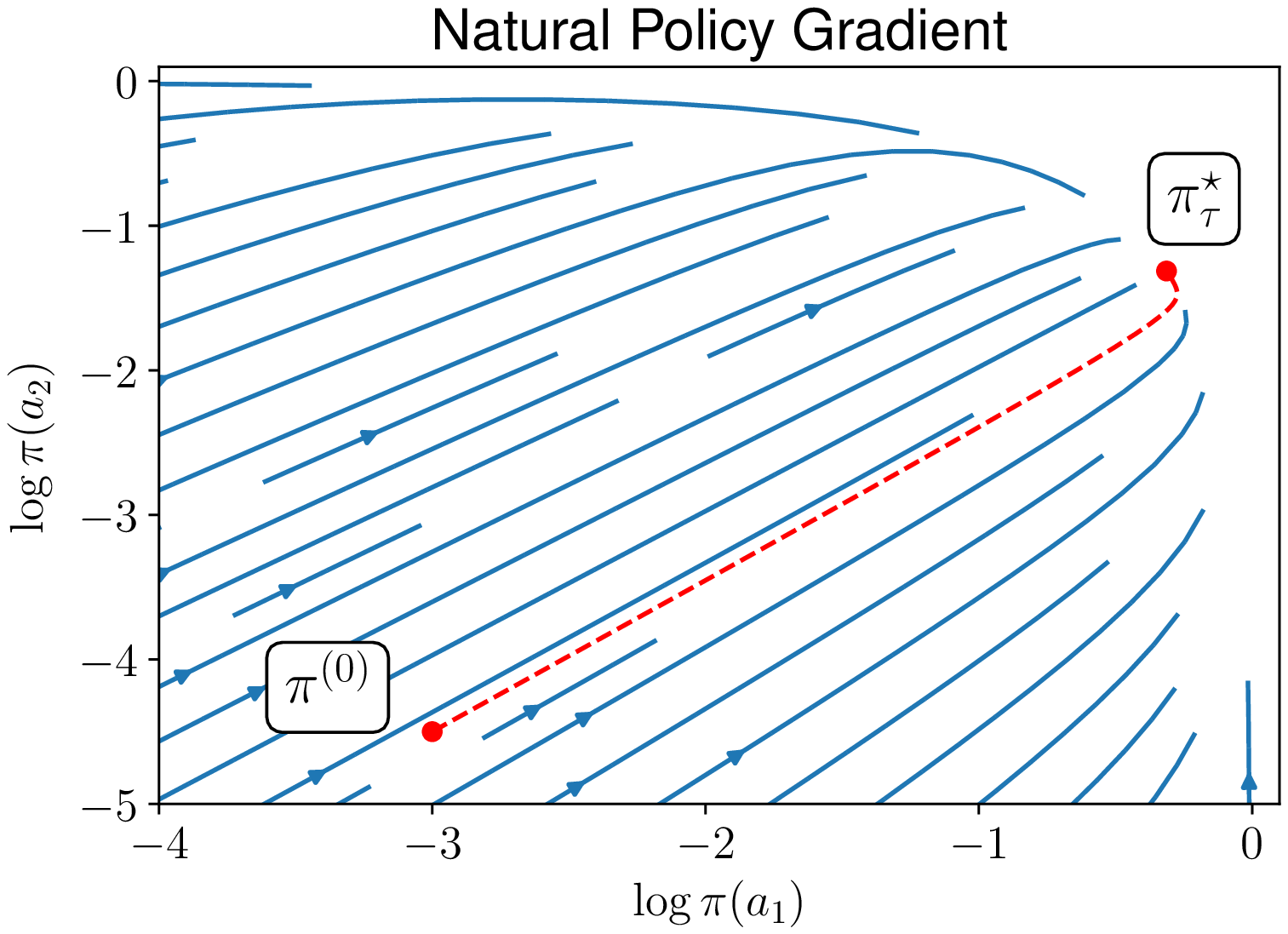}} &\hspace{-2ex}\includegraphics[width=0.33\textwidth]{{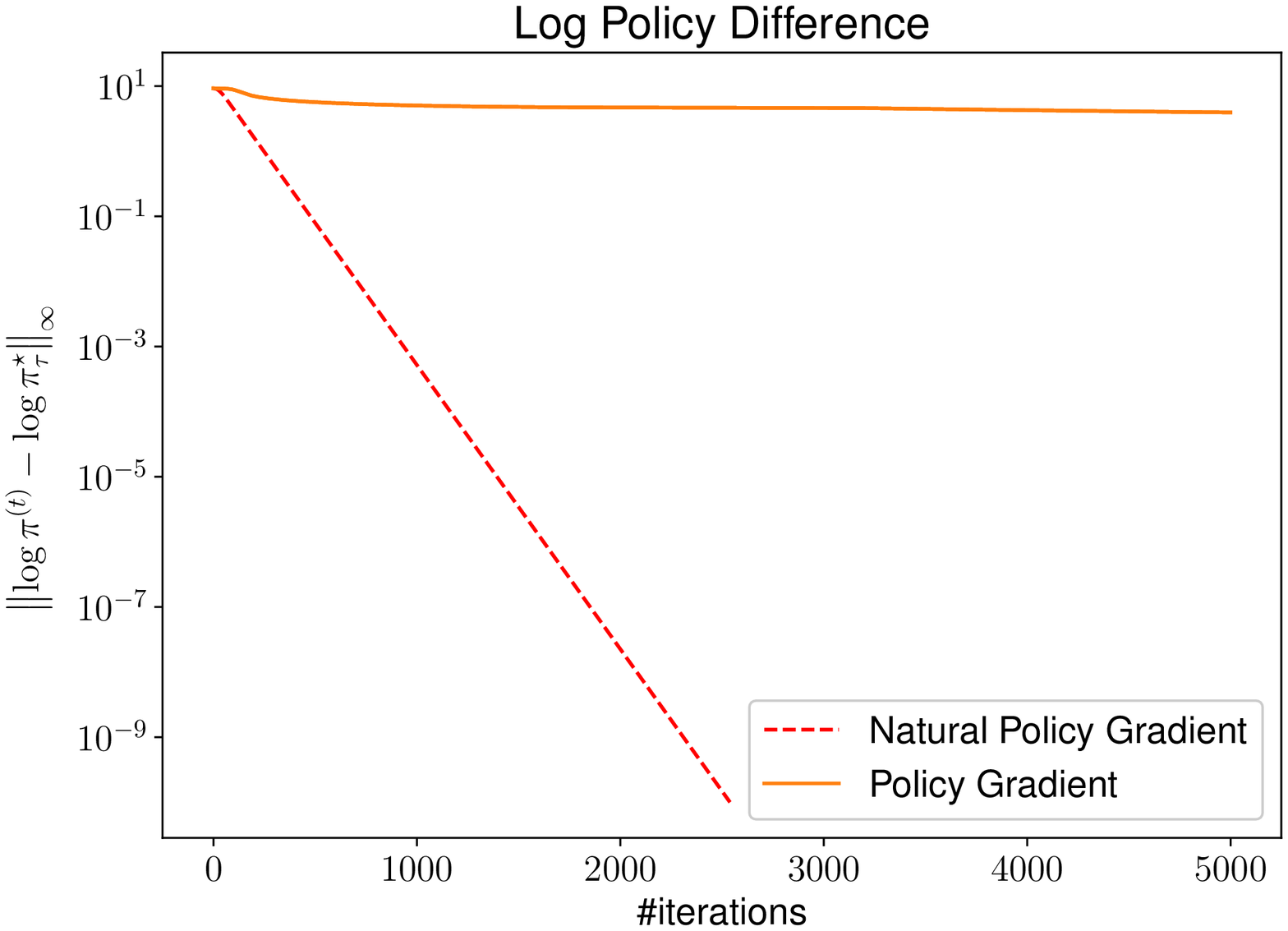}} \\
(a)  regularized PG with $\tau=0.1$ & (b)  regularized NPG with $\tau=0.1$ & (c) error contraction with $\tau=0.1$ \\
\hspace{-2ex}\includegraphics[width=0.33\textwidth]{{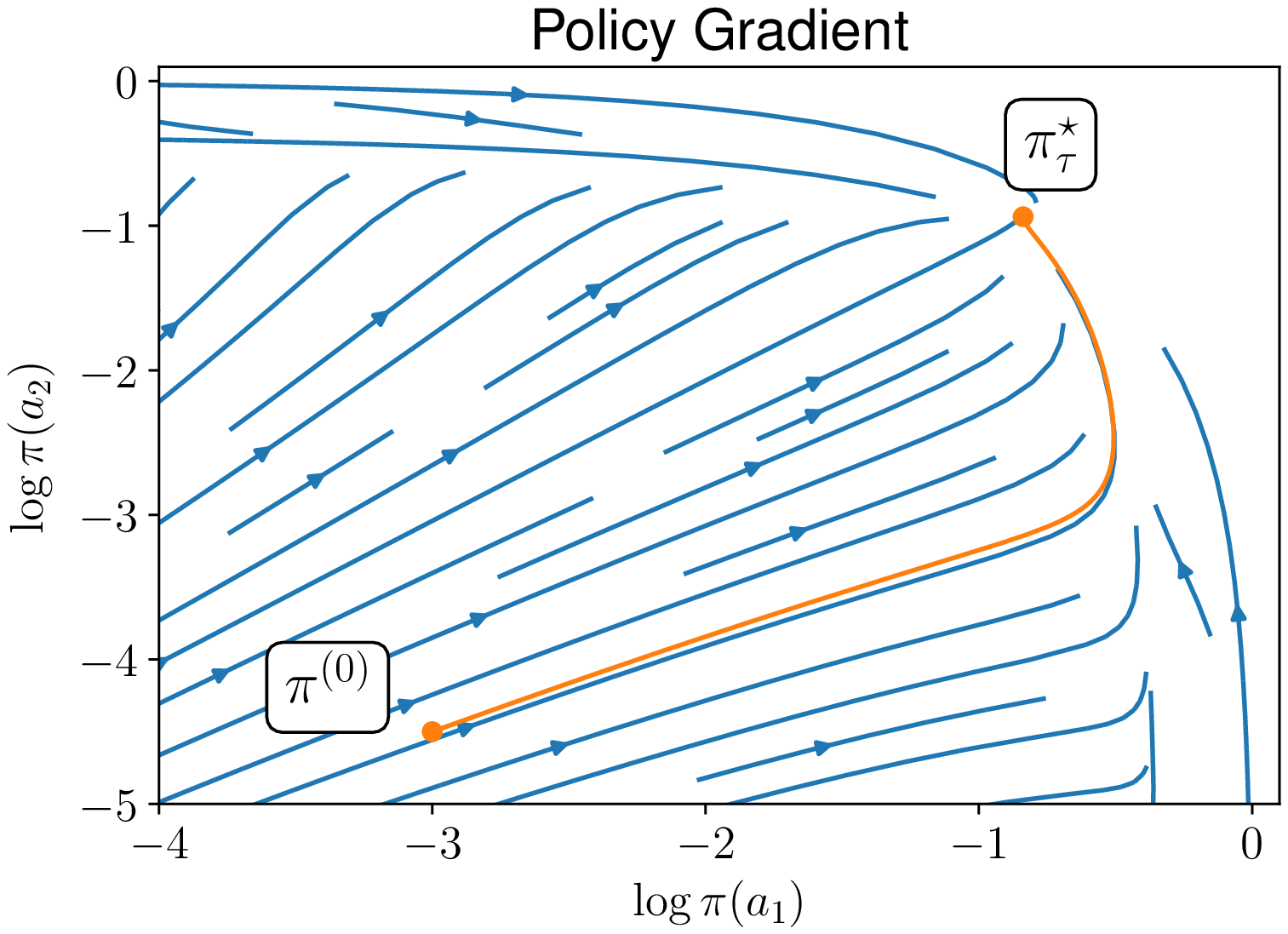}} & \hspace{-2ex}\includegraphics[width=0.33\textwidth]{{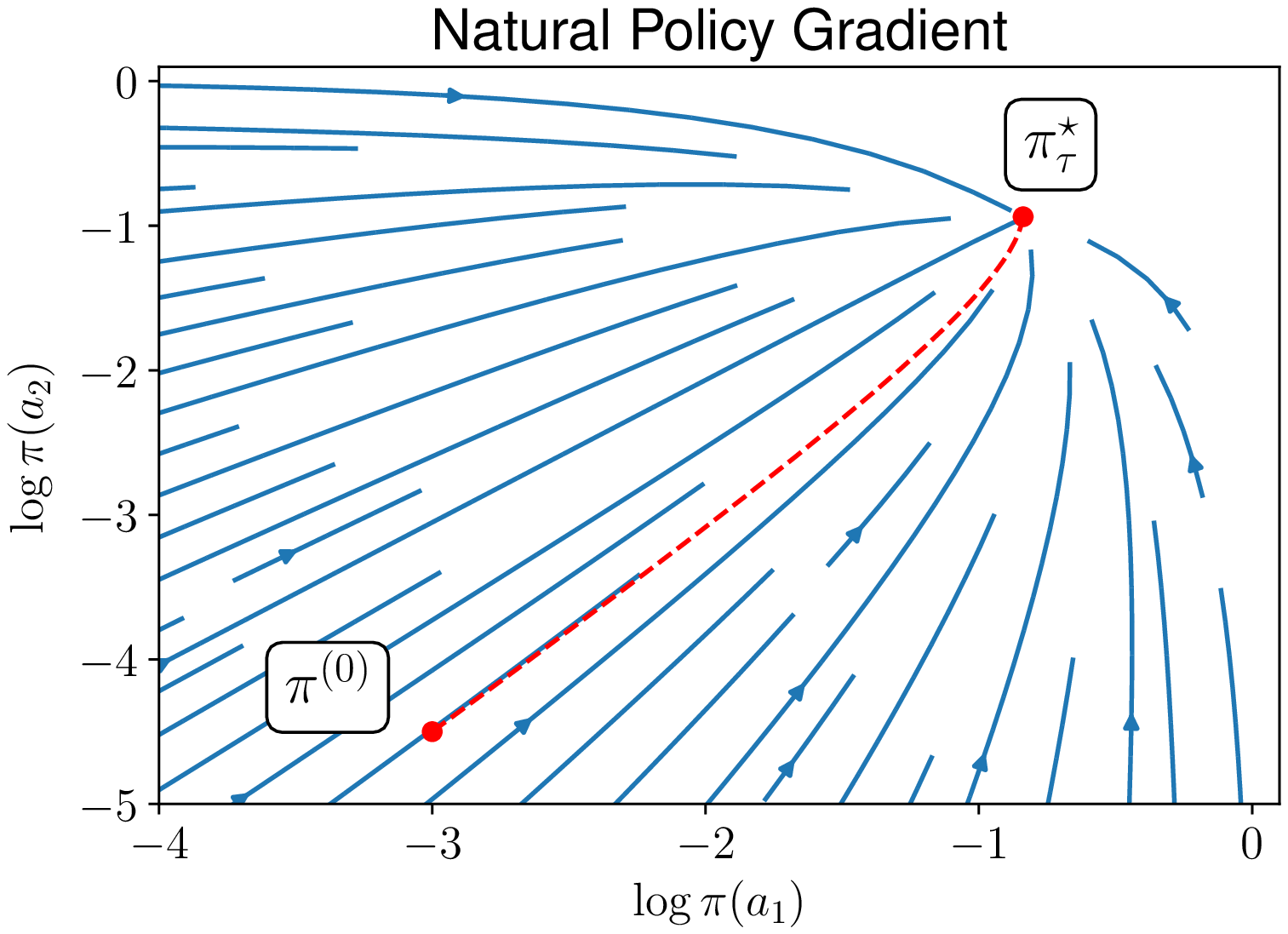}} &\hspace{-2ex}\includegraphics[width=0.33\textwidth]{{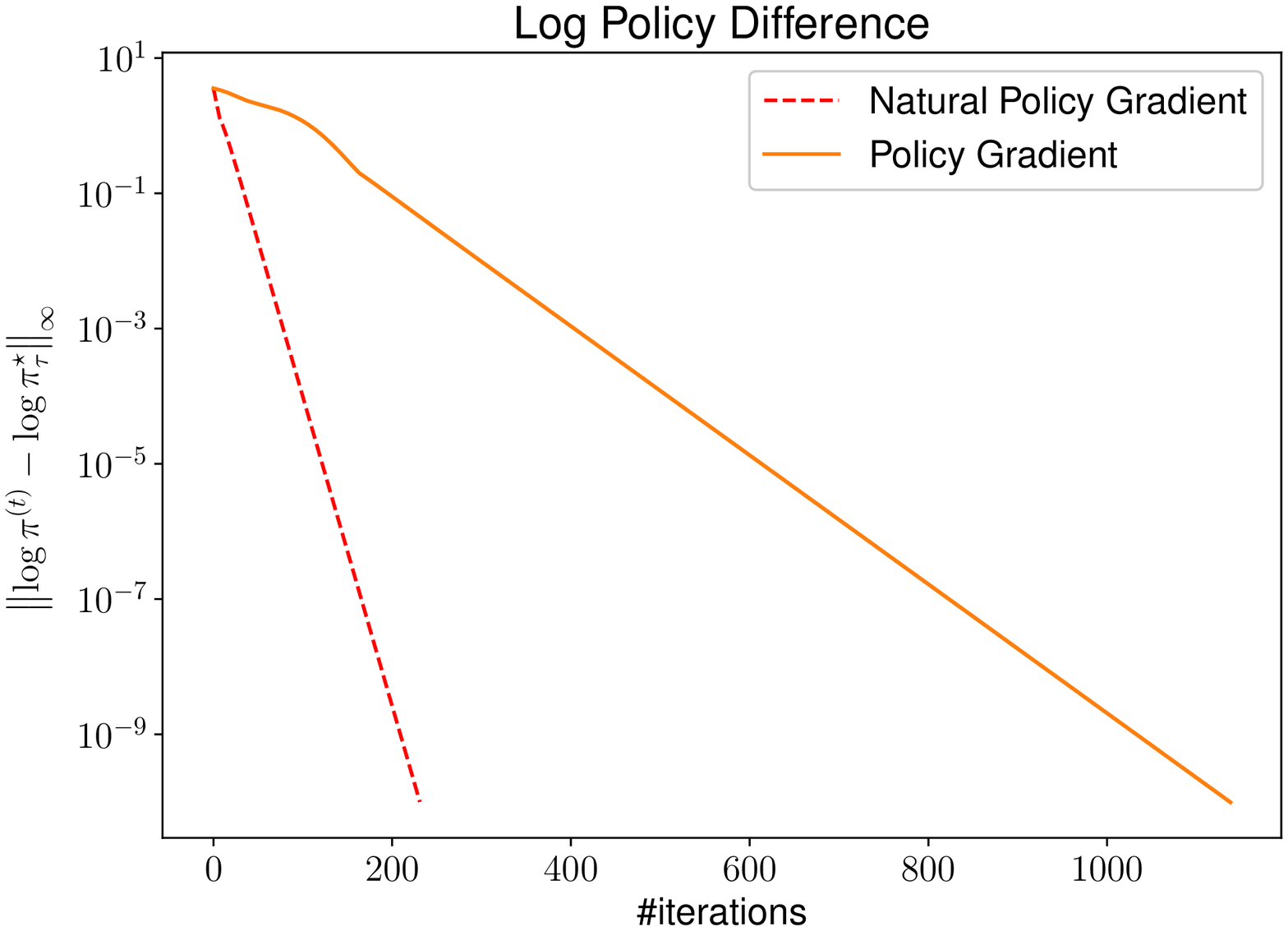}} \\
(d) regularized PG with $\tau=1$ & (e) regularized NPG with $\tau=1$  & (f) error contraction with $\tau=1$
\end{tabular}
\end{center}
\caption{Comparisons of PG and NPG methods with entropy regularization for a bandit problem ($\gamma=0$) with $3$ actions, whose corresponding rewards are $1.0$, $0.9$ and $0.1$, respectively. The regularization parameter is set as $\tau=0.1$ for the first row and $\tau=1$ for the second row. In (a) and (d), the policy paths of $(\log \pi(a_1), \log \pi(a_2))$ following the PG method are plotted in orange, with the blue lines indicating the gradient flow; in (b) and (e), the policy paths of $(\log \pi(a_1), \log \pi(a_2))$ following the NPG method are depicted in red, with the blue lines indicating the natural gradient flow. The error contractions of both PG and NPG methods with $\eta=0.1$ are shown in  (c) and (f).}\label{fig:pg_vs_npg}
\end{figure}

	\subsection{Other related works}


There has been a flurry of recent activities in studying theoretical behaviors of policy optimization methods. For example, \cite{fazel2018global,jansch2020convergence,tu2019gap,zhang2019policy,mohammadi2019convergence} established the global convergence of policy optimization methods for a couple of control problems; \cite{bhandari2019global} identified structural properties that guarantee the global optimality of PG methods without parameterization; \cite{karimi2019non} studied the convergence of PG methods to an approximate first-order stationary point, and \cite{zhang2019global} proposed a variant of PG methods that converges to locally optimal policies leveraging saddle-point escaping algorithms in nonconvex optimization. Beyond the tabular setting, the convergence of PG methods with function approximations has been studied in \cite{agarwal2019optimality,wang2019neural,liu2019neural}. In particular, \cite{cai2019provably} developed an optimistic variant of NPG that incorporates linear function approximation. We do not elaborate on this line of works since our focus is on understanding the performance of entropy-regularized NPG in the tabular setting; we also do not elaborate on PG methods that involve sample-based estimates, since we primarily consider exact gradients or black-box gradient estimators.

Regarding entropy regularization, \cite{neu2017unified,geist2019theory} provided unified views of entropy-regularized MDPs from an optimization perspective by connecting them to algorithms such as mirror descent \citep{nemirovsky1983problem} and dual averaging \citep{nesterov2009primal}. The soft policy iteration algorithm has been identified as a special case of entropy-regularized NPG, highlighting again the link between policy gradient methods and soft Q-learning \citep{schulman2017equivalence}. The asymptotic convergence of soft policy iteration was established in \cite{haarnoja2017reinforcement}, which fell short of providing explicit convergence rate guarantees. Additionally, \citet{grill2019planning} developed planning algorithms for entropy-regularized MDPs, and \citet{mei2020global} showed that the sub-optimality gap of soft policy iteration is small if the policy improvement is small in consecutive iterations.


\subsection{Notation}

We denote by $\Delta(\cS)$ (resp.~$\Delta(\cA)$) the probability simplex over the set $\cS$ (resp.~$\cA$).
When scalar functions such as $|\cdot|$, $\exp(\cdot)$ and $\log(\cdot)$ are applied to vectors, their applications should be understood in an entry-wise fashion. 
For instance, given any vector $z=[z_i]_{1\leq i\leq n}\in \mathbb{R}^n$, the notation $|\cdot|$ denotes $|z| \coloneqq [|z_i|]_{1\leq i\leq n}$; other functions are defined analogously. For any vectors $z=[z_i]_{1\leq i\leq n}$ and $w=[w_i]_{1\leq i\leq n}$, the notation $z \geq w$ (resp.~$ z \leq  w$) means $z_i \geq w_i$ (resp.~$z_i\leq w_i$) for all $1\leq i\leq n$. 
The softmax function $\mathsf{softmax}: \mathbb{R}^n \mapsto \mathbb{R}^n$ is defined such that
$ [\mathsf{softmax}(\theta) ]_i \coloneqq  \exp(\theta_i) / \big( \sum_{i} \exp(\theta_i)  \big) $ for a vector $\theta=[\theta_i]_{1\leq i\leq n}\in \mathbb{R}^n$. Given two probability distributions $\pi_1$ and $\pi_2$ over $\cA$,  the Kullback-Leibler (KL) divergence from $\pi_2$ to $\pi_1$ is defined by $\KL(\pi_1 \,\|\, \pi_2)  \coloneqq \sum_{a\in\cA}\pi_1(a) \log\frac{\pi_1(a)}{\pi_2(a)}$. Given two probability distributions $p$ and $q$ over $\cS$, we introduce the notation $\big\| \frac{p}{q} \big\|_{\infty} \coloneqq \max_{s\in \cS} \frac{p(s)}{q(s)}$ and $\big\| \frac{1}{q} \big\|_{\infty} \coloneqq \max_{s\in \cS} \frac{1}{q(s)}$ .

	\section{Model and algorithms}
\label{sec:models}

\subsection{Problem settings}
\label{sec:MDP}

\paragraph{Markov decision processes.} 
The current paper studies 
 a discounted Markov decision process (MDP) \citep{puterman2014markov} denoted by $\mathcal{M} = (\cS,\cA, P, r,\gamma)$, where $\cS$ is the state space, $\cA$ is the action space,  $\gamma\in (0,1)$ indicates the discount factor,  
$P:\cS\times\cA \rightarrow \Delta(\cS)$ is the transition kernel, and $r:  \cS\times\cA \rightarrow [0,1]$ stands for the reward function.\footnote{For the sake of simplicity, we assume throughout that the reward resides within $[0,1]$. Our results can be generalized in a straightforward manner to other ranges of bounded rewards.}  
To be more specific,  for each state-action pair $(s,a)\in \cS\times \cA$ and any state $s'\in \cS$,  
we denote by $P(s' | {s,a})$ the transition probability from state $s$ to state $s'$ when action $a$ is taken, and  $r(s,a)$ the instantaneous reward received in state $s$ due to action $a$.
 A policy $\pi: \cS \rightarrow \Delta(\cA)$ represents a (randomized) action selection rule, namely,  $\pi(a | s)$  specifies the probability of executing action $a$ in state $s$ for each $(s,a)\in \cS\times \cA$.

\paragraph{Value functions and Q-functions.} For any given policy $\pi$, 
we denote by $V^{\pi}: \cS \rightarrow \real$ the corresponding value function, namely, the expected discounted cumulative reward with an initial state $s_0=s$, given by
\begin{align}
	\label{defn:value-function}
 \forall s\in \cS: \qquad V^{\pi}(s) := \mathbb{E} \left[ \sum_{t=0}^{\infty} \gamma^t r(s_t,a_t ) \,\big|\, s_0 =s \right] ,
\end{align} 
where the action $a_t \sim \pi(\cdot |s_t)$ follows the policy $\pi$ and $s_{t+1}\sim P(\cdot | s_t, a_t)$  is generated by the MDP $\mathcal{M}$ for all $t\geq 0$.
We also overload the notation $\V^\plcy(\rho)$ to indicate the expected value function of a policy $\pi$ when the initial state is drawn from a distribution $\rho$ over $\cS$, namely,
	\begin{align}
		\label{defn:V-pi-rho}
		\V^\plcy(\rho) := \ex{\s\sim \rho}{\V^\plcy(\s)}.
	\end{align}
Additionally, the Q-function $Q^{\pi}: \cS \times \cA \rightarrow \real$ of a policy $\pi$ --- namely, the expected discounted cumulative reward with an initial state $s_0=s$ and an initial action $a_0=a$ --- is defined by 
\begin{equation}
	\label{defn:Q-function}
 	\forall (s,a)\in \cS \times \cA: \qquad Q^{\pi}(s,a) := \mathbb{E} \left[ \sum_{t=0}^{\infty} \gamma^t r(s_t,a_t ) \,\big|\, s_0 =s, a_0 = a \right],
\end{equation} 
where the action  $a_t \sim \pi(\cdot |s_t)$ follows the policy $\pi$ for all $t\geq 1$, and  $s_{t+1}\sim P(\cdot | s_t, a_t)$  is generated by the MDP $\mathcal{M}$ for all $t\geq 0$.

\paragraph{Discounted state visitation distributions.}
	A type of marginal distributions --- commonly dubbed as {\em discounted state visitation distributions} --- plays an important role in our theoretical development. To be specific, 
	the discounted state visitation distribution $d_{s_0}^{\pi}$ of a policy $\pi$ given the initial state $s_0\in \ssp$ is defined by
	\begin{equation}
		\label{eq:defn-d-s0}
		\forall s\in \ssp: \qquad d_{s_0}^{\pi}(s) : = (1-\gamma) \sum_{t=0}^{\infty} \gamma^t \mathbb{P}(s_t = s \mid s_0), 
	\end{equation}
	where the trajectory $(s_0,s_1,\cdots)$ is generated by the MDP $\mathcal{M}$ under policy $\pi$ starting from state $s_0$. 
	In words, $d_{s_0}^{\pi}(\cdot)$ captures the state occupancy probabilities when each state visitation is properly discounted depending on the time stamp. 
	Further, for any distribution $\rho$ over $\cS$, we define the distribution $d_{\rho}^{\pi}$ as follows
	\begin{equation}
		\label{eq:defn-d-rho}
		\forall s\in \ssp: \qquad d_{\rho}^{\pi}(s) : = \mathbb{E}_{s_0\sim \rho}\big[  d_{s_0}^{\pi}(s) \big],
	\end{equation}
	which describes the discounted state visitation distribution when the initial state $s_0$ is randomly drawn from a prescribed initial distribution $\rho$.

	\paragraph{Softmax parameterization.} It is common practice to parameterize the class of feasible policies in a way that is amenable to policy optimization. The focal point of this paper is softmax parameterization --- a widely adopted scheme which naturally ensures that the policy lies in the probability simplex. Specifically, for any $\theta: \ssp \times \asp \to \real$ (called ``logic values''), the corresponding softmax policy $\pi_\prm $ is generated through the softmax transform 
	\begin{align} \label{eq:definition_softmax}
		\plcy_\prm := \mathsf{softmax}(\theta) \qquad \text{or} \qquad 
		\forall (s,a) \in \ssp \times \asp: \quad \plcy_\prm(a|\s) := \frac{ \exp(\theta(s,a)) }{ \sum_{a'\in \asp} \exp(\theta(s,a'))  } .
	\end{align}
	In what follows, we shall often abuse the notation to treat $\pi_\theta$ and ${\theta}$ as vectors in $\mathbb{R}^{|\cS||\cA|}$, and suppress the subscript $\theta$ from $\pi_\theta$, whenever it is clear from the context.

\paragraph{Entropy-regularized value maximization.} 
 To promote exploration and discourage premature convergence to suboptimal policies, a widely used strategy is entropy regularization, which searches for a policy that maximizes the following entropy-regularized value function
	\begin{align}
		\label{defn:V-tau}
		\soft{\V}^{\plcy}(\rho) := \V^\plcy(\rho) + \tau \cdot \mathcal{H}(\rho, \plcy). 
	\end{align}
	%
	%
	Here, the quantity $\tau \ge 0$ denotes the regularization parameter, and $\mathcal{H}(\rho, \plcy)$ stands for a sort of {\em discounted entropy} defined as follows
	\begin{align}
		\label{defn:entropy}
 		\mathcal{H}(\rho, \plcy) := \exlim{\substack{\s_0 \sim \rho, \ac_t\sim \plcy(\cdot|\s_t),\\\s_{t+1}\sim \prob(\cdot | \s_t, \ac_t), \forall t\geq 0}}{\sum_{t=0}^\infty -\disct^t \log \plcy(\ac_t|\s_t)} 
		 =  \frac{1}{1-\gamma}  \mathop\mathbb{E}\limits_{s\sim d_{\rho}^{\pi} } \Bigg[  \sum_{\ac\in \mathcal{A}} \plcy(\ac |\s) \log \frac{1}{ \plcy(\ac |\s) } \Bigg].  
	\end{align}
	Equivalently, $\soft{\V}^{\plcy}$ can be viewed as the value function of $\pi$ by  
	adjusting the instantaneous reward to be policy-dependent regularized version as follows
	\begin{equation}
			\label{eq:regularized_reward}
			 \forall (s,a) \in \cS\times \cA: \qquad r_{\tau}(s,a) := r(s,a) - \tau \log \pi(a| s).
	\end{equation}
	We also define $\soft{\V}^{\plcy}(s)$ analogously when the initial state is fixed to be any given state $s\in \cS$. The regularized Q-function $\soft{\Q}^{\pi}$ of a policy $\pi$, also known as the soft Q-function,\footnote{In this paper, we use the terms ``regularized'' value (resp.~Q) functions and ``soft'' value (resp.~Q) functions interchangeably.} is related to $\soft{\V}^{\plcy}$ as  
	\begin{subequations}
	\begin{align}
	\forall (s,a) \in \cS\times \cA:\qquad	\soft{\Q}^{\pi} (s,a) & = r(s,a) + \gamma \mathbb{E}_{s'\sim P(\cdot|s,a)} \big[ V^{\pi}_{\tau}(s') \big], 		\label{eq:defn-regularized-Q} \\
	\forall s \in \cS:\quad \qquad V^{\pi}_{\tau}(s) &=  \mathbb{E}_{a \sim \pi(\cdot|s)}  \big[ -\tau\log \pi(a|s) + \soft{\Q}^{\pi} (s,a) \big].\label{eq:regularized-V-to-Q} 
	\end{align} 
	\end{subequations}

\paragraph{Optimal policies and stationary distributions.} 
Denote by $\pi^{\star}$ (resp.~$\pi_\tau^{\star}$) the policy that maximizes the value function (resp.~regularized value function with regularization parameter $\tau$), and let $V^\star$ (resp.~$V_\tau^\star$) represent the resulting optimal value function (resp.~regularized value function). Importantly, the optimal policies $\pi^{\star}$ and $\pi_\tau^{\star}$ of the MDP do not depend on the initial distribution $\rho$ \citep{mei2020global}. In addition, $\pi^{\star}$ and $\pi_\tau^{\star}$ maximize the Q-function and the soft Q-function, respectively (which is self-evident from \eqref{eq:defn-regularized-Q}). 
A simple yet crucial connection between $\pi^{\star}$ and $\pi_\tau^{\star}$ can be demonstrated via the following sandwich bound\footnote{To see this, invoke the optimality of $\pi_\tau^{\star}$ and the elementary entropy bound $0\leq \mathcal{H}(\rho,\pi)\leq \frac{1}{1-\gamma}\log|\cA|$ to obtain
\begin{equation*}
V^{\pi_\tau^{\star}}(\rho) + \tfrac{\tau}{1-\gamma} \log |\cA| \geq V^{\pi_\tau^{\star}}(\rho)+ \tau \mathcal{H}(\rho,\pi_\tau^{\star}) = V_\tau^\star(\rho) \geq V_\tau^{\pi_\star}(\rho) \geq V^{\pi_\star}(\rho).
\end{equation*}	
}
\begin{equation}
	\label{eq:set_tau}
	V^{\pi_\tau^{\star}}(\rho) \leq V^{\pi_\star}(\rho)   \leq V^{\pi_\tau^{\star}}(\rho) + \frac{\tau}{1-\gamma} \log |\cA| ,
\end{equation}
which holds for all initial distributions $\rho$. The key takeaway message is that: the optimal policy $\pi_\tau^{\star}$ of the regularized problem could also be nearly optimal in terms of the unregularized value function, as long as the regularization parameter $\tau$ is chosen to be sufficiently small.

\subsection{Algorithm: NPG methods with entropy regularization}
\label{sec:NPG-entropy}


\paragraph{Natural policy gradient methods.} Towards computing the optimal policy (in the parameterized form),  perhaps the first strategy that comes into mind is to run gradient ascent w.r.t.~the parameter $\theta$ until convergence --- a first-order method commonly referred to as the {\em policy gradient} (PG) algorithm (e.g.~\citet{sutton2000policy}). 
In comparison, the {\em natural policy gradient} (NPG) method \citep{kakade2002natural} adopts a pre-conditioned gradient update rule
\begin{align}
	\prm ~\gets~ \prm + \eta \big(\mathcal{F}_\rho^{\prm}\big)^\dagger \nabla_{\theta} {\V}^{\pi_{\theta}} (\rho),
	\label{eq:NPG-original-0}
\end{align}
in the hope of searching along a direction independent of the policy parameterization in use. Here, $\eta$ is the learning rate or stepsize, $\mathcal{F}_\rho^\prm$ denotes the Fisher information matrix given by
\begin{align}
	\label{eq:defn-fisher-information}
	\mathcal{F}_\rho^\prm := \exlim{\s\sim d_{\rho}^{\plcy_\prm},\ac\sim \plcy_\prm(\cdot|\s)}{  \big( \nabla_{\theta} \log \plcy_\prm(\ac|\s) \big) \big( \nabla_{\theta} \log \plcy_\prm(\ac|\s) \big)^\top} ,
\end{align}
and we use $B^\dagger$ to indicate the Moore-Penrose pseudoinverse of a matrix $B$. It has been understood that the NPG method essentially attempts to monitor/control the policy changes approximately in terms of the Kullback-Leibler (KL) divergence (see e.g.~\citet[Section 7]{schulman2015trust}).

\paragraph{NPG methods with entropy regularization.}
Equipped with entropy regularization, the NPG update rule can be written as 
	\begin{align}		
		\label{eq:NPG-original}
		\prm ~\gets~ \prm + \eta \big(\mathcal{F}_\rho^{\prm}\big)^\dagger \nabla_{\theta} {\V}_{\tau}^{\pi_{\theta}} (\rho),
	\end{align}
	where $\mathcal{F}_\rho^{\prm}$ is defined in \eqref{eq:defn-fisher-information} and $\soft{\V}^{\plcy}(\rho)$ is defined in \eqref{defn:V-tau}. Under softmax parameterization, this update rule admits a fairly simple form in the policy space (see Appendix~\ref{sec:derivations_npg} for detailed derivations), which, interestingly, is invariant to the choice of $\rho$. More precisely, if we let $\theta^{(t)}$ denote the $t$-th iterate and $\plcy^{(t)}=\mathsf{softmax}(\theta^{(t)})$ the associated policy, then the entropy-regularized NPG updates satisfy
	\begin{align}\label{eqn:npg-general}
		\plcy^{(t+1)}(\ac|\s) ~=~
		 \frac{1}{ Z^{(t)}(s) } \big( \plcy^{(t)}(\ac|\s) \big)^{1- \frac{\eta\tau}{1-\disct}}\exp\Big( \frac{\eta{\soft{\Q}^{\plcy^{(t)}}(\s, \ac)}} { {1-\disct}} \Big) , 
	\end{align}
	where $\soft{\Q}^{\pi^{(t)}}$ is the soft Q-function of  policy $\pi^{(t)}$, and $Z^{(t)}(s)$ is some normalization factor. This can alternatively be viewed as an instantiation/variant of the {\em trust region policy optimization} (TRPO) algorithm (see \citet{schulman2015trust,shani2019adaptive}). As an important special case, the update rule \eqref{eqn:npg-general} reduces to
	\begin{align}
	\label{eq:SPI-update}
		\plcy^{(t+1)}(\cdot|\s) ~=~ \frac{ 1}{ Z^{(t)}(s) }\exp\Big( \frac{ {\soft{\Q}^{\plcy^{(t)}}(\s, \cdot)}} { {\tau}} \Big)   \qquad \text{when }\eta = \frac{1-\gamma}{\tau}
	\end{align}
	for some normalization factor $Z^{(t)}(s)$. 
	The procedure \eqref{eq:SPI-update} can be interpreted as a ``soft'' version of the classical policy iteration algorithm \citep{bertsekas2017dynamic} (as it employs a softmax function to approximate the max operator) w.r.t. the soft Q-function, and is often dubbed as {\em soft policy iteration} (SPI) (see \citet[Section 4.1]{haarnoja2018soft}).

	To simplify notation,  we shall use $\soft{\V}^{(t)}$, $\soft{\Q}^{(t)}$ and $d_{\rho}^{(t)}$ throughout to denote $\soft{V}^{\plcy^{(t)}}$, $\soft{\Q}^{\plcy^{(t)}}$ and $d_{\rho}^{\plcy^{(t)}}$, respectively. The complete procedure is summarized in Algorithm~\ref{alg:entropy-npg}.


\begin{algorithm}[ht]
\DontPrintSemicolon
   \textbf{inputs:} learning rate $\eta$, initialization $\plcy^{(0)}$. \\

   \For{$t=0,1,2,\cdots$}
	{
		Compute the regularized Q-function $\soft{\Q}^{(t)}$ (defined in \eqref{eq:defn-regularized-Q}) of policy $\plcy^{(t)}$. \\
	     	Update the policy:
		\begin{align}\label{eq:entropy_npg}
			\forall (s,a)\in \cS\times \cA: \quad 
			\plcy^{(t+1)}(\ac|\s) = 
		 	\frac{1} {Z^{(t)}(s)} \big( \plcy^{(t)}(\ac|\s) \big)^{1- \frac{\eta\tau}{1-\disct}} 
				\exp\Big( \frac{\eta{\soft{\Q}^{(t)}(\s, \ac)}} { {1-\disct}} \Big),
		\end{align}
		where $Z^{(t)}(s)= \sum_{a'\in \cA}\big( \plcy^{(t)}(a'|\s) \big)^{1- \frac{\eta\tau}{1-\disct}} 
				\exp\big( \frac{\eta{\soft{\Q}^{(t)}(\s, a')}} { {1-\disct}} \big)$. 
	}

    \caption{Entropy-regularized NPG with exact policy evaluation}
 \label{alg:entropy-npg}
\end{algorithm}

	\subsection{A warm-up example: the bandit case}


Inspired by \citet{schulman2017equivalence,mei2020global}, we look at a toy example --- the bandit case  --- before proceeding to general MDPs. To be more precise, this is concerned with an MDP with only a single state and discount factor $\gamma=0$. Despite its simplicity, the exposition of this example sheds light upon the convergence behavior of the regularized NPG methods of interest. 

In this single-state example with $\gamma = 0$, the aim reduces to computing a policy $\pi_\theta: \mathcal{A} \rightarrow \Delta(\cA)$ that solves the following
optimization problem
	\begin{equation} \label{eq:entropy_bandit}
		\mathop\mathrm{maximize}\limits_{\prm} \mathop\mathbb{E}\limits_{\ac\sim \plcy_\prm} \big[ r(\ac) - \tau \log \plcy_{\theta}(a) \big],
	\end{equation}
where $r(a)$ is the instantaneous reward of taking action  $a$ (i.e.~pulling arm $a$ in the bandit language). As demonstrated in \citet[Proposition 1]{mei2020global}, this toy case is already non-concave and hence nontrivial to solve. 
As it turns out, 
 direct calculation reveals that the optimal policy of \eqref{eq:entropy_bandit} is given by
\begin{align}
	\plcy^{\star}_\tau =\mathsf{softmax} (r/\tau),
	\label{eq:optimal-pi-bandit}
\end{align}
which is in general a randomized policy. When applied to this example, the entropy-regularized NPG update rule \eqref{eq:entropy_npg} simplifies to (up to normalization) 
	\begin{align}
		\label{eq:bandit_update}
		\plcy^{(t+1)}(\ac) ~\propto~ \plcy^{(t)}(\ac)\exp\prn{\eta r(\ac)-\eta \tau\log \plcy^{(t)}(\ac)} =  \big( \plcy^{(t)}(\ac) \big)^{1-\eta\tau}\exp\big( \eta r(\ac) \big),
	\end{align}
with $\eta$ the learning rate. The following proposition, whose proof is fairly elementary and can be found in Appendix~\ref{Sec:Pfbandit}, reveals that the above procedure converges (at least) linearly to the optimal policy $\plcy^{\star}_\tau$.

	\begin{prop}[The bandit case]
		\label{thm:bandit}
		The algorithm \eqref{eq:bandit_update} converges linearly to $\plcy^{\star}_\tau$ (cf.~\eqref{eq:optimal-pi-bandit}) in an entrywise fashion, namely,
		\[
			\big\| \log \plcy^{(t)} - \log \plcy^{\star}_\tau \big\|_\infty \le 2(1-\tau\eta)^t \big\| \log \plcy^{(0)} - \log \plcy^{\star}_\tau \big\|_\infty.
		\]
	\end{prop}
While this result concentrates only on a toy example, 
it hints at the potential capability of entropy-regularized NPG methods in achieving rapid convergence. In particular, by setting the learning rate to be $\eta = 1/\tau$, the algorithm converges in {\em a single iteration}. This special choice corresponds to the SPI update \eqref{eq:SPI-update}, which will be singled out in our general theory due to its appealing convergence properties. 


	\section{Main results}

Given its appealing convergence behavior when applied to 
the preceding warm-up example (the bandit case), it is natural to ask whether the entropy-regularized NPG method is fast-convergent for general MDPs. 
This section answers this question in the affirmative. 

\subsection{Exact entropy-regularized NPG methods}

We first study the convergence behavior of entropy-regularized NPG methods \eqref{eq:entropy_npg} assuming access to exact policy evaluation in every iteration (namely, we assume the soft Q-function $\soft{\Q}^{{(t)}}$ can be evaluated accurately in all $t$).  Remarkably, this algorithm converges linearly ---  in terms of computing both the optimal soft Q-function $\soft{Q}^{\star}$ and the associated log policy $ \log \pi_\tau^{\star}$ --- as asserted by the following theorem. The proof of this result is provided in Section~\ref{Sec:pfNPG-eta}.
 
\begin{theorem}[Linear convergence of exact entropy-regularized NPG]
\label{thm:npg_exact}
For any learning rate $0 < \eta  \le (1-\disct)/\tau$, the entropy-regularized NPG updates \eqref{eq:entropy_npg} satisfy 
\begin{subequations} \label{eq:linear_convergence_general}
\begin{align}
\normbig{\soft{Q}^{\star} - \soft{Q}^{(t+1)}}_\infty & \le C_1 \disct\prn{1 - \eta\tau}^{t} \\
\normbig{\log \soft{\plcy}^\star - \log\plcy^{(t+1)}}_\infty & \le 2C_1 \tau^{-1}(1 - \eta \tau)^{t} 
\end{align}
\end{subequations}
for all $t\geq 0$, where
\begin{equation}\label{eq:C1}
C_1 := \normbig{\soft{Q}^\star - \soft{Q}^{(0)}}_\infty + 2\tau  \prn{1 - \frac{\eta \tau}{1 - \disct}} \normbig{\log \soft{\plcy}^\star - \log \plcy^{(0)}}_\infty.
\end{equation}

\end{theorem}

It is worth emphasizing that Theorem~\ref{thm:npg_exact} is stated in a completely non-asymptotic form containing \emph{no} hidden constants, and that our result covers any learning rate $\eta$ in the range $(0,  (1-\disct)/\tau]$. A few implications of this theorem are in order.
\begin{itemize}
\item \textbf{Linear convergence of soft Q-functions.} To reach $\normbig{\soft{Q}^{\star} - \soft{Q}^{(t)}}_\infty \leq \epsilon$, the entropy-regularized NPG method needs at most
$\frac{1}{\eta \tau} \log \prn{\frac{C_1 \gamma}{\epsilon}}$
%
iterations. Remarkably, the iteration complexity almost does not depend on the dimensions of the MDP (except for some very weak dependency embedded in $\log C_1$) --- this inherits a dimension-free feature of NPG methods that has been highlighted in \citet{agarwal2019optimality} for the unregularized case. When the learning rate $\eta$ is fixed in the admissible range, the iteration complexity scales inverse proportionally with $\tau$, 
		suggesting a higher level of entropy regularization might accelerate convergence, albeit to the solution of a regularized problem that is further away from the original MDP.


\item \textbf{Linear convergence of log policies.} In contrast to the unregularized case, entropy regularization ensures uniqueness of the optimal policy and, therefore, makes it possible to study the convergence  of the policy directly. Our theorem reveals that the entropy-regularized NPG method needs at most
$\frac{1}{\eta \tau} \log \prn{\frac{2C_1}{\epsilon\tau}}$
iterations to yield $\normbig{\log\soft{\plcy}^\star - \log\plcy^{(t+1)}}_\infty \leq \epsilon$. 

 \item 
\textbf{Linear convergence of soft value functions.}  As a byproduct, Theorem~\ref{thm:npg_exact} implies that the iterates of soft value functions also converge linearly, namely, 
\begin{equation}
 	\normbig{\soft{V}^{\star} - \soft{V}^{(t+1)}}_\infty \le 3C_1 \disct\prn{1 - \eta\tau}^{t}.
	\label{eq:convergence-soft-value}
\end{equation}
To see this, we make note of the following relation previously established in \citet{nachum2017bridging}: 
\[
	\forall (s,a)\in \cS\times \cA: \qquad \soft{\V}^\star(\s) = -\tau \log \soft{\plcy}^\star(\ac|\s) + \soft{\Q}^\star(\s, \ac),
\]
\[
	\Longrightarrow \qquad \soft{\V}^\star(\s) = \exlimbig{\ac\sim \plcy^{(t+1)}(\cdot|\s)}{-\tau \log \soft{\plcy}^\star(\ac|\s) + \soft{\Q}^\star(\s, \ac)}. 
\]
Consequently, combining this with the definition \eqref{eq:regularized-V-to-Q} yields
\begin{align*}
	\big| \soft{V}^{\star}(\s) - \soft{V}^{(t+1)}(\s) \big| &=  \exlim{\ac\sim \plcy^{(t+1)}(\cdot|\s)}{ \prnBig{ -\tau \log \soft{\plcy}^\star(\ac|\s) + \soft{\Q}^\star(\s, \ac) } - \prnBig{-\tau \log \soft{\plcy}^{(t+1)}(\ac|\s) + \soft{\Q}^{(t+1)}(\s, \ac)}}\\
 &\le \tau \,\normbig{\log \soft{\plcy}^\star - \log \soft{\plcy}^{(t+1)}}_\infty + \normbig{\soft{\Q}^\star - \soft{\Q}^{(t+1)}}_\infty,
\end{align*}
which together with \eqref{eq:linear_convergence_general} immediately establishes \eqref{eq:convergence-soft-value}.



\item \textbf{Convergence rate of SPI.} The best convergence guarantee is achieved when $\eta = (1-\gamma)/\tau$ (i.e.~the SPI case), where the iteration complexity to reach $\normbig{\soft{Q}^{\star} - \soft{Q}^{(t)}}_\infty \leq \epsilon$ reduces to 
	$$\frac{1}{1-\gamma}\log \prnBigg{\frac{\gamma \normbig{\soft{Q}^\star - \soft{Q}^{(0)}}_\infty}{\epsilon}},$$
	which is proportional to the effective horizon $\frac{1}{1-\gamma}$ modulo some log factor. 
		This means the iteration complexity of SPI recovers that of policy iteration \citep{puterman2014markov}. Interestingly, the contraction rate  in this case (which is $\gamma$) is independent of the choice of the regularization parameter $\tau$. Similarly, the iteration complexity of SPI to reach $\norm{\log\soft{\plcy}^\star - \log\plcy^{(t+1)}}_\infty \leq \epsilon$ becomes $\frac{1}{1-\gamma}\log \prnbig{\frac{2\|{\soft{Q}^\star - \soft{Q}^{(0)}}\|_\infty}{\epsilon\tau}}$, and the contraction rate is again independent of $\tau$.
\end{itemize}

\paragraph{Comparison with entropy-regularized policy gradient methods.}
\citet[Theorem 6]{mei2020global} proved that the entropy-regularized policy gradient method achieves\footnote{Here, we have assumed the exact policy gradient is computed with respect to $\soft{\V}^{(t)}(\rho)$.}  
\begin{align*}
\soft{\V}^{\star}(\rho) - \soft{\V}^{(t)}(\rho) 
\le&\prn{\soft{\V}^{\star}(\rho) - \soft{\V}^{(0)}(\rho)}\\
&\hspace{-5ex} \cdot \exp\prn{ - \frac{(1-\disct)^4t }{(8/\tau + 4 + 8\log |\asp|)|\ssp|}\norm{\frac{d_\rho^{\plcy^\star_\tau}}{\rho}}^{-1}_\infty \min_s \rho(s) \prn{\inf_{0\leq k\leq t-1}\min_{\s, \ac} \plcy^{(k)}(\ac|\s) }^2},
\end{align*}
and they further showed that $\inf_{k\geq 0}\min_{\s, \ac} \plcy^{(k)}(\ac|\s) $  is non-vanishing in $t$. It remains unclear, however, how  $\inf_{t\geq 0}\min_{\s, \ac} \plcy^{(t)}(\ac|\s) $ scales with other potentially large salient parameters like $(|\cS|, |\cA|,  \frac{1}{1-\gamma}, \frac{1}{\tau})$. 
In truth, existing theory does not rule out the possibility of exponential dependency on these salient parameters.   
It would thus be of great interest to establish algorithm-dependent lower bounds to uncover the right scaling with these important parameters. 
In contrast, our convergence guarantees for entropy-regularized NPG methods unveil concrete dependencies on all problem parameters.

\paragraph{Computing an $\epsilon$-optimal policy for the original MDP.}
Thus far, we have established an intriguing convergence behavior of the entropy-regularized
NPG method. However, caution needs to be exercised when interpreting the efficacy of this method:  
 the preceding results are concerned with convergence to
the optimal regularized value function $V_{\tau}^{\star}$, as opposed
to finding the optimal value function $V^{\star}$ of the original MDP. Fortunately, by choosing the
regularization parameter $\tau$ to be sufficiently small (in accordance with the target accuracy level $\epsilon$), we can guarantee that $V_{\tau}^{\star}\approx V^{\star}$ (cf.~\eqref{eq:set_tau}),
thus ensuring the relevance and applicability of our results for solving the original MDP. To be specific, let us adopt the following choice of $\tau$:
\begin{equation}
	\tau=\frac{(1-\gamma)\epsilon}{4\log|\mathcal{A}|} ,
	\label{eq:choice-tau-xi}
\end{equation}
and assume the error of the regularized value function satisfies $\normbig{ V_{\tau}^{\star}  -  V_{\tau}^{(t)} }_{\infty} <{\epsilon}/{2}$. 
By virtue of Theorem~\ref{thm:npg_exact}, this optimization accuracy can be achieved via no more than $\frac{4\log |\cA|}{(1-\gamma)\eta \epsilon} \log \prn{\frac{2C_1 \gamma}{\epsilon}}$ iterations of entropy-regularized NPG updates with a general learning rate,\footnote{This result is in fact better than the iteration complexity $\frac{2}{(1-\gamma)^2\epsilon}$ of the unregularized NPG method established in \citet{agarwal2019optimality} as soon as $\eta\geq 2 (1-\gamma)\log|\mathcal{A}| \log \prn{\frac{2C_1 \gamma}{\epsilon}}$.
Consequently, our finding hints at the potential advantage of entropy-regularized NPG methods over the unregularized counterpart even when solving the original MDP. } or no more than $\frac{1}{1-\gamma}\log \prnBig{\frac{\gamma \norm{\soft{Q}^\star - \soft{Q}^{(0)}}_\infty}{\epsilon}}$ iterations with the specific choice $\eta= \frac{1-\gamma}{\tau}$. 
It then follows that
\begin{align*}
	V^{\star}(s)  - V^{(t)}(s) & = V^{\star}(s)  - V_{\tau}^{\star}(s) + V_{\tau}^{\star}(s)  -  V_{\tau}^{(t)}(s) +  V_{\tau}^{(t)}(s)  - V^{(t)}(s) \\
	& \leq \big( V^{\star}(s)  - V_{\tau}^{\star}(s) \big) + \normbig{ V_{\tau}^{\star}  -  V_{\tau}^{(t)} }_{\infty}  + \prnbig{ V_{\tau}^{(t)} (s) - V^{(t)} (s)} \\
	& \leq \frac{2\tau\log |\cA|}{1-\gamma} +\frac{\epsilon}{2} = \epsilon
\end{align*}
for any $s\in \cS$, where we have used our choice of $\tau $ in \eqref{eq:choice-tau-xi}. 
Here, the second inequality arises from \eqref{eq:set_tau} as well as the fact that for any policy $\pi$,
\[
	\big\| V_{\tau}^{\pi}-V^{\pi}\big\|_{\infty}=\tau\max_{s}\big|\mathcal{H}(s,\pi)\big|\leq\frac{\tau\log|\mathcal{A}|}{1-\gamma},
\]
given  the elementary entropy  bound $0\leq\mathcal{H}(s,\pi)\leq\frac{1}{1-\gamma}\log|\mathcal{A}|$.

\paragraph{Convergence guarantee for conservative policy iteration (CPI).} Our analysis framework also leads to a similar convergence guarantee for a type of policy updates adopted in {\em conservative policy iteration} \citep{kakade2002approximately}, where 
the policy is updated as a convex combination of the previous policy and an improved one. We refer
the interested reader to Appendix~\ref{sec:cpi} for details.

\subsection{Approximate entropy-regularized NPG methods} 
There is no shortage of scenarios where the soft Q-function $\soft{\Q}^{{(t)}}(\s, \ac)$ is available only in an approximate fashion, e.g.~the cases when the value function has to be evaluated using finite samples. To account for inexactness of policy evaluation, we extend our theory to accommodate the following approximate update rule: for any $s\in \cS$ and any $t\geq 0$, 
\begin{equation}
	\label{eq:soft_greedy_approximate}
	\plcy^{(t+1)}(\cdot|\s) ~\propto~\big( \plcy^{(t)}(\cdot|\s) \big)^{1- \frac{\eta\tau}{1-\disct}}\exp\Big( \frac{\eta{ \estsoftQ(\s, \cdot)}} { {1-\disct}} \Big), \quad \mbox{where} \quad 
	\big\|\estsoftQ - \soft{\Q}^{{(t)}} \big\|_{\infty} \leq \delta .
\end{equation}
Here, $\delta$ is some quantity that captures the size of approximation errors. 
We do not specify the estimator for the soft Q-function (as long as it satisfies the entrywise estimation bound), thus allowing one to plug in both model-based and model-free value function estimators designed for a variety of sampling mechanisms  (e.g.~\citet{azar2013minimax,li2020sample}).  Encouragingly, the algorithm \eqref{eq:soft_greedy_approximate} is robust vis-\`a-vis inexactness of value function estimates, as it still converges linearly until an error floor is hit. This is formalized in the following theorem, with the proof postponed to Section~\ref{sec:Pf-inexact}. 
 
 \begin{theorem}[Linear convergence of approximate entropy-regularized NPG]
 	\label{thm:npg_inexact}
 	When $0 < \eta  \le (1-\disct)/\tau$, the inexact entropy-regularized NPG updates \eqref{eq:soft_greedy_approximate} 
 	satisfy
\begin{subequations}
\begin{align}
 	\normbig{\soft{Q}^{\star} - \soft{Q}^{(t+1)}}_\infty&\le \disct\brk{\prn{1 - \eta\tau}^{t} C_1 + C_2} \\
	\normbig{\log \soft{\plcy}^\star - \log\plcy^{(t+1)}}_\infty &\le 2\tau^{-1}\brk{\prn{1 - \eta \tau}^{t} C_1 + C_2} 
 	\end{align}
	\end{subequations}
 	for all $t\geq 0$, where $C_1$ is the same as defined in \eqref{eq:C1} and $C_2$ is given by
 	\begin{equation} 
		\label{eq:C2}
 		C_2 :=\frac{2\delta}{1-\gamma}\prn{ 1+\frac{\gamma}{\eta\tau}} = \frac{2\delta}{(1-\disct)^2} \brk{1+\disct \prn{\frac{1-\disct}{\eta\tau}-1}}.
 	\end{equation}
	 %
 \end{theorem}
Apparently, Theorem~\ref{thm:npg_inexact} reduces to Theorem~\ref{thm:npg_exact} when $\delta=0$. As implied by this theorem, if the $\ell_\infty$ error of the soft-Q function estimates does not exceed
$$ \delta \leq   \frac{(1-\gamma)^2\epsilon}{2\gamma \brk{1+\disct \prn{\frac{1-\disct}{\eta\tau}-1}}} ,$$ 
then the algorithm \eqref{eq:soft_greedy_approximate} achieves $2\epsilon$-accuracy (i.e.~$\normbig{\soft{Q}^{\star} - \soft{Q}^{(t)}}_\infty \leq 2\epsilon$)
within $\frac{1}{\eta\tau} \log\big(\frac{C_1\gamma}{\epsilon}\big)$ iterations. 
%
%
In particular, in the case of soft policy iteration (i.e.~$\eta =\frac{1-\gamma}{\tau}$), the tolerance level $\delta$ can be up to $\frac{(1-\gamma)^2\epsilon}{2\gamma }$, which matches the theory of approximate policy iteration in \citet{agarwal2019reinforcement}.

\begin{remark}
It is straightforward to combine Theorem~\ref{thm:npg_inexact} with known sample complexities for approximate policy evaluation to obtain a crude sample complexity bound. For instance, assuming access to a generative model, \citet{li2020breaking} asserts that for any fixed policy $\pi$, model-based policy evaluation 
achieves $\big\|\widehat{Q}_\tau^\pi - {Q}_\tau^\pi\big\|_\infty \le \delta$ with high probability, as long as the number of samples per state-action pair exceeds the order of
	\[
		\frac{1}{ (1-\gamma)^3 \delta^2}  
	\]
	 up to some logarithmic factor. By employing fresh samples for each policy evaluation, we can set $\delta = \frac{(1-\gamma)^2\epsilon}{2\gamma }$ and invoke the union bound over 
$\widetilde{O}\big(\frac{1}{1-\gamma}\big)$ iterations to demonstrate that: 
SPI with model-based policy evaluation needs at most
	\[
		\widetilde{O}\left(\frac{|\cS|\, |\cA|}{ (1-\gamma)^{8} \epsilon^2}  \right)
	\]
	samples to find an $\epsilon$-optimal policy. Here, $\widetilde{O}(\cdot)$ hides any logarithmic factor. We note, however, that the above sample analysis is extremely crude and might be improvable by, say, allowing sample reuses across iterations. It remains an interesting open question as to whether NPG with entropy regularization is minimax-optimal with a generative model, where the minimax lower bound is on the order of $ \frac{|\cS|\, |\cA|}{ (1-\gamma)^{3} \epsilon^2}  $  \citep{azar2013minimax} and achievable by model-based plug-in estimators \citep{agarwal2019model,li2020breaking} but not by vanilla Q-learning \citep{li2021tightening}. 
\end{remark}


\subsection{Quadratic convergence in the small-$\epsilon$ regime}
 
Somewhat remarkably, the regularized NPG method with $\eta = \frac{1-\gamma}{\tau}$ achieves super-linear convergence in computing $\soft{\V}^{\star}$, once the algorithm enters a sufficiently small local neighborhood surrounding the optimizer. 
 
Before presenting the result, we need to introduce the stationary distribution over $\mathcal{S}$ of the MDP $\mathcal{M}$ under policy $\pi_{\tau}^{\star}$, denoted by $\mu_{\tau}^\star \in \Delta(\cS)$. It is straightforward to verify the following basic property
\begin{equation} \label{eq:stationary_dist_disc}
d_{\mu_{\tau}^\star}^{\plcy^{\star}_\tau} = \mu_{\tau}^\star,
\end{equation}
given that the state visitation distribution remains unchanged if the initial state is already in a steady state. Throughout this paper, we assume that $\min_s \mu_{\tau}^\star(s) > 0$. Our finding is stated in the following theorem, with the proof deferred to Section~\ref{sec:quadratic-sketch}.

\begin{theorem}[Quadratic convergence of exact regularized NPG] 
\label{thm:spi_sp}
Suppose that the algorithm \eqref{eq:SPI-update} with $\eta = \frac{1-\gamma}{\tau}$ (or SPI) satifies
\begin{equation}
\label{eqn:pit-condition}
	\normbig{\log \plcy^{(t)} - \log \soft{\plcy}^\star}_\infty \le 1.
\end{equation}
%
%
for all $t\geq 0$, then one has
	\begin{equation*}
		\soft{V}^\star(\rho) - \soft{V}^{(t)}(\rho) \le 
		\norm{\frac{\rho}{\soft{\mu}^\star}}_\infty 
		\frac{(1 - \disct)\tau}{4\disct^2} \norm{\frac{1}{\soft{\mu}^\star}}_\infty ^{-1}
		 \prn{\frac{4\disct^2}{(1 - \disct)\tau} \norm{\frac{1}{\soft{\mu}^\star}}_\infty  \prn{\soft{V}^\star(\soft{\mu}^\star) - \soft{V}^{(0)}(\soft{\mu}^\star)}}^{2^t} .
	\end{equation*}
\end{theorem}


\begin{remark}
In view of the convergence guarantees in Theorem~\ref{thm:npg_inexact},  
a suitable initialization of $\pi^{(0)}$ and $\soft{V}^{(0)}$ (such that 
$\frac{4\disct^2}{(1 - \disct)\tau} \normbig{\frac{1}{\soft{\mu}^\star}}_\infty (\soft{V}^\star(\soft{\mu}^\star) - \soft{V}^{(0)}(\soft{\mu}^\star)) < 1$)
can be obtained by running SPI for sufficiently many iterations; further, all subsequent iterations are then guaranteed to satisfy \eqref{eqn:pit-condition} according to Theorem~\ref{thm:npg_inexact}. 
\end{remark}

Under the assumptions of Theorem~\ref{thm:spi_sp},  our result indicates that: when $\epsilon$ is sufficiently small, 
the iteration complexity for SPI to yield an $\epsilon$ optimization accuracy --- that is, $\soft{\V}^\star(\rho) - \soft{\V}^{(t)}(\rho) \le \epsilon$ --- is at most on the order of 
\begin{equation}
 	\log\log\prn{\frac{(1 - \disct)\tau}{4\disct^2} \norm{\frac{1}{\soft{\mu}^\star}}_\infty^{-1}\norm{\frac{\rho}{\soft{\mu}^\star}}_\infty\frac{1}{\epsilon}}.
\end{equation}
This uncovers the faster-than-linear convergence behavior of regularized NPG methods in the high-accuracy regime, 
accommodating a range of {\em optimization} accuracy and all possible choices of the regularization parameter $\tau$. 
It is worth noting, however, that our quadratic convergence result is stated in terms of the optimization accuracy 
(namely, convergence to the soft value function $\soft{\V}^{\star}(\rho)$) as opposed to the accuracy w.r.t.~the original unregularized MDP. 
Thus, interpreting Theorem~\ref{thm:spi_sp} in practice requires caution, since the approximation error $\soft{\V}^{\star}(\rho)-V^{\star}(\rho)$ might sometimes dominate the optimization error in this regime. 


%
%

		\section{Analysis}

	\subsection{Main pillars for the convergence analysis}
	
	Before proceeding, we isolate a few ingredients that provide the main pillars
	for our theoretical development. 
	
	\paragraph{Performance improvement and monotonicity.}
	This lemma is a sort of {\em ascent lemma}, which quantifies the progress made over each iteration  --- measured in terms of the soft value function.	
	\begin{lemma}[Performance improvement] 
	\label{lemma:dalaopo}
		Suppose that $0<\eta \le (1-\disct)/\tau$. For any distribution $\rho$, one has
	\begin{align}
		\soft{\V}^{(t+1)}(\rho) -  \soft{\V}^{(t)}(\rho) 
		=& \exlim{\s \sim d_{\rho}^{(t+1)}} 
		{\prn{\frac{1}{\eta}-\frac{\tau}{1-\disct}} \KL\Big( \plcy^{(t+1)}(\cdot|\s) \,\big\|\, \plcy^{(t)}(\cdot|\s) \Big) 
		+ \frac{1}{\eta} \KL\prn{\plcy^{(t)}(\cdot |\s) \,\big\|\, \plcy^{(t+1)}(\cdot | \s)}}.
		\label{eq:Lemma-1-V-diff-identity}
	\end{align}
	\end{lemma}
	\begin{proof} See Appendix~\ref{sec:pflemdalaopo}. \end{proof}

	In a nutshell, Lemma~\ref{lemma:dalaopo} asserts that each iteration of the entropy-regularized NPG method is guaranteed to improve the estimates of the soft value function, with the improvement depending on the KL divergence between the current policy $\pi^{(t)}$ and the updated one $\pi^{(t+1)}$. 
	In fact, the arbitrary choice of $\rho$ readily reveals a sort of {\em pointwise} monotoncity for the above range of learning rates, in the sense that 
	$\soft{\V}^{(t+1)}(s) \geq  \soft{\V}^{(t)}(s)$ for all $s\in \cS$. 
	Indeed, this lemma can be viewed as the counterpart of the performance difference lemma in \citet{kakade2002approximately} for the unregularized form.  Lemma~\ref{lemma:dalaopo} also implies the monotonicity of the soft Q-function in $t$, since for any $(\s, \ac)\in \ssp \times \asp$ one has
	\begin{align}
	\soft{\Q}^{(t+1)}(\s, \ac) &= r(\s, \ac) + \disct\exlim{\s' \sim \prob(\cdot|\s, \ac)}{\soft{\V}^{(t+1)}(\s')}  
	 \ge r(\s, \ac) + \disct\exlim{\s' \sim \prob(\cdot|\s, \ac)}{\soft{\V}^{(t)}(\s')} 	= \soft{\Q}^{(t)} (\s, \ac), \label{eq:Q_mono}
	\end{align}
where the equalities follow from the definition \eqref{eq:defn-regularized-Q}, and the inequality follows since $\soft{\V}^{(t+1)}(\s) \ge \soft{\V}^{(t)}(\s)$ for all $\s \in \ssp$ --- a consequence of 
Lemma~\ref{lemma:dalaopo} and the non-negativity of the KL divergence.

\paragraph{A key contraction operator: the soft Bellman optimality operator.} An operator that plays a pivotal role in the theory of dynamic programming \citep{bellman1952theory} is the renowned 
 Bellman optimality operator ${\mathcal{T}}: \real^{|\ssp|| \asp|} \to \real^{|\ssp|| \asp|}$, defined as follows
 \begin{equation}	
	 \forall (s,a)\in \cS\times \cA: \quad {\mathcal{T}}(\Q)(\s, \ac) := r(\s, \ac) + \disct \exlim{\s' \sim \prob(\cdot|\s, \ac)}{\max_{\ac'}{Q(\s', \ac')}} .
\end{equation}
In order to facilitate analysis for entropy-regularized MDPs, we find it particularly fruitful to introduce a ``soft'' Bellman optimality operator $\soft{\mathcal{T}}: \real^{|\ssp|| \asp|} \to \real^{|\ssp|| \asp|}$ as follows
 \begin{equation}	
	 \label{defn:soft_bellman}
	 \forall (s,a)\in \cS\times \cA: \quad \soft{\mathcal{T}}(\Q)(\s, \ac) := r(\s, \ac) + \disct \exlim{\s' \sim \prob(\cdot|\s, \ac)}{\max_{\plcy(\cdot|\s') \in \Delta(\cA)}\exlimBig{\ac' \sim \plcy(\cdot |\s')}{Q(\s', \ac') - \tau \log \plcy(\ac'|\s')}},
\end{equation}
which reduces to  ${\mathcal{T}}$ when $\tau=0$. To see this, observe that
\begin{align}
	{\mathcal{T}}_0(\Q)(\s, \ac) &= r(\s, \ac) + \disct \exlim{\s' \sim \prob(\cdot|\s, \ac)}{\max_{\plcy(\cdot|\s') \in \Delta(\cA)}\exlimbig{\ac' \sim \plcy(\cdot |\s')}{Q(\s', \ac')}} \notag \\
	&=r(\s, \ac) + \disct \exlim{\s' \sim \prob(\cdot|\s, \ac)}{\max_{\ac'}{Q(\s', \ac')}} = \mathcal{T}(\Q)(\s, \ac), \notag
	\end{align}
where the last line follows since the optimal policy is exactly the greedy policy w.r.t.~$Q$ \citep{puterman2014markov}. The operator  $\soft{\mathcal{T}}$ plays a similar role as does the Bellman optimality operator for the unregularized case, whose key properties are summarized below. Similar results have been derived in \citet[Section 3.1]{dai2018sbeed}.

\begin{lemma}[Soft Bellman optimality operator] \label{lemma:xinhuan}  The operator $\soft{\mathcal{T}}$ defined in \eqref{defn:soft_bellman} satisfies the properties below.
\begin{itemize}
\item $\soft{\mathcal{T}}$ admits the following closed-form expression:
	\begin{equation}
		\soft{\mathcal{T}}(\Q)(\s, \ac) = r(\s, \ac) + \disct \exlimBig{\s' \sim \prob(\cdot|\s, \ac)}{\tau \log\prnbig{ \normbig{\exp\prnbig{\Q(\s', \cdot)/\tau}}_1}}. 
	\label{defn:soft_bellman_cf}
	\end{equation}
\item The optimal soft Q-function $\soft{\Q}^\star$ is a fixed point of $\soft{\mathcal{T}}$, namely, 
\begin{equation} \label{defn:soft_bellman_fixpoint}
	\soft{\mathcal{T}} \prnbig{ \soft{\Q}^\star } =  \soft{\Q}^\star.
\end{equation}
\item $\soft{\mathcal{T}}$ is a $\disct$-contraction in the $\ell_{\infty}$ norm, namely, for any $ \Q_1, \Q_2 \in \real^{|\ssp|| \asp|}$ one has
	\begin{equation} \label{eq:soft_bellman_contraction}
		\normbig{\soft{\mathcal{T}} (\Q_1) - \soft{\mathcal{T}} (\Q_2) }_\infty \le \disct \normbig{\Q_1 - \Q_2}_\infty.
	\end{equation}
\end{itemize}
\end{lemma}

\begin{proof} See Appendix~\ref{sec:pflemxinhuan}. \end{proof}

For those familiar with dynamic programming, it should become evident that $\soft{\mathcal{T}}$ inherits many appealing features of the original Bellman optimality operator $\mathcal{T}$. For example, as an immediate application of the $\gamma$-contraction property \eqref{eq:soft_bellman_contraction} and the fixed-point property \eqref{defn:soft_bellman_fixpoint},  the following soft $\Q$-value iteration
	\[
		\Q^{(t+1)}_{\mathsf{svi}} ~=~ \soft{\mathcal{T}}\big(\Q^{(t)}_{\mathsf{svi}}\big), \qquad t\geq 0
	\]
	is guaranteed to converge linearly to the optimal $\soft{\Q}^\star$ with a contraction rate $\gamma$ --- a simple observation consistent with the behavior of  value iteration designed for unregularized MDPs.

\subsection{Analysis of exact entropy-regularized NPG methods}
\label{Sec:pfNPG-eta}

\subsubsection{The SPI case (i.e.~$\eta= (1-\disct)/\tau$)} 

With the help of the soft Bellman optimality operator, we have
  \begin{align}
	  \soft{\Q}^{(t+1)}(\s, \ac) &\overset{\mathrm{(i)}}{=} r(\s, \ac) + \disct \exlim{\s' \sim \prob(\cdot|\s, \ac)}{\soft{\V}^{(t+1)}(\s')} \notag\\
			&\overset{\mathrm{(ii)}}{=} r(\s, \ac) + \disct \exlim{\substack{\s' \sim \prob(\cdot|\s, \ac),\\\ac' \sim \plcy^{(t+1)}(\cdot|\s')}}{-\tau \log\plcy^{(t+1)}(\ac'|\s') + \soft{\Q}^{(t+1)}(\s', \ac')} \notag\\
			&\overset{\mathrm{(iii)}}{\ge} r(\s, \ac) + \disct \exlim{\substack{\s' \sim \prob(\cdot|\s, \ac),\\\ac' \sim \plcy^{(t+1)}(\cdot|\s')}}{-\tau \log\plcy^{(t+1)}(\ac'|\s') + \soft{\Q}^{(t)}(\s', \ac')} \notag\\
			&\overset{\mathrm{(iv)}}{=} r(\s, \ac) + \disct \exlim{\s' \sim \prob(\cdot|\s, \ac)}{\tau \log\prn{ \normbig{\exp\prnbig{\Q^{(t)}(\s', \cdot)/\tau}}_1}} \notag\\
			&\overset{\mathrm{(v)}}{=} \soft{\mathcal{T}} \big( \soft{\Q}^{(t)} \big) (\s, \ac). \label{eq:Q-TQ-mon}
  \end{align}
Here, (i) comes from  the definition \eqref{eq:defn-regularized-Q} of the soft Q-function, (ii) follows from the relation \eqref{eq:regularized-V-to-Q}, (iii) relies on the monotonicity of the soft Q-function (see \eqref{eq:Q_mono}),  (iv) uses the form of $\pi^{(t+1)}$ in \eqref{eq:SPI-update}, whereas (v) makes use of the expression \eqref{defn:soft_bellman_cf}. The inequality \eqref{eq:Q-TQ-mon} further leads to $0\leq \soft{\Q}^\star - \soft{\Q}^{(t+1)} \leq \soft{\Q}^\star - \soft{\mathcal{T}} \big( \soft{\Q}^{(t+1)} )$, and hence
		\begin{align}
			\normbig{\soft{\Q}^\star - \soft{\Q}^{(t+1)}}_\infty & \le \normbig{\soft{\Q}^\star - \soft{\mathcal{T}} \big( \soft{\Q}^{(t)} \big) }_\infty = \normbig{\soft{\mathcal{T}} \big( \soft{\Q}^\star \big) - \soft{\mathcal{T}} \big(\soft{\Q}^{(t)} \big)}_\infty  \le \disct \normbig{\soft{\Q}^\star - \soft{\Q}^{(t)}}_\infty \label{eq:contraction_spi}\\
		 & \leq \disct^{t+1}\normbig{\soft{\Q}^\star - \soft{\Q}^{(0)}}_\infty, \notag
		\end{align}
where the first equality follows from the fixed-point property \eqref{defn:soft_bellman_fixpoint}, and the second inequality is due to the contraction property \eqref{eq:soft_bellman_contraction}. We have thus established linear convergence of $\soft{\Q}^{(t)}$ in $\|\cdot\|_{\infty}$ for this case. 
		
Turning to the log policies, recall that 
\begin{align*}
	\pi^{(t+1)}(\cdot|s) \propto \exp\big( {\soft{\Q}^{(t)}(s,\cdot)}/\tau \big) \qquad  \text{and} \qquad 
	\pi_{\tau}^{\star}(\cdot|s) \propto \exp\big( {\soft{\Q}^{\star}(s,\cdot)}/\tau \big) ,
\end{align*}
where the second relation comes from 
\citet[Eqn.~(12)]{nachum2017bridging}. It then follows from 
 an elementary property of the softmax function (see \eqref{eq:log_pi_diff} in Appendix~\ref{sec:properety_logexp}) that
 \[
 	\big\|\log \plcy^{(t+1)} - \log \plcy_\tau^{\star}\big\|_\infty  \le  \frac{2 }{\tau }\normbig{\soft{\Q}^{(t)} - \soft{\Q}^{\star}  }_\infty \leq \frac{2}{\tau} \disct^{t}\normbig{\soft{\Q}^\star - \soft{\Q}^{(0)}}_\infty,
\]
thus concluding the proof for this case.

\subsubsection{The case with general learning rates} 

We now move to the case with a general learning rate. For the sake of brevity, we shall  denote 
\begin{equation} 
\label{eq:def_alpha}
\alpha := 1 - \frac{\eta \tau}{1 - \disct}.
\end{equation} 
Additionally, it is helpful to introduce an auxiliary sequence $\{\xi^{(t)}\in \mathbb{R}^{|\cS||\cA|}\}$ constructed recursively by
	\begin{subequations}
		\label{defn:xi-t-sequence}
		\begin{align}
		\xi^{(0)}(\s, \ac) :=& \norm{\exp\big( {\soft{\Q}^{\star}(\s, \cdot)}/\tau \big) }_1 \cdot \plcy^{(0)}(\ac | \s), \label{defn:xi-t-sequence-1}\\
			\xi^{(t+1)}(s,a) :=& \big[ \xi^{(t)}(s,a) \big]^{\alpha} \exp\prnBig{\prn{1-\alpha}\frac{\soft{\Q}^{(t)} (s,a)}{\tau} }, \qquad  \forall~(s,a)\in \cS\times \cA, ~t\geq 0. \label{defn:xi-t-sequence-2}
		\end{align}
	\end{subequations}
It is easily seen from the construction \eqref{defn:xi-t-sequence-2} that
\begin{align}
	\soft{Q}^\star - \tau \log \xi^{(t+1)} &  =\soft{Q}^\star - \tau \alpha \log \xi^{(t)} - (1-\alpha)  \soft{Q}^{(t)}	\notag \\
	& = \alpha \prnbig{\soft{Q}^\star - \tau \log \xi^{(t)}} + (1-\alpha) \prnbig{\soft{Q}^\star - \soft{Q}^{(t)}	} \label{eq:linsys_eq2}
\end{align}
and, consequently, 
\begin{align}
	\normbig{ \soft{Q}^\star - \tau \log \xi^{(t+1)} }_{\infty} \leq \alpha \normbig{\soft{Q}^\star - \tau \log \xi^{(t)}}_{\infty} + (1-\alpha) \normbig{\soft{Q}^\star - \soft{Q}^{(t)}	}_{\infty}. \label{eq:linsys_eq22}
\end{align}

\paragraph{Step 1: a linear system that describes the error recursions.} In the case with general learning rates, the estimation error $\big\| \soft{Q}^\star  - \soft{Q}^{(t)}\big\|_{\infty} $ does not contract in the same form as that of soft policy iteration; instead,  it is more succinctly controlled with the aid of an auxiliary quantity $\norm{\soft{Q}^\star - \tau \log \xi^{(t)}}_\infty$. In what follows, we leverage a simple yet powerful technique by describing the dynamics concerning $\big\| \soft{Q}^\star  - \soft{Q}^{(t)}\big\|_{\infty} $ and $\norm{\soft{Q}^\star - \tau \log \xi^{(t)}}_\infty$ via a linear system, whose spectral properties dictate the convergence rate. Towards this, we start with the following key observation, whose proof is deferred to Appendix~\ref{sec:linsys_1_pf}. 
	\begin{lemma} 
	\label{lemma:linsys_1}
	For any learning rate $0 < \eta  \le (1-\disct)/\tau$, the entropy-regularized NPG updates \eqref{eq:entropy_npg} satisfy 
	\begin{align}
	\label{eq:linsys_eq1}
	\normbig{ \soft{Q}^\star  - \soft{Q}^{(t+1)}}_{\infty} \leq \disct \normbig{\soft{Q}^\star - \tau \log \xi^{(t+1)}}_\infty + \disct \alpha^{t+1} \normbig{\soft{Q}^{(0)} - \tau \log \xi^{(0)}}_\infty,	
	\end{align}
	where $\alpha$  is defined in \eqref{eq:def_alpha}. 
	\end{lemma}
	%
If we substitute \eqref{eq:linsys_eq2} into \eqref{eq:linsys_eq1}, it is straightforwardly seen that Lemma~\ref{lemma:linsys_1} is a generalization of the contraction property \eqref{eq:contraction_spi} of soft policy iteration (the case corresponding to $\alpha=0$). Given that Lemma \ref{lemma:linsys_1} involves the interaction of more than one quantities, it is convenient to combine \eqref{eq:linsys_eq22} and \eqref{eq:linsys_eq1}  into the following linear system
\begin{equation} \label{eq:linsys}
x_{t+1} \leq A x_{t} + \disct \alpha^{t+1} y,
\end{equation}
	where
	\begin{align}
		A := \begin{bmatrix}
	\disct (1 - \alpha) & \disct  \alpha \\
	1 - \alpha & \alpha
	\end{bmatrix}, 
		\quad 
		x_{t} := \begin{bmatrix}
	\normbig{\soft{Q}^\star - \soft{Q}^{(t)}}_\infty \\[1ex]
	\normbig{\soft{Q}^\star - \tau \log \xi^{(t)}}_\infty 
	\end{bmatrix} 
		\quad \text{and} \quad 
		y :=  \begin{bmatrix}
	\normbig{\soft{Q}^{(0)} - \tau \log \xi^{(0)}}_\infty \\[0.5ex]
	0
	\end{bmatrix}.
		\label{eq:defn-A-x-y}
	\end{align}
We shall make note of the following appealing features of the rank-1 system matrix $A$:
\begin{equation}
	\label{eq:property_A}
	A = \begin{bmatrix}\disct\\1\end{bmatrix} \begin{bmatrix}1-\alpha,\, \alpha\end{bmatrix}, \qquad\mbox{and}\quad A^t = (1-\eta \tau)^{t-1} A \qquad\forall t\geq 0,
\end{equation}
which relies on the identity $(1-\alpha)\gamma + \alpha = 1-\eta \tau$ (according to the definition \eqref{eq:def_alpha} of $\alpha$).

\begin{remark}
By left multiplying both sides of  \eqref{eq:linsys} by $[1-\alpha, \alpha]$, we obtain
\begin{align*}
&L^{(t+1)} \le (1-\eta\tau) L^{(t)} + \gamma (1-\alpha)\alpha^{t+1}\normbig{\soft{Q}^{(0)} - \tau \log \xi^{(0)}}_\infty,
\end{align*}
where  $L^{(t)} := (1-\alpha)\normbig{\soft{Q}^\star - \soft{Q}^{(t)}}_\infty + \alpha \normbig{\soft{Q}^\star - \tau \log \xi^{(t)}}_\infty$ 
can be viewed as a sort of Lyapunov function. This hints at the intimate connection between our proof and the Lyapunov-type analysis used in system theory. 
\end{remark}

\paragraph{Step 2: characterizing the contraction rate from the linear system.} 
In view of the recursion formula \eqref{eq:linsys} and the non-negativity of $(A, x_t, y)$, it is immediate to deduce that 
	\begin{align}
	x_{t+1} & \leq  A (A x_{t-1} +\gamma \alpha^{t}y) + \gamma \alpha^{t+1} y \notag\\
	& \leq A^{t+1} x_0 + \disct \prn{\alpha^{t+1} I + \alpha^{t}  A + \cdots + \alpha A^t} y  \notag\\
	& = A^{t+1} x_0 + \disct \prn{A^{t+1} - \alpha^{t+1} I }   \prn{\alpha^{-1} A - I}^{-1} y.
		\label{eq:xt-UB-At}
	\end{align}
Here, the last line follows from  the elementary relation
\begin{align*}
\prn{\alpha^{t+1} I + \alpha^{t}  A + \cdots + \alpha A^t} \prn{\alpha^{-1} A - I} & = A^{t+1} - \alpha^{t+1} I 
\end{align*}
and the invertibility of $\alpha^{-1} A - I$ (since $\alpha^{-1} A$ is a rank-1 matrix whose non-zero singular value is larger than 1). 
In addition, the Woodbury matrix inversion formula together with the decomposition \eqref{eq:property_A} yields
	\begin{align}
	\disct \prn{\alpha^{-1} A -  {I}}^{-1} y 
	=
	\gamma\left\{ \left[\begin{array}{cc}
1 & \frac{\alpha}{1-\alpha}\\
\frac{1}{\gamma} & \frac{\alpha}{(1-\alpha)\gamma}
\end{array}\right]-I\right\} y
	= \begin{bmatrix}
	0 & \frac{\disct \alpha}{1-\alpha} \\
	1 & \frac{\disct \alpha + \alpha - \disct }{1 - \alpha}
	\end{bmatrix}  y = \begin{bmatrix}
	0 \\
	\normbig{\soft{Q}^{(0)} - \tau \log \xi^{(0)}}_\infty
	\end{bmatrix},
		\label{eq:inversion-formula-Aalpha}
	\end{align}
which is a non-negative vector. Consequently, this taken together with \eqref{eq:xt-UB-At} gives
	\begin{align}
		{x}_{t+1} & \leq A ^{t+1} \left[ x_0 + \disct \prn{\alpha^{-1} A -  {I}}^{-1} y  \right] -  \alpha^{t+1} \left\{ \gamma \prn{\alpha^{-1} A -  {I}}^{-1} y \right\} \notag \\
		& \leq A ^{t+1} \left[ x_0 + \disct \prn{\alpha^{-1} A -  {I}}^{-1} y  \right] \notag \\
		& = (1 - \eta \tau)^{t} \prn{\begin{bmatrix}\disct\\1\end{bmatrix} \begin{bmatrix}1-\alpha, \alpha\end{bmatrix}}\begin{bmatrix}
		\normbig{\soft{Q}^\star - \soft{Q}^{(0)}}_\infty \\[0.5ex]
		\normbig{\soft{Q}^\star - \tau \log \xi^{(0)}}_\infty + \normbig{\soft{Q}^{(0)} - \tau \log \xi^{(0)}}_\infty
	\end{bmatrix} \notag \\
		& = (1 - \eta \tau)^{t} \left\{(1-\alpha)\normbig{\soft{Q}^\star - \soft{Q}^{(0)}}_\infty + \alpha  \prn{\normbig{\soft{Q}^\star - \tau \log \xi^{(0)}}_\infty + \normbig{\soft{Q}^{(0)} - \tau \log \xi^{(0)}}_\infty} \right\}
		\begin{bmatrix}
		\disct \\
		1
		\end{bmatrix}, \label{eq:intermediate}
	\end{align}
where the third line follows from \eqref{eq:property_A}, \eqref{eq:inversion-formula-Aalpha} and the definition of $x_t$. Further, observe that
	\begin{align}
		&\normbig{\soft{Q}^\star - \tau \log \xi^{(0)}}_\infty + \normbig{\soft{Q}^{(0)} - \tau \log \xi^{(0)}}_\infty - \normbig{\soft{Q}^\star - \soft{Q}^{(0)}}_\infty \notag \\
		& \qquad \qquad \le 2\normbig{\soft{Q}^\star - \tau \log \xi^{(0)}}_\infty 
			 = 2\tau \normbig{\log \soft{\plcy}^\star - \log \plcy^{(0)}}_\infty, 
			 \label{eq:initial_reform}
	\end{align}
where the inequality comes from the triangle inequality, and the last identity follows from \eqref{defn:xi-t-sequence-1}. Substituting this back into \eqref{eq:intermediate}, we obtain
	\begin{equation}
		x_{t+1} \le  (1 - \eta \tau)^{t}  \left\{ \normbig{\soft{Q}^\star - \soft{Q}^{(0)}}_\infty +  2\alpha \tau \normbig{\log \soft{\plcy}^\star - \log \plcy^{(0)}}_\infty \right\} \begin{bmatrix}
	\disct \\
	1
	\end{bmatrix}.
		\label{eq:xtplus1-recursion-bound}
	\end{equation}
	
To finish up, recall that $\pi^{(t)}$ is related to $\xi^{(t)}$ as follows
\begin{align}
	\forall s\in \cS: \qquad \pi^{(t)}(\cdot | s) =  \frac{1}{\|\xi^{(t)} (s,\cdot)\|_1} \xi^{(t)} (s,\cdot) ,
\end{align}
which can be seen by comparing \eqref{defn:xi-t-sequence} with \eqref{eq:entropy_npg}. Therefore, invoking the 
  elementary property of the softmax function (see \eqref{eq:log_pi_diff} in Appendix~\ref{sec:properety_logexp}), we arrive at 
$$\normbig{\log \soft{\plcy}^\star - \log \plcy^{(t+1)}}_\infty \le 2\normbig{\soft{\Q}^\star/\tau - \log \xi^{(t+1)}}_\infty.$$ 
This combined with \eqref{eq:xtplus1-recursion-bound} as well as the definition \eqref{eq:defn-A-x-y} of $x_{t+1}$ 
immediately establishes Theorem~\ref{thm:npg_exact}.

\subsection{Analysis of approximate entropy-regularized NPG methods} \label{sec:Pf-inexact}

We now turn to the convergence properties of approximate entropy-regularized NPG methods --- as claimed in Theorem~\ref{thm:npg_inexact} --- when only inexact policy evaluation $\estsoftQ$ is available (in the sense of \eqref{eq:soft_greedy_approximate}).

\paragraph{Step 1: performance difference accounting for inexact policy evaluation.} 
We first bound the quality of the policy updates~\eqref{eq:soft_greedy_approximate} by examining the difference between $\soft{\V}^{(t+1)}$ and $\soft{\V}^{(t)}$ and how it is impacted by the imperfectness of policy evaluation. This is made precise by the following lemma.

\begin{lemma}[Performance difference of approximate entropy-regularized NPG]
	\label{lemma:inexact_bound}
Suppose that $0<\eta \le (1-\disct)/\tau$. For any state $s_0\in\mathcal{S}$, one has
	\begin{align}
		\label{eq:improvement_inexact}  
		\soft{\V}^{(t)}(s_0) & \le \soft{\V}^{(t+1)}(s_0) + \frac{2}{1 - \disct}\normbig{\estsoftQ - \soft{Q}^{(t)}}_\infty.
	\end{align}
\end{lemma}
	\begin{proof}
		See Appendix~\ref{sec:pflemdalaopo_pert}.
	\end{proof}

The careful reader might already realize that the above lemma is a relaxation of Lemma~\ref{lemma:dalaopo}; in particular, the last term of \eqref{eq:improvement_inexact} quantifies the effect of the approximation error (i.e.~the difference between $\estsoftQ$ and $\soft{\Q}^{(t)}$) upon performance improvement. Under the assumption $\normbig{\estsoftQ - \soft{Q}^{(t)}}_\infty \leq \delta$, repeating the argument of \eqref{eq:Q_mono} reveals that the  soft $Q$-function estimates are not far from being monotone in $t$, in the sense that
\begin{align} 
	\label{eq:almost_mono_Q}
	\forall (s,a)\in\mathcal{S}\times\mathcal{A}: \qquad
	\soft{Q}^{(t)}(s,a) - \soft{Q}^{(t+1)}(s,a) & = \disct \exlim{\s' \sim \prob(\cdot|\s, \ac)}{\soft{V}^{(t)}(s') - \soft{V}^{(t+1)}(s')} \leq \frac{2\disct \delta}{1 - \disct} .
\end{align}

\paragraph{Step 2: a linear system accounting for inexact policy evaluation.} With the assistance of \eqref{eq:almost_mono_Q}, it is possible to construct a linear system --- similar to the one built in Section~\ref{Sec:pfNPG-eta} --- that takes into account inexact policy evaluation. Towards this end, we adopt a similar approach as in \eqref{defn:xi-t-sequence}
by introducing the following auxiliary sequence $\widehat{\xi}^{(t)}$ defined recursively using $\estsoftQ$:
	\begin{subequations}
		\label{defn:hat-xi-t-sequence}
		\begin{align}
		\widehat{\xi}^{(0)}(\s, \ac) :=& \norm{\exp\big( {\soft{\Q}^{\star}(\s, \cdot)}/\tau \big) }_1 \cdot \plcy^{(0)}(\s, \ac), \label{defn:hat-xi-t-sequence-1}\\
		\widehat{\xi}^{(t+1)}(s,a) :=& \brk{\widehat{\xi}^{(t)}(s,a)}^{\alpha} \exp\prnBig{\prn{1-\alpha}\frac{\estsoftQ (s,a)}{\tau} }, \qquad  \forall~(s,a)\in \cS\times \cA, ~t\geq 0, \label{defn:hat-xi-t-sequence-2}
		\end{align}
	\end{subequations}
where $\alpha := 1- \frac{\eta\tau}{1-\gamma}$ as before.

We claim that the following linear system tracks the error dynamics of the policy updates:
\begin{equation} \label{eq:inexact_ls}
z_{t+1} \leq  B z_t + b,
\end{equation}
 where
\begin{align} 
	\label{eq:expressions_inexact_ls}
	B := \begin{bmatrix}
		\disct(1-\alpha) & \disct\alpha & \disct \alpha \\
	1 - \alpha & \alpha & 0 \\
	0 & 0 &  \alpha
	\end{bmatrix}, \; 
	z_t := \begin{bmatrix}
	\big\| \soft{Q}^\star - \soft{Q}^{(t)} \big\|_{\infty} \\[1ex] 
	\big\|\soft{Q}^\star - \tau \log \widehat{\xi}^{(t)}\big\|_\infty \\[1ex]  
	-\min_{s,a} \big(\soft{Q}^{(t)}(s,a) - \tau\log \widehat{\xi}^{(t)}(s,a) \big)
	\end{bmatrix}, \; 
	b := (1-\alpha)\delta\begin{bmatrix}
\disct\prn{2+\frac{2\disct}{\eta\tau}} \\
1 \\
1 + \frac{2\disct }{\eta\tau}
\end{bmatrix}.
\end{align}
Here, the system matrix $B$ (in particular its eigenvalues) governs the contraction rate, while the term $b$ captures the error introduced by inexact policy evaluation.  Theorem~\ref{thm:npg_inexact} then follows by carrying out a similar analysis argument as in Section~\ref{Sec:pfNPG-eta} to characterize the error dynamics. Details are postponed to Appendix~\ref{sec:proof_approximate_npg}.

\subsection{Analysis of local quadratic convergence}
\label{sec:quadratic-sketch}

We now sketch the proof of Theorem~\ref{thm:spi_sp}, which establishes local quadratic convergence of SPI.

\paragraph{Step 1: characterization of the sub-optimality gap.} 
Lemma~\ref{lemma:dalaopo} bounds the performance improvement of SPI by the KL divergence between the current policy $\pi^{(t)}$ and the updated policy $\pi^{(t+1)}$. Interestingly, the type of KL divergence can  be further employed to bound  the sub-optimality gap for each iteration.
 	\begin{lemma}[Sub-optimality gap] 
	\label{lemma:xiaoqie}
	Suppose that $\eta = (1-\disct)/\tau$. For any distribution $\rho$, one has
	\begin{align*}
		\soft{\V}^{\star}(\rho) - \soft{\V}^{(t)}(\rho) \le \frac{1}{\eta}\exlim{\s \sim d_{\rho}^{\plcy^\star_\tau}}{ {\KL \Big( \plcy^{(t)}(\cdot|\s) \,\big\|\, {\plcy}^{(t+1)}(\cdot|\s) \Big)}} .
	\end{align*} 
	%
	\end{lemma}
	\begin{proof} This result has appeared in \citet[Eqn.~(486)]{mei2020global}. For completeness we include a proof in Appendix~\ref{sec:pflemxiaoqie}. \end{proof}
	In words, Lemma~\ref{lemma:xiaoqie} formalizes the connection between the sub-optimality gap (w.r.t.~the optimal soft value function) and the proximity of the two consecutive policy iterates. 	As reflected by this lemma,  if the current and the updated policies do not differ by much (which indicates that the algorithm might be close to convergence), then 
the current estimate of the soft value function is close to optimal. 

\paragraph{Step 2: a contraction property.}
The importance of the above two lemmas is made apparent by the following contraction property when $ \eta = (1-\disct)/\tau$: 
	\begin{align}
		\soft{\V}^{\star}(\rho) - \soft{\V}^{(t+1)}(\rho) &= \soft{\V}^{\star}(\rho) - \soft{\V}^{(t)}(\rho) + \Big( \soft{\V}^{(t)}(\rho) - \soft{\V}^{(t+1)}(\rho) \Big) \notag\\
		& \overset{\mathrm{(i)}}{=} \soft{\V}^{\star}(\rho) - \soft{\V}^{(t)}(\rho)  - \frac{1}{\eta} \exlim{\s \sim d_{\rho}^{(t+1)}}{  \KL\prn{\plcy^{(t)}(\cdot |\s) \,\big\|\, \plcy^{(t+1)}(\cdot | \s)}}\notag\\
		& \overset{\mathrm{(ii)}}{\leq} \soft{\V}^{\star}(\rho) - \soft{\V}^{(t)}(\rho)  -  \frac{1}{\eta} \norm{\frac{d_\rho^{\plcy^\star_\tau}}{d_{\rho}^{(t+1)}}}^{-1}_\infty \exlim{\s \sim d_{\rho}^{\plcy^\star_\tau}}{ {\KL \Big( \plcy^{(t)}(\cdot|\s) \,\big\|\, {\plcy}^{(t+1)}(\cdot|\s) \Big)}}  \notag\\
		& \overset{\mathrm{(iii)}}{\le}   \soft{\V}^{\star}(\rho) - \soft{\V}^{(t)}(\rho) - \norm{\frac{d_\rho^{\plcy^\star_\tau}}{d_{\rho}^{(t+1)}}}^{-1}_\infty \Big( \soft{\V}^{\star}(\rho) - \soft{\V}^{(t)}(\rho) \Big) \notag\\
		& =   \left(1 - \norm{\frac{d_\rho^{\plcy^\star_\tau}}{d_{\rho}^{(t+1)}}}^{-1}_\infty  \right) \prn{\soft{\V}^{\star}(\rho) - \soft{\V}^{(t)}(\rho)} .
	\label{eq:core}
	\end{align}
Here, (i) arises from Lemma~\ref{lemma:dalaopo},  (ii) employs the pre-factor $\big\| {d_\rho^{\plcy^\star_\tau}}/{d_{\rho}^{(t+1)}} \big\|^{-1}_\infty $ to accommodate the change of distributions, whereas (iii) follows from Lemma \ref{lemma:xiaoqie}. 


\paragraph{Step 3: super-linear convergence in the small-$\epsilon$ regime.}
The contraction property 
\eqref{eq:core} implies that $\soft{\V}^{(t+1)}(\rho)$ converges {\em super-linearly} to $\soft{\V}^{\star}$, once $\plcy^{(t)}$ gets sufficiently close to $\plcy^{\star}_\tau$. In fact, once the ratio $d_\rho^{(t+1)}/d_{\rho}^{\pi_{\tau}^{\star}}$ becomes sufficiently close to 1, the contraction factor $1 -   \big\| {d_\rho^{\plcy^\star_\tau}}/{d_{\rho}^{(t+1)}} \big\|_\infty^{-1}$ in \eqref{eq:core} is approaching 0, thereby accelerating convergence. This observation underlies Theorem~\ref{thm:spi_sp}, whose complete analysis is postponed to Appendix~\ref{sec:superlinear}.

	\section{Discussions}

This paper establishes non-asymptotic convergence of entropy-regularized natural policy gradient methods, providing theoretical footings for the role of entropy regularization in guaranteeing fast convergence. Our analysis opens up several directions for future research; we close the paper by sampling a few of them.     
 
\begin{itemize}

\item {\em Extended analysis of policy gradient methods with inexact gradients.} It would be of interest to see whether our analysis framework can be applied to improve the theory of policy gradient methods \citep{mei2020global} to accommodate the case with inexact policy gradients.

\item {\em Finite-sample analysis in the presence of sample-based policy evaluation.} Another natural extension is towards understanding the sample complexity of entropy-regularized NPG methods when the value functions are estimated using rollout trajectories (see e.g.~\citet{kakade2002approximately,agarwal2019optimality,shani2019adaptive}), or using bootstrapping (see e.g.~\citet{xu2020non,haarnoja2018soft,wu2020finite}).

\item {\em Function approximation.} The current work has been limited to the tabular setting. It would certainly be interesting, and fundamentally important, to understand entropy-regularized NPG methods in conjunction with function approximation; see \citet{sutton2000policy, agarwal2019reinforcement,agarwal2019optimality} for a few representative scenarios.

	\item {\em Beyond softmax parameterization.} The current paper has been devoted to softmax parameterization, which enables a concise and NPG update rule. 
A couple of other parameterization schemes have been proposed for (vanilla) PG methods as well \citep{agarwal2019reinforcement,agarwal2019optimality,bhandari2019global,bhandari2020note}, e.g.~vanilla parameterization (paired with proper projection onto the probability simplex in each iteration), log-linear parameterization, and neural softmax parameterization. 
Unfortunately, the analysis in our paper relies heavily on the softmax NPG update rule, and does not immediately extend to other parameterization. It would be of great importance to  establish convergence guarantees that accommodate other parameterizations of practical interest.


\end{itemize}

%


	\section*{Acknowledgments}
The authors are grateful to anonymous reviewers for helpful suggestions, particularly for bringing \citet{dai2018sbeed} to our attention.
S.~Cen and Y.~Chi are supported in part by the grants ONR N00014-18-1-2142 and N00014-19-1-2404, ARO W911NF-18-1-0303, NSF CCF-1806154, CCF-1901199 and CCF-2007911.
C.~Cheng is supported by the William R. Hewlett Stanford graduate fellowship.
Y.~Wei is supported in part by the NSF grants CCF-2007911 and DMS-2015447. 
Y.~Chen is supported in part by the grants AFOSR YIP award FA9550-19-1-0030,
ONR N00014-19-1-2120, ARO YIP award W911NF-20-1-0097, ARO W911NF-18-1-0303, NSF CCF-1907661, IIS-1900140 and DMS-2014279,  
and the Princeton SEAS Innovation Award.

	\bibliographystyle{apalike}
	\bibliography{bibfileRL}
		
	\appendix

\section{Preliminaries}
\label{sec:preliminary}

\subsection{Derivation of entropy-regularized NPG methods} \label{sec:derivations_npg}

This subsection establishes the equivalence between the update rules \eqref{eq:NPG-original} and \eqref{eq:entropy_npg}. 
Such derivations are inherently similar to the ones for the NPG update rule (without entropy regularization) (see, e.g., \citet{agarwal2019reinforcement}); we provide the proof here for pedagogical reasons. 

First of all, let us follow the convention to introduce the advantage function $A_\tau^{\pi}:  \cS \times \cA \rightarrow \real$ of a policy $\pi$ w.r.t.~the entropy-regularized value function: 
 \begin{equation}
	\forall (s,a)\in \cS \times \cA: \qquad \soft{\A}^{\plcy}(\s,\ac) := \soft{\Q}^{\plcy}(\s, \ac)-\tau\log \plcy(\ac|\s) - \soft{\V}^{\plcy}(\s) \label{eqn:asofttau}
\end{equation} 
with $\soft{\Q}^{\plcy}$ defined in \eqref{eq:defn-regularized-Q}, 
which reflects the gain one can harvest by executing action $a$  instead of following the policy $\pi$ in state $s$. This advantage function plays a crucial role in the calculation of policy gradients, due to the following fundamental relation (see Appendix~\ref{sec:proof-lemma:derivative-grad} for the proof):
\begin{lemma}
\label{lemma:derivative-grad}
	Under softmax parameterization \eqref{eq:definition_softmax}, the gradient of the regularized value function satisfies
\begin{subequations}
\begin{align}
	\frac{\partial \soft{\V}^{\plcy_\prm}(\rho)}{\partial \prm(\s, \ac)} &= \frac{1}{1-\disct} d_\rho^{\plcy_\prm}(\s)\cdot \plcy_\prm (\ac|\s) \cdot \soft{\A}^{\plcy_\prm}(\s,\ac) ;\label{eqn:derivative} \\
	\left[ \big( \mathcal{F}_\rho^{\prm} \big)^\dagger \nabla_{\theta} \soft{\V}^{\plcy_\prm}(\rho)\right] (\s,\ac) &= \frac{1}{1-\gamma}\soft{\A}^{\plcy_\prm}(\s,\ac) + c(\s)
		\label{eq:search-direction-NPG}
\end{align}
\end{subequations}
	for any $(s,a)\in \cS \times \cA$, where $c(s):=\sum_{a}\pi_{\theta}(a|s)w_{s,a}$ is some function depending only on $s$. 
\end{lemma} 


	It is worth highlighting that the search direction of NPG, given in \eqref{eq:search-direction-NPG}, is invariant to the choice of $\rho$. With the above calculations in place, it is seen that for any $s\in \cS$, the regularized NPG update rule \eqref{eq:NPG-original} results in a policy update as follows
	\begin{align*}
		\pi^{(t+1)}(a|s) 
		& \overset{\mathrm{(i)}}{\propto}\, \exp\left(\theta^{(t+1)}(s,a)\right) \overset{\mathrm{(ii)}}{=} \exp\left(\theta^{(t)}(s,a)
		  + \eta\left[\big(\mathcal{F}_{\rho}^{\prm^{(t)}}\big)^{\dagger}\nabla_{\theta}\soft{\V}^{(t)}(\rho)\right](\s,\ac)\right)\\
		& \overset{\mathrm{(iii)}}{\propto} \exp\left(\theta^{(t)}(s,a)+\frac{\eta}{1-\gamma}A_{\tau}^{(t)}(\s,\ac)\right)\\
 		& \overset{\mathrm{(iv)}}{\propto} \pi^{(t)}(a|s)\exp\left(\frac{\eta}{1-\gamma}Q_{\tau}^{(t)}(\s,\ac)-\frac{\eta\tau}{1-\gamma}\log\pi^{(t)}(a|s)\right)\\
 		& =\left(\pi^{(t)}(a|s)\right)^{1-\frac{\eta\tau}{1-\gamma}}\exp\left(\frac{\eta}{1-\gamma}Q_{\tau}^{(t)}(\s,\ac)\right) .
	\end{align*}
	where we use $A_{\tau}^{(t)}$ to abbreviate $A_{\tau}^{\pi^{(t)}}$. 
	Here, (i) uses the definition of the softmax policy, (ii) comes from the update rule \eqref{eq:NPG-original}, (iii) is a consequence of \eqref{eq:search-direction-NPG} (since $c(\cdot)$ does not depend on $a$), whereas (iv) results from the definition \eqref{eqn:asofttau} and the fact that $\soft{\V}^{\plcy}(\cdot)$ is not dependent on $a$. This validates the equivalence between \eqref{eq:NPG-original} and \eqref{eq:entropy_npg}.

\subsection{Basic facts about the function $\log ( \norm{\exp({\theta})}_1 )$} \label{sec:properety_logexp}

In the current paper, we often encounter the function $\log \big( \norm{\exp({\theta})}_1 \big) :=  \log \big( \sum\nolimits_{1\leq a\leq |\cA|} \exp(\theta_a) \big)$
for any vector $\theta = [\theta_a]_{1\leq a\leq |\cA|} \in \mathbb{R}^{|\cA|}$. 
To facilitate analysis, we single out several basic properties concerning this function, which will be used multiple times when establishing our main results. For notational convenience, we denote by $\pi_{\theta}\in \mathbb{R}^{|\cA|}$ the softmax transform of $\theta$ such that
\begin{align}
	\pi_\theta (a) = \frac{\exp(\theta_a)}{\sum_{1\leq j\leq |\cA|} \exp(\theta_j)}, \qquad 1\leq a\leq |\cA|. 
	\label{eq:defn-pi-theta-a}
\end{align}
By straightforward calculations, the gradient of the function $\log \big( \norm{\exp({\theta})}_1 \big)$ is given by
		\begin{align} 
			\nabla_{\theta} \log \big( \norm{\exp({\theta})}_1 \big)  &=  \frac{1}{\norm{\exp({\theta})}_1} \exp({\theta}) = \plcy_\theta; \label{eq:gradient_logexp} .
		\end{align}		 

		
\paragraph{Difference of log policies.}
		In the analysis, we often need to control the difference of two policies, towards which the following bounds prove useful. 
		To begin with, the mean value theorem reveals a Lipschitz continuity property (w.r.t.~the $\ell_{\infty}$ norm): for any $\theta_1,\theta_2\in \mathbb{R}^{|\cA|}$, 	
		\begin{align}
			\abs{\log \big( \|\exp(\theta_1)\|_1 \big) - \log \big( \|\exp(\theta_2)\|_1\big) } &= \abs{\left\langle \theta_1 - \theta_2, \nabla_{\theta} \log \big( \norm{\exp(\theta)}_1  \big) |_{\theta=\theta_c} \right \rangle} \nonumber \\
			& \le \norm{\theta_1 - \theta_2}_\infty\| \nabla_{\theta} \log \big( \norm{\exp(\theta)}_1 \big) |_{\theta=\theta_c} \|_1 =\norm{\theta_1 - \theta_2}_\infty ,\label{eq:bound_logexp_diff}
		\end{align}
		where $\theta_c$ is a certain convex combination of $\theta_1$ and $\theta_2$, and the second line relies on \eqref{eq:gradient_logexp}. In addition, for any two vectors $\plcy_{\prm_1}$ and $\plcy_{\prm_2}$ defined w.r.t.~$\prm_1,\prm_2\in \mathbb{R}^{|\cA|}$ (see~\eqref{eq:defn-pi-theta-a}), one has
		\begin{equation}
			\norm{\log \plcy_{\prm_1} - \log \plcy_{\prm_2}}_\infty \leq 2\norm{\prm_1 - \prm_2}_\infty,
			\label{eq:log_pi_diff}
		\end{equation}
		where $\log(\cdot)$ denotes entrywise operation.   To justify \eqref{eq:log_pi_diff}, we observe from the definition \eqref{eq:defn-pi-theta-a} that
		\begin{align*}
			\norm{\log \plcy_{\prm_1} - \log \plcy_{\prm_2}}_\infty 
			&  \leq   \norm{\prm_1 - \prm_2 }_\infty + \Big| \log \big( \|\exp(\theta_1)\|_1 \big) -\log \big( \|\exp(\theta_2)\|_1 \big) \Big| 
			 \le 2  \norm{\prm_1 - \prm_2 }_\infty ,   
		\end{align*}
where  the last inequality is a consequence of \eqref{eq:bound_logexp_diff}. 

	\section{Proof for the bandit case (Proposition~\ref{thm:bandit})}
\label{Sec:Pfbandit}

We start by defining an auxiliary sequence $\xi^{(t)}\in \real^{|\asp|}$ $(t\geq 0)$ recursively as follows
\begin{align*}
\xi^{(0)} &:= \norm{\exp(r /\tau)}_1 \cdot \plcy^{(0)}, \\
	\xi^{(t+1)}(\ac) &:= \big( \xi^{(t)}(\ac) \big) ^{1-\tau\eta}\exp\big( \eta r(\ac) \big),\qquad a\in \cA.
\end{align*}
When combined with \eqref{eq:bandit_update}, it is easily seen that $\plcy^{(t)}(\cdot) \propto \xi^{(t)}(\cdot)$ and, as a result, 
$\plcy^{(t)} = \xi^{(t)}/ \big\| \xi^{(t)} \big\|_1.$

By construction, the auxiliary sequence satisfies the following property
\begin{align*}
	\log \big( \xi^{(t+1)}(\ac) \big) - r(\ac)/\tau &= \prn{1-\tau \eta} \log \big( \xi^{(t)}(\ac) \big) + \eta r(\ac)- r(\ac)/\tau \\
	& =\prn{1-\tau \eta}\prn{\log \big( \xi^{(t)}(\ac) \big) - r(\ac)/\tau},
\end{align*}
thus indicating that
\begin{equation}
	\label{eq:bandit_xi_contraction}
	\norm{\log  \xi^{(t)}  -  r/\tau }_\infty \le (1-\tau\eta)^t \norm{\log   \xi^{(0)}   -  r/\tau}_\infty.
\end{equation}
This taken together with the optimal policy $\plcy^{\star}_\tau =\mathsf{softmax} (r/\tau) \propto \exp\prn{r/\tau}$ leads to
\begin{align*}
\norm{\log \plcy^{(t)} - \log \plcy^\star_\tau}_\infty \le 2\norm{\log\xi^{(t)}- {r}/\tau}_{\infty}
& \le 2(1-\tau\eta)^t \norm{\log \xi^{(0)} -  {r}/\tau}_\infty\\
	&= 2(1-\tau\eta)^t \norm{\log \plcy^{(0)} + \big( \log \norm{\exp( r/\tau)}_1 \big) \cdot \mathbf{1} -  {r}/\tau}_\infty\\
&= 2(1-\tau\eta)^t \norm{\log \plcy^{(0)} - \log \plcy^\star_\tau}_\infty,
\end{align*}
where the first line follows from the inequality~\eqref{eq:log_pi_diff}, the second line follows from the expression~\eqref{eq:bandit_xi_contraction}, whereas the last line follows from the form of $\plcy^{\star}_\tau$.
We have thus completed the proof of Proposition~\ref{thm:bandit}.

\section{Proof for key lemmas}

\subsection{Proof of Lemma~\ref{lemma:dalaopo}}
\label{sec:pflemdalaopo}

To begin with, the regularized NPG update rule (see \eqref{eq:entropy_npg} in Algorithm~\ref{alg:entropy-npg}) indicates that 
\begin{align}
\log \plcy^{(t+1)}(\ac|\s) = \prn{1-\frac{\eta\tau}{1-\disct}}\log \plcy^{(t)}(\ac|\s) + \frac{\eta}{1-\disct}\soft{\Q}^{(t)}(\s, \ac) - \log Z^{(t)}(\s),
\label{eq:policy-tplus1-t}
\end{align}
where $Z^{(t)}$ is some quantity depending only on the state $\s $ (but not the action $a$). Rearranging terms gives
\begin{equation}
\label{eq:value_decomp_aid}
-\tau \log\plcy^{(t)}(\ac|\s) + \soft{\Q}^{(t)}(\s, \ac) = \frac{1-\disct}{\eta}\prn{\log \plcy^{(t+1)}(\ac|\s) - \log \plcy^{(t)}(\ac|\s)} + \frac{1-\disct}{\eta} \log Z^{(t)}(\s).
\end{equation}
This in turn allows us to express $\soft{\V}^{(t)}(\s _0)$ for any $s_0\in \cS$ as follows
\begin{align}
\soft{\V}^{(t)}(\s _0) 
& = \exlim{{\ac_0 \sim \plcy^{(t)}(\cdot | \s _0)}}{-\tau \log \plcy^{(t)}(\ac_0|\s _0) + \soft{Q}^{(t)}(\s _0, \ac_0)} \notag\\
& = \exlim{{\ac_0 \sim \plcy^{(t)}(\cdot | \s _0)}}{ \frac{1-\disct}{\eta} \log Z^{(t)}(\s_0)} + \exlim{{\ac_0 \sim \plcy^{(t)}(\cdot | \s _0)}}{ \frac{1-\disct}{\eta}\prn{\log \plcy^{(t+1)}(\ac_0|\s_0) - \log \plcy^{(t)}(\ac_0|\s_0)} } \notag\\	
&=  \frac{1-\disct}{\eta} \log Z^{(t)}(\s_0) - \frac{1-\gamma}{\eta} \KL\prn{\plcy^{(t)}(\cdot |\s _0) \,\big\|\, \plcy^{(t+1)}(\cdot | \s _0)} \notag\\
&= \exlim{{\ac_0 \sim \plcy^{(t+1)}(\cdot | \s _0)}}{ \frac{1-\disct}{\eta} \log Z^{(t)}(\s_0)} - \frac{1-\gamma}{\eta} \KL\prn{\plcy^{(t)}(\cdot |\s _0) \,\big\|\, \plcy^{(t+1)}(\cdot | \s _0)}  , \label{eq:proof_another_fancy_middle_step}
\end{align}
where the first identity makes use of the definitions \eqref{defn:V-tau} and \eqref{eq:defn-regularized-Q}, the second line follows from \eqref{eq:value_decomp_aid}, the third line relies on the definition of the KL divergence, and the last line follows since $Z^{(t)}(s)$ does not depend on $a$. Invoking  \eqref{eq:value_decomp_aid} again to rewrite $\log Z^{(t)}(\s_0)$ appearing in the first term of \eqref{eq:proof_another_fancy_middle_step}, we reach 
\begin{align}
\soft{\V}^{(t)}(\s _0) &= \exlim{{\ac_0 \sim \plcy^{(t+1)}(\cdot | \s _0)}}{-\tau \log \plcy^{(t+1)}(\ac_0|\s _0) + \soft{Q}^{(t)}(\s _0, \ac_0) + \prn{\tau - \frac{1-\disct}{\eta}}\prn{\log \plcy^{(t+1)}(\ac_0|\s_0) - \log \plcy^{(t)}(\ac_0|\s_0)}}\notag\\
&\qquad \qquad - \frac{1-\gamma}{\eta} \KL\prn{\plcy^{(t)}(\cdot |\s _0) \,\big\|\, \plcy^{(t+1)}(\cdot | \s _0)}\notag\\
=& \exlim{{\ac_0 \sim \plcy^{(t+1)}(\cdot | \s _0)}}{-\tau \log \plcy^{(t+1)}(\ac_0|\s _0) + \soft{Q}^{(t)}(\s _0, \ac_0)} + \prn{\tau - \frac{1-\disct}{\eta}}\KL\prn{\plcy^{(t+1)}(\cdot |\s _0)  \,\big\|\,  \plcy^{(t)}(\cdot | \s _0)}\notag\\
& \qquad \qquad - \frac{1-\gamma}{\eta} \KL\prn{\plcy^{(t)}(\cdot |\s _0) \,\big\|\, \plcy^{(t+1)}(\cdot | \s _0)}\notag\\
=& \exlim{\substack{\ac_0 \sim \plcy^{(t+1)}(\cdot | \s _0),\\\s _1\sim \prob(\cdot | \s _0, \ac_0)}}{-\tau \log \plcy^{(t+1)}(\ac_0|\s _0) + r(\s _0, \ac_0) + \disct\soft{\V}^{(t)}(\s _1)} \notag\\
& \qquad - \prn{\frac{1-\disct}{\eta}-\tau}\KL\prn{\plcy^{(t+1)}(\cdot |\s _0)  \,\big\|\,  \plcy^{(t)}(\cdot | \s _0)}
- \frac{1-\gamma}{\eta} \KL\prn{\plcy^{(t)}(\cdot |\s _0) \,\big\|\, \plcy^{(t+1)}(\cdot | \s _0)} ,
\label{eq:expansion-Vt-s0}
\end{align}
where the second line uses the definition of the KL divergence, and the third line expands $\soft{Q}^{(t)}$ using the definition \eqref{eq:defn-regularized-Q}.

To finish up, applying the above relation \eqref{eq:expansion-Vt-s0} recursively to expand $V_\tau^{(t)}(s_i)$ ($i\geq 1$), we arrive at
\begin{align}
\soft{\V}^{(t)}(\s _0)	 & = \mathop\mathbb{E}\limits_{\substack{\ac_i \sim \plcy^{(t+1)}(\cdot | \s _i),\\\s _{i+1}\sim \prob(\cdot | \s _i, \ac_i), \forall i\geq 0}} \Bigg[   \sum_{i=0}^{\infty} \gamma^i \left\{ r(s_i, a_i) - {\tau \log \plcy^{(t+1)}(\ac_i|\s _i) } \right\} \notag\\
\qquad & \qquad \qquad  - \sum_{i=0}^{\infty}  \gamma^i  \left\{ \prn{\frac{1-\disct}{\eta}-\tau} \KL\prn{\plcy^{(t+1)}(\cdot |\s _i) \,\big\|\, \plcy^{(t)}(\cdot | \s _i)} + \frac{1-\gamma}{\eta} \KL\prn{\plcy^{(t)}(\cdot |\s _i) \,\big\|\, \plcy^{(t+1)}(\cdot | \s _i)}  \right\} \Bigg] \notag \\	
& = \soft{\V}^{(t+1)}(\s _0) -  \mathop\mathbb{E}\limits_{\s \sim d_{\s _0}^{(t+1)}}\left[ \prn{\frac{1}{\eta}-\frac{\tau}{1-\disct}} \KL(\plcy^{(t+1)}(\cdot|\s) \,\big\|\, \plcy^{(t)}(\cdot|\s)) + \frac{1}{\eta} \KL\prn{\plcy^{(t)}(\cdot |\s) \,\big\|\, \plcy^{(t+1)}(\cdot | \s)} \right],	
\label{eqn:useful}
\end{align}	
where the second line follows since the regularized value function $V^{(t+1)}_{\tau}$ can be viewed as the value function of $\pi^{(t+1)}$ with adjusted rewards $r_{\tau}^{(t+1)}(s,a) := r(s,a) - \tau \log \pi^{(t+1)}(a| s)$. 
Averaging the initial state $s_0$ over the distribution $\rho$ concludes the proof.

\subsection{Proof of Lemma~\ref{lemma:xinhuan}}
\label{sec:pflemxinhuan}
In the sequel, we prove each claim in Lemma~\ref{lemma:xinhuan} in order.
  
	 \paragraph{Proof  of Eqn.~\eqref{defn:soft_bellman_cf}.}
	 Jensen's inequality tells us that: for any $s\in \cS$ one has
	 \begin{align}
	 \exlimBig{\ac \sim \plcy(\cdot|\s)}{{{\Q}(\s, \ac) - \tau \log \plcy(\ac|\s)}}
	 &= \tau\sum_{a}\pi(a|s)\log \left( \frac{\exp\big(Q(s,a)/\tau\big)}{\pi(a|s)} \right) \notag\\
	 &\leq \tau\log\left(\sum_{a}\pi(a|s)\frac{\exp\big(Q(s,a)/\tau\big)}{\pi(a|s)}\right) \notag\\
	 &= \tau\log \left( \sum_{\ac}\exp\prnbig{ {Q}(\s, \ac)/\tau} \right) = \tau\log \big( \normbig{\exp\prnbig{ {Q}(\s, \cdot)/\tau}}_{1} \big),
	 \end{align}
	 where in the second line,  equality is attained if $\plcy(\cdot|\s) \propto \exp(\Q(\s, \cdot)/\tau)$. This immediately gives rise to
	  \begin{align*}
		  \soft{\mathcal{T}}(\Q)(\s, \ac) & = r(\s, \ac) + \disct \exlim{\s' \sim \prob(\cdot|\s, \ac)}{\max_{\plcy(\cdot|\s') \in \Delta(\cA)}\exlimBig{\ac' \sim \plcy(\cdot |\s')}{Q(\s', \ac') - \tau \log \plcy(\ac'|\s')}}\\
	 &= r(\s, \ac) + \disct \exlimBig{\s' \sim \prob(\cdot|\s, \ac)}{\tau \log\prnbig{ \normbig{\exp\prnbig{\Q(\s', \cdot)/\tau}}_1}}.
	 \end{align*}
	 
	 \paragraph{Proof of Eqn.~\eqref{defn:soft_bellman_fixpoint}.}
	 Recall the characterization of  $\soft{\plcy}^\star$ and $\soft{\V}^\star$ established in \cite{nachum2017bridging}:
	 \begin{subequations}
	 \begin{align}
	 	\soft{\plcy}^\star(\ac|\s) &= \exp\prn{\frac{\soft{\Q}^\star(\s, \ac) - \soft{\V}^\star(\s)}{\tau}},\\
		 \soft{\V}^\star (\s) &= \tau \log \prnbig{ \normbig{\exp\prnbig{{\soft{\Q}^\star(\s, \cdot)}/{\tau}}}_1 }.
	 \label{eq:opt_consist}
	 \end{align}
	 \end{subequations}
	Substitution into the expression \eqref{defn:soft_bellman_cf} tells us that for any $(s,a)\in \cS\times \cA$, 
   \begin{align*}
	   \soft{\mathcal{T}} \prnbig{ \soft{\Q}^\star } (\s, \ac) &= r(\s, \ac) + \disct \exlimBig{\s' \sim \prob(\cdot|\s, \ac)}{\tau \log\prnbig{ \normbig{\exp\prnbig{\soft{\Q}^\star(\s', \cdot)/\tau}}_1}}\\
	   &=  r(\s, \ac) + \disct \exlimbig{\s' \sim \prob(\cdot|\s, \ac)}{\soft{\V}^\star (\s')}\\
	   &= \soft{\Q}^\star(\s, \ac),
	\end{align*}	 
where the second line results from \eqref{eq:opt_consist}, and the last line follows from the definition of the soft Q-function.

	\paragraph{Proof of Eqn.~\eqref{eq:soft_bellman_contraction}.} Invoking again the expression \eqref{defn:soft_bellman_cf}, we can demonstrate that for any $Q_1$ and $Q_2$,
		\begin{align*}
			\Big| \soft{\mathcal{T}} (\Q_1) &(\s, \ac) - \soft{\mathcal{T}} (\Q_2) (\s, \ac) \Big|   \\
			&=  \abs{\disct\exlimBig{\s' \sim \prob(\cdot|\s, \ac)}{\tau \log\prnbig{ \normbig{\exp\prn{\Q_1(\s', \cdot)/\tau}}_1}} - \disct \exlimBig{\s' \sim \prob(\cdot|\s, \ac)}{\tau \log\prnbig{ \normbig{\exp\prn{\Q_2(\s', \cdot)/\tau}}_1}}} \\
			&= \disct\tau  \abs{\exlimBig{\s' \sim \prob(\cdot|\s, \ac)}{\log\prnbig{ \normbig{\exp\prn{\Q_1(\s', \cdot)/\tau}}_1} - \log\prnbig{ \normbig{\exp\prn{\Q_2(\s', \cdot)/\tau}}_1}}} \\
			&\le \disct\tau\norm{\Q_1/\tau - \Q_2/\tau}_\infty \\
			&= \disct \norm{\Q_1 - \Q_2}_\infty 
		\end{align*}
holds for all $(s,a)\in \cS\times \cA$, where the inequality follows from the Lipschitz property \eqref{eq:bound_logexp_diff}.

\subsection{Proof of Lemma~\ref{lemma:linsys_1}}
	\label{sec:linsys_1_pf}

For any state-action pair $(s,a) \in \mathcal{S}\times \mathcal{A}$, we observe that
	\begin{align}
	\soft{Q}^\star&(s,a)  - \soft{Q}^{(t+1)}(s,a) \notag \\
	&= r(\s, \ac) + \disct \exlimbig{\s' \sim \prob(\cdot|\s, \ac)}{\soft{\V}^\star(\s')} - \prnBig{r(\s, \ac) + \disct \exlimbig{\s' \sim \prob(\cdot|\s, \ac)}{\soft{\V}^{(t+1)}(\s')}}\notag \\
	& = \disct \exlim{\s' \sim \prob(\cdot|\s, \ac)}{\tau \log \prn{ \norm{\exp\prn{\frac{\soft{Q}^{\star}(\s', \cdot)}{\tau}}}_1}} - \disct \exlimBig{\substack{\s' \sim \prob(\cdot|\s, \ac),\\\ac' \sim \plcy^{(t+1)}(\cdot|\s')}}{ \soft{\Q}^{(t+1)}(\s', \ac') -\tau \log\plcy^{(t+1)}(\ac'|\s')}, \label{eq:linsys_1_intermediate}
	\end{align}
	where the first step invokes the definition \eqref{eq:defn-regularized-Q} of $\soft{\Q}$, and the second step is due to the expression \eqref{eq:opt_consist} of $\soft{\V}^\star$. To continue, recall that $\pi^{(t)}$ is related to $\xi^{(t)}$ as
\begin{align}
	\forall s\in \cS: \qquad \pi^{(t)}(\cdot | s) =  \frac{1}{\|\xi^{(t)} (s,\cdot)\|_1} \xi^{(t)} (s,\cdot)
\end{align}
which can be seen by comparing \eqref{defn:xi-t-sequence} with \eqref{eq:entropy_npg}. This in turn leads to
	\begin{align}
		\log { \plcy^{(t+1)}(a|s) } &= \log { \xi^{(t+1)}(\s, \ac) } - \log  \prnbig{ \normbig{\xi^{(t+1)}(\s, \cdot)}_1 } \notag\\
		&= \alpha  \log { \xi^{(t)}(\s, \ac) } + (1-\alpha)\frac{\soft{\Q}^{(t)}(\s,\ac)}{\tau} - \log \prnbig{ \normbig{\xi^{(t+1)}(\s, \cdot)}_1 },
	\label{eq:log_pi_2_log_xi}
	\end{align}
where the second line comes from \eqref{defn:xi-t-sequence-2}. By plugging  \eqref{eq:log_pi_2_log_xi} into \eqref{eq:linsys_1_intermediate} we obtain
	\begin{align}
\soft{Q}^\star(s,a) - \soft{Q}^{(t+1)}(s,a) 	
		& = \disct \exlim{\s' \sim \prob(\cdot|\s, \ac)}{\tau \log{ \prnbig{ \normbig{\exp\prn{{\soft{Q}^{\star}(\s', \cdot)}/{\tau}}}_1} }
		- \tau \log \prnbig{ \normbig{\xi^{(t+1)}(s',\cdot)}_1} } \notag \\
		&- \disct \mathop\mathbb{E}\limits_{\substack{\s' \sim \prob(\cdot|\s, \ac),\\\ac' \sim \plcy^{(t+1)}(\cdot|\s')}} \Bigg[ \soft{\Q}^{(t+1)}(\s', \ac') 
		- \tau  \underbrace{ \prnBigg{ \alpha  \log \xi^{(t)}(\s', \ac') + (1-\alpha)\frac{\soft{\Q}^{(t)}(\s',\ac')}{\tau} } }_{=\,\log \xi^{(t+1)}(s',a')} \Bigg]
		\label{eq:linsys_1_penguin}
\end{align}
for any $(s,a)\in \cS\times \cA$. 
In the sequel, we bound each term on the right-hand side of \eqref{eq:linsys_1_penguin} separately. 
\begin{itemize}
\item In view of the property \eqref{eq:bound_logexp_diff}, the first term on the right-hand side of \eqref{eq:linsys_1_penguin} can be bounded by 
\[  
			\tau \log{ \prnbig{ \normbig{\exp\prnbig{{\soft{Q}^{\star}(\s', \cdot)}/{\tau}}}_1} } -\tau \log \prnbig{ \normbig{\xi^{(t+1)}(s',\cdot)}_1} 
		\leq \normbig{\soft{Q}^\star - \tau \log \xi^{(t+1)}}_\infty .
\]
\item Regarding the second term, the monotonicity \eqref{eq:Q_mono} of the soft Q-function allows us to derive
\begin{align}
	&
	\soft{\Q}^{(t+1)}(\s, \ac) -\tau \prn{\alpha  \log { \xi^{(t)}(\s, \ac) } + (1-\alpha){\soft{\Q}^{(t)}(\s,\ac)}/{\tau} }
	\notag\\
	&\qquad \geq  
	\soft{\Q}^{(t)}(\s, \ac) -\tau \prn{\alpha  \log \xi^{(t)}(\s, \ac) + (1-\alpha){\soft{\Q}^{(t)}(\s,\ac)}/{\tau} }
	\notag\\	
	&\qquad =\, \alpha 
	\prn{ \soft{\Q}^{(t)}(\s, \ac) -\tau   \log { \xi^{(t)}(\s, \ac) }  
	}  
	\notag\\	
	&\qquad \overset{(\mathrm{i})}{=}\,  \alpha 
	\prn{ \alpha \prn{\soft{\Q}^{(t-1)}(s,a) - \tau \log \xi^{(t-1)}(s,a)} + \soft{\Q}^{(t)}(s,a) - \soft{\Q}^{(t-1)}(s,a) }  
	\notag\\
	&\qquad \overset{(\mathrm{ii})}{\geq}\,  \alpha^2 
	\prn{ \soft{\Q}^{(t-1)}(s,a) - \tau \log \xi^{(t-1)}(s,a)  }  
	\notag\\
	& \qquad \overset{(\mathrm{iii})}{\geq}\,   \alpha^{t+1} 
	\prn{ \soft{\Q}^{(0)}(s,a) - \tau \log \xi^{(0)}(s,a)  }  
	\notag\\
	& \qquad  \overset{(\mathrm{iv})}{\geq}\,  - \alpha^{t+1} \normbig{\soft{Q}^{(0)} - \tau \log \xi^{(0)}}_\infty  \notag
	\end{align}
for any $(s,a)\in \cS\times \cA$. 
		Here, (i) follows by construction \eqref{defn:xi-t-sequence-2}, (ii) invokes  the monotonicity property \eqref{eq:Q_mono} (so that $\soft{\Q}^{(t)}\geq \soft{\Q}^{(t-1)}$), and (iii) follows by repeating the arguments (i) and (ii) recursively.
 \end{itemize}
Combining the preceding two bounds with the expression \eqref{eq:linsys_1_penguin}, we conclude that
\begin{align}
	0\leq 
	\soft{Q}^\star(s,a) - \soft{Q}^{(t+1)}(s,a) 	
	\leq \gamma \normbig{\soft{Q}^\star - \tau \log \xi^{(t+1)}}_\infty + \gamma \alpha^{t+1} \normbig{\soft{Q}^{(0)} - \tau \log \xi^{(0)}}_\infty
\end{align}
for any $(s,a)\in \cS \times \cA$, thus concluding the proof.

\subsection{Proof of Lemma~\ref{lemma:inexact_bound}}
\label{sec:pflemdalaopo_pert}

Recall that, in this scenario, the policies are updated using inexact policy evaluation via \eqref{eq:soft_greedy_approximate}, namely,
\begin{align}
\label{eqn:inexact-update}
\forall (s,a)\in \mathcal{S} \times \cA, \qquad \plcy^{(t+1)}(a|\s) =  \frac{ \big( \plcy^{(t)}(a|\s) \big)^{ 1- \frac{\eta\tau}{1-\disct}}\exp\big( \frac{\eta}{1-\gamma} \estsoftQ(\s, a) \big) }{ \widehat{Z}^{(t)}(s) },
\end{align}
where $\widehat{Z}^{(t)}(\s) := \sum_{a'} \plcy^{(t)}(a'|\s)^{ 1- \frac{\eta\tau}{1-\disct}}\exp\big( \frac{\eta}{1-\gamma} \estsoftQ(\s, a')  \big)$.  
To facilitate analysis, we further introduce another auxiliary policy sequence $\{\breve{\plcy}^{(t)}\}$, which corresponds to the policy update as if we had access to exact soft Q-function of $ \plcy^{(t)}$ in the $t$-th iteration; this is defined as
\begin{align}
\label{eqn:inexact-update-auxiliary}
\forall (s,a)\in \mathcal{S} \times \cA,  \qquad \breve{\plcy}^{(t+1)}(a|\s) = \frac{ \big( \plcy^{(t)}(a|\s) \big)^{1- \frac{\eta\tau}{1-\disct}}\exp\big( \frac{\eta}{1-\gamma}   \soft{\Q}^{(t)}(\s, a)  \big) }{ Z^{(t)} (s) }  ,
\end{align}
where we abuse the notation by letting $ Z^{(t)} (s)  :=  \sum_{a'} \plcy^{(t)}(a'|\s)^{1- \frac{\eta\tau}{1-\disct}}\exp\big( \frac{\eta}{1-\gamma}  \soft{\Q}^{(t)}(\s, a') \big)$. 
It is worth emphasizing that $\breve{\plcy}^{(t+1)}$ is produced on the basis of ${\plcy}^{(t)}$ as opposed to $\breve{\plcy}^{(t)}$; it should be viewed as a one-step perfect update from a given policy ${\plcy}^{(t)}$.

We first make note of the following fact: for any step size $0<\eta \leq  (1-\disct)/\tau$, it follows from \eqref{eq:log_pi_diff} --- together with the construction \eqref{eqn:inexact-update} and \eqref{eqn:inexact-update-auxiliary} --- that
\begin{align}
& \normbig{\log \plcy^{(t+1)} - \log \exctplcy{t+1}}_\infty \notag \\
& \le 2 \norm{\log \left( \plcy^{(t)}(a|\s)^{1-\eta\tau/(1-\disct)}\exp\Big( \frac{\eta}{1-\gamma}{\estsoftQ(\s, a)} \Big) \right) - \log \left( \plcy^{(t)}(a|\s)^{1-\eta\tau/(1-\disct)}\exp\Big( \frac{\eta}{1-\gamma} {\soft{\Q}^{(t)}(\s, a)} \Big) \right) }_\infty  \notag\\
& =  \frac{2\eta}{1 - \disct}\normbig{\estsoftQ - \soft{Q}^{(t)}}_\infty.
\label{eqn:boya}
\end{align}

Next, let us recall the inequality~\eqref{eq:proof_another_fancy_middle_step} in the proof of Lemma~\ref{lemma:dalaopo} under exact policy evaluation $\exctplcy{t+1}(\cdot|\s)$; when applied to the current setting, it essentially indicates that 
\begin{align}
\soft{\V}^{(t)}(\s _0) 
& =\exlim{\substack{\ac_0 \sim \exctplcy{t+1}(\cdot | \s _0)}}{ \frac{1-\disct}{\eta} \log Z^{(t)}(\s_0)} - \frac{1-\disct}{\eta} \KL\prn{\plcy^{(t)}(\cdot |\s _0) \,\big\|\, \exctplcy{t+1}(\cdot | \s _0)} \notag \\
& = \exlim{\substack{\ac_0 \sim \plcy^{(t+1)}(\cdot | \s _0)}}{\frac{1-\disct}{\eta} \log Z^{(t)}(\s_0)} - \frac{1-\disct}{\eta} \KL\prn{\plcy^{(t)}(\cdot |\s _0) \,\big\|\, \exctplcy{t+1}(\cdot | \s _0)}, 
\label{eqn:dalaopo-v}
\end{align}
where the last step follows since the quantity $Z^{(t)}(s)$ does not depend on $a$ at all. 
In order to control the first term of \eqref{eqn:dalaopo-v}, we invoke the definition of $\breve{\plcy}^{(t+1)}(\cdot|\s)$ to show that
\begin{align}
&	\exlim{\substack{\ac_0 \sim \plcy^{(t+1)}(\cdot | \s _0)}}{\frac{1-\disct}{\eta} \log Z^{(t)}(\s_0)} \notag \\
& \overset{\mathrm{(i)}}{=} \exlim{\substack{\ac_0 \sim \plcy^{(t+1)}(\cdot | \s _0)}}{-\tau \log \exctplcy{t+1}(\ac_0|\s _0) + \soft{Q}^{(t)}(\s _0, \ac_0) + \prn{\tau - \frac{1-\disct}{\eta}}\prn{\log \exctplcy{t+1}(\ac_0|\s_0) - \log \plcy^{(t)}(\ac_0|\s_0)}}\notag\\
& = \exlim{\substack{\ac_0 \sim \plcy^{(t+1)}(\cdot | \s _0)}}{-\tau \log \plcy^{(t+1)}(\ac_0|\s _0) + \soft{Q}^{(t)}(\s _0, \ac_0)  } + \prn{\tau - \frac{1-\disct}{\eta}}\KL\prn{\plcy^{(t+1)}(\cdot |\s _0) \,\big\|\, \plcy^{(t)}(\cdot | \s _0)}\notag\\
&  \qquad\qquad	-{\frac{1-\disct}{\eta}}   \mathop\mathbb{E}\limits_{\substack{\ac_0 \sim \plcy^{(t+1)}(\cdot | \s _0)}} \left[\log \exctplcy{t+1}(\ac_0|\s_0) - \log \plcy^{(t)}(\ac_0|\s_0) \right]	\notag\\
& \le \exlim{\substack{\ac_0 \sim \plcy^{(t+1)}(\cdot | \s _0)}}{-\tau \log \plcy^{(t+1)}(\ac_0|\s _0) + \soft{Q}^{(t)}(\s _0, \ac_0)  } + \prn{\tau - \frac{1-\disct}{\eta}}\KL\prn{\plcy^{(t+1)}(\cdot |\s _0) \,\big\|\, \plcy^{(t)}(\cdot | \s _0)}\notag\\
&  \qquad\qquad	+2\normbig{\estsoftQ - \soft{Q}^{(t)}}_\infty,\label{eq:middle_step1}
\end{align}
where the final step results from \eqref{eqn:boya}.
Putting the above bound together with \eqref{eqn:dalaopo-v} guarantees that 
\begin{align}
\soft{\V}^{(t)}(\s_0) 
& \le \exlim{\substack{\ac_0 \sim \plcy^{(t+1)}(\cdot | \s _0)}}{-\tau \log \plcy^{(t+1)}(\ac_0|\s _0) + \soft{Q}^{(t)}(\s _0, \ac_0)}  - \frac{1-\disct}{\eta} \KL\prn{\plcy^{(t)}(\cdot |\s _0) \,\big\|\, \exctplcy{t+1}(\cdot | \s _0)} \notag\\
&\qquad\quad - \prn{\frac{1-\disct}{\eta} - \tau}\KL\prn{\plcy^{(t+1)}(\cdot |\s _0) \,\big\|\, \plcy^{(t)}(\cdot | \s _0)}+ 2\normbig{\estsoftQ - \soft{Q}^{(t)}}_\infty \notag\\
&\le \exlim{\substack{\ac_0 \sim \plcy^{(t+1)}(\cdot | \s _0)}}{-\tau \log \plcy^{(t+1)}(\ac_0|\s _0) + \soft{Q}^{(t)}(\s _0, \ac_0)} + 2\normbig{\estsoftQ - \soft{Q}^{(t)}}_\infty \notag\\
	&\le \exlim{\substack{\ac_0 \sim \plcy^{(t+1)}(\cdot | \s _0)}}{-\tau \log \plcy^{(t+1)}(\ac_0|\s _0) + r(s_0, a_0)  + \gamma \mathop\mathbb{E}\limits_{\s _1\sim \prob(\cdot | \s _0, \ac_0)} 
\big[\soft{\V}^{(t)}(\s_1)\big] } 
	+ 2\normbig{\estsoftQ - \soft{Q}^{(t)}}_\infty. \notag
\end{align}
where the last identity makes use of the relation $\soft{Q}^{(t)}(\s_0, \ac_0) = r(s_0, a_0) + \gamma \mathbb{E}_{\s _1\sim \prob(\cdot | \s _0, \ac_0)} 
\big[\soft{\V}^{(t)}(\s_1)\big]$.   Invoking the above inequality recursively as in the expression~\eqref{eqn:useful} (see Lemma~\ref{lemma:dalaopo}), we can expand it to establish 
\begin{align*}
\soft{\V}^{(t)}(\s _0)  
& \le \soft{\V}^{(t+1)}(\s _0) + 2\normbig{\estsoftQ - \soft{Q}^{(t)}}_\infty \sum_{i=0}^{\infty} \gamma^i = \soft{\V}^{(t+1)}(\s _0) + \frac{2}{{1-\disct}}\normbig{\estsoftQ - \soft{Q}^{(t)}}_\infty .
\end{align*} 

\subsection{Proof of Lemma~\ref{lemma:xiaoqie}}
\label{sec:pflemxiaoqie}

First of all, we follow the definition \eqref{defn:V-tau} of the entropy-regularized value function to deduce that
\begin{align}
& \soft{\V}^{\star}(\rho) - \soft{\V}^{(t)}(\rho) 
= \exlim{\substack{\s _0\sim \rho, a_i \sim \plcy^\star_\tau(\cdot| s_i),\\ \s _{i+1}\sim \prob(\cdot| s_i, a_i), \forall i\geq 0}}{\sum_{i=0}^{\infty}\disct^i \prn{r(s_i, a_i) - \tau \log \plcy^\star_\tau (a_i|s_i)}} - \soft{\V}^{(t)}(\rho)\notag\\
&\quad = \exlim{\substack{\s _0\sim \rho, a_i \sim \plcy^\star_\tau(\cdot| s_i),\\ \s _{i+1}\sim \prob(\cdot| s_i, a_i), \forall i\geq 0}}{\sum_{i=0}^{\infty} \disct^i \prn{r(s_i, a_i) - \tau \log \plcy^\star_\tau (a_i|s_i) + \soft{\V}^{(t)}(s_i) - \soft{\V}^{(t)}(s_i)}} - \soft{\V}^{(t)}(\rho)\notag\\
&\quad= \exlim{\substack{\s _0\sim \rho, a_i \sim \plcy^\star_\tau(\cdot| s_i),\\ \s _{i+1}\sim \prob(\cdot| s_i, a_i), \forall i\geq 0}}{\soft{\V}^{(t)}(s_{0}) + \sum_{i=0}^{\infty}\disct^i \prn{r(s_i, a_i) - \tau \log \plcy^\star_\tau (a_i|s_i) + \disct\soft{\V}^{(t)}(s_{i+1}) - \soft{\V}^{(t)}(s_i) }} - \soft{\V}^{(t)}(\rho) \notag\\
&\quad \overset{\mathrm{(i)}}{=} \exlim{\substack{\s _0\sim \rho, a_i \sim \plcy^\star_\tau(\cdot| s_i),\\ \s _{i+1}\sim \prob(\cdot| s_i, a_i), \forall i\geq 0}}{\sum_{i=0}^{\infty}\disct^i \prn{r(s_i, a_i) - \tau \log \plcy^\star_\tau (a_i|s_i) + \disct\soft{\V}^{(t)}(s_{i+1}) - \soft{\V}^{(t)}(s_i) }}\notag\\
	&\quad \overset{\mathrm{(ii)}}{=} \frac{1}{1-\disct}\exlim{\s\sim d_\rho^{\plcy^\star_\tau}}{\sum_{\ac} \plcy^\star_\tau(\ac|\s) \prn{r(\s, \ac) - \tau \log \plcy^\star_\tau (\ac|\s) + \disct    \exlim{s'\sim P(\cdot |s,a)}{\soft{\V}^{(t)}(\s') }  - \soft{\V}^{(t)}(\s)}}\notag\\
&\quad \overset{\mathrm{(iii)}}{=} \frac{1}{1-\disct}\exlim{\s\sim d_\rho^{\plcy^\star_\tau}}{\sum_{\ac}\plcy^\star_\tau(\ac|\s) \prn{\soft{\Q}^{(t)}(\s, \ac) - \tau \log \plcy^\star_\tau (\ac|\s)} - \soft{\V}^{(t)}(\s)}.
\label{eq:step_1}
\end{align}
Here, (i) is due to the definition $ \soft{\V}^{(t)}(\rho) = \mathbb{E}_{s_0\sim \rho}{\big[\soft{\V}^{(t)}(s_0)\big]}$, (ii) follows by aggregating terms corresponding to the same state-action pair and the definition of $ d_\rho^{\plcy^\star_\tau}$ (cf.~\eqref{eq:defn-d-s0}), whereas (iii) results from the definition \eqref{eq:defn-regularized-Q} of the regularized Q-function.

To continue, we shall attempt to control each part of \eqref{eq:step_1} separately. 
To begin with, observe that  the first part of \eqref{eq:step_1} can be bounded by Jensen's inequality,	namely, 
\begin{align}
	\sum_{\ac}\plcy^\star_\tau(\ac|\s) \prn{\soft{\Q}^{(t)}(\s, \ac) - \tau \log \plcy^\star_\tau (\ac|\s)} 
	&= \tau\sum_{a}\pi_{\tau}^{\star}(a|s)\log \left( \frac{\exp\big(Q_{\tau}^{(t)}(s,a)/\tau\big)}{\pi_{\tau}^{\star}(a|s)} \right) \notag\\
	&\leq \tau\log\left(\sum_{a}\pi_{\tau}^{\star}(a|s)\frac{\exp\big(Q_{\tau}^{(t)}(s,a)/\tau\big)}{\pi_{\tau}^{\star}(a|s)}\right) \notag\\
	&= \tau\log \left( \sum_{\ac}\exp\prn{ \soft{Q}^{(t)}(\s, \ac)/\tau} \right).
\label{eq:bound1_1}
\end{align}
With regards to the second part of \eqref{eq:step_1}, it is seen from the definition of ${\plcy}^{(t+1)}$ (cf.~\eqref{eq:SPI-update}) that
\begin{align}
	Q_{\tau}^{(t)}(s,a)=\tau \log\pi_{\tau}^{(t+1)}(a|s)+\tau\log  \left( \sum_{\ac}\exp\prn{ \soft{Q}^{(t)}(\s, \ac)/\tau} \right),
	\label{eq:Qt-decomposition-123}
\end{align}
thus allowing one to derive
\begin{align}
	V_{\tau}^{(t)}(s) & =\sum_{a}\pi_{\tau}^{(t)}(a|s)\left(Q_{\tau}^{(t)}(s,a)-\tau\log\pi_{\tau}^{(t)}(a|s)\right) \notag\\
	& \overset{(\mathrm{i})}{=} \tau \sum_{a}\pi_{\tau}^{(t)}(a|s)\left\{  \log\pi_{\tau}^{(t+1)}(a|s)+ \log  \left( \sum_{\ac}\exp\prn{ \soft{Q}^{(t)}(\s, \ac)/\tau} \right)- \log\pi_{\tau}^{(t)}(a|s)\right\} \notag\\
		& =\tau \log  \left( \sum_{\ac}\exp\prn{ \soft{Q}^{(t)}(\s, \ac)/\tau} \right)+\tau\sum_{a}\pi_{\tau}^{(t)}(a|s)\left( \log\pi_{\tau}^{(t+1)}(a|s) - \log\pi_{\tau}^{(t)}(a|s)\right) \notag\\
		& =\tau \log  \left( \sum_{\ac}\exp\prn{ \soft{Q}^{(t)}(\s, \ac)/\tau} \right)- \tau \mathsf{KL}\Big(\pi_{\tau}^{(t)}(a|s)\,\big\|\,\pi_{\tau}^{(t+1)}(\cdot|s)\Big) ,
\label{eq:bound1_2}
\end{align}		
where (i) relies on the identity \eqref{eq:Qt-decomposition-123}.
Substituting the inequalities \eqref{eq:bound1_1} and \eqref{eq:bound1_2} into the expression \eqref{eq:step_1}, we can demonstrate with a little algebra that
\begin{align*}
	& \soft{\V}^{\star}(\rho) - \soft{\V}^{(t)}(\rho) \le \frac{1}{\eta} \exlim{\s\sim d_\rho^{\plcy^\star_\tau}}{   \KL\Big(\plcy^{(t)}(\cdot|\s) \,\big\|\, {\plcy}^{(t+1)}(\cdot|\s)\Big)}.
\end{align*}
%


\subsection{Proof of Lemma \ref{lemma:derivative-grad}}
\label{sec:proof-lemma:derivative-grad}

The results of this lemma, or some similar versions, have appeared in prior work (e.g.~\citet[Lemma 10]{mei2020global} and \citet[Lemma 5.6]{agarwal2019optimality}). 
We include the proof here primarily for the sake of self-completeness.

\paragraph{Proof of Eqn.~\eqref{eqn:derivative}.}

	The policy gradient of the unregularized value function $V^{\pi_\theta}(s_0)$ is well-known as the policy gradient theorem~\citep{sutton2000policy}. Here, we deal with a slightly different variant -- an entropy-regularized value function $V^{\pi_\theta}_\tau(s_0)$ in the expression~\eqref{defn:value-function} with the softmax policy parameterization in \eqref{eq:definition_softmax}.  
	Invoking the Bellman equation and recognizing that $V^{\pi_\theta}_\tau(s_0)$ can be viewed as an unregularized value function with instantaneous rewards $r(s,a)-\tau \log \pi_{\theta}(a|s)$ for any $(s,a)$, we obtain
\begin{align*}
	\nabla_{\theta} V^{\pi_\theta}_\tau(s_0) &= \nabla_{\theta} \left[
		\sum_{a_0} \pi_{\theta}(a_0|s_0) \Big(r(s_0, a_0) -\tau\log \plcy_\prm(a_0|s_0) + \disct \mathop\mathbb{E}\limits_{\s' \sim \prob(\cdot|s_0, a_0)}\big[ \soft{\V}^{\plcy_\prm}(\s') \big] \Big)
	\right] \\
	& \overset{(\mathrm{i})}{=} \nabla_{\theta} \left[\sum_{a_0} \pi_{\theta}(a_0|s_0) \Big(Q^{\pi_\theta}_\tau(s_0,a_0)-\tau\log \plcy_\prm(a_0|s_0)\Big) \right] \\
	&= \sum_{a_0}  \big(\nabla_{\theta} \pi_{\theta}(a_0|s_0)\big) \Big(Q^{\pi_\theta}_\tau(s_0,a_0)-\tau\log \plcy_\prm(a_0|s_0)\Big) + 
	\sum_{a_0}  \pi_{\theta}(a_0|s_0) \nabla_{\theta} \Big( Q^{\pi_\theta}_\tau(s_0,a_0)-\tau\log \plcy_\prm(a_0|s_0) \Big) \\
	& \overset{(\mathrm{ii})}{=} \sum_{a_0}  \Big( \pi_{\theta}(a_0|s_0) \nabla_\theta \log \pi_{\theta}(a_0|s_0) \Big) \Big(Q^{\pi_\theta}_\tau(s_0,a_0)-\tau\log \plcy_\prm(a_0|s_0)\Big) \\
	& \hspace{1.5cm} + \sum_{a_0}  \pi_{\theta}(a_0|s_0) \nabla_{\theta} 
	\Big(r(s_0,a_0) + \gamma \sum_{s_1} \prob(s_1|s_0, a_0){\soft{\V}^{\plcy_\prm}(s_1)} 
	-\tau\log \plcy_\prm(a_0|s_0)\Big),
\end{align*}
where (i) relies on the definition \eqref{eq:defn-regularized-Q} of $Q^{\pi_\theta}_\tau$, and (ii) makes use of the identity 
$$\nabla_{\theta} \pi_{\theta}(a_0|s_0)= \pi_{\theta}(a_0|s_0) \nabla_\theta \log \pi_{\theta}(a_0|s_0)$$ as well as the definition  \eqref{eq:defn-regularized-Q}  of $Q^{\pi_\theta}_\tau$. Given that
\begin{align}
	\label{eq:sum-pi-1}
	\sum_{a_0}  \pi_{\theta}(a_0|s_0) \nabla_{\theta} \log \plcy_\prm(a_0|s_0) = \sum_{a_0} \nabla_{\theta}   \pi_{\theta}(a_0|s_0) =\nabla_{\theta} \Big(\sum_{a_0} \pi_{\theta}(a_0|s_0) \Big) = \nabla_{\theta} 1 = 0
\end{align}
and that $r(s,a)$ is independent of $\theta$, 
one can continue the above derivative to reach
\begin{align*}
	\nabla_{\theta} V^{\pi_\theta}_\tau(s_0) 
	&= \sum_{a_0}  \Big( \pi_{\theta}(a_0|s_0)  \nabla_\theta  \log \pi_{\theta}(a_0|s_0) \Big) \Big(Q^{\pi_\theta}_\tau(s_0,a_0)-\tau\log \plcy_\prm(a_0|s_0)\Big) \\
	& \hspace{3cm} + \gamma \sum_{a_0}  \pi_{\theta}(a_0|s_0)  \sum_{s_1} \prob(s_1|s_0, a_0)\nabla_{\theta}{\soft{\V}^{\plcy_\prm}(s_1)} \\
%
	&= \mathop\mathbb{E}\limits_{\substack{\ac_i \sim \plcy_{\theta}(\cdot | \s _i),\\\s _{i+1}\sim \prob(\cdot | \s _i, \ac_i), \forall i\geq 0}} 
	 \Big[ \big( \nabla_{\theta} \log \pi_{\theta}(a_0|s_0)\big) \Big(Q^{\pi_\theta}_\tau(s_0,a_0)-\tau\log \plcy_\prm(a_0|s_0)\Big) + \gamma \nabla_{\theta}{\soft{\V}^{\plcy_\prm}(s_1)}\Big].
\end{align*}
Repeating the above calculations recursively, we arrive at 
\begin{align}
	\nabla_\theta V^{\pi_\theta}_\tau(s_0) 
%
	&= \mathop\mathbb{E}\limits_{\substack{\ac_i \sim \plcy_{\theta}(\cdot | \s _i),\\\s _{i+1}\sim \prob(\cdot | \s _i, \ac_i), \forall i\geq 0}} 
	\left[\sum_{t=0}^{\infty} \gamma^t \big( \nabla_{\theta} \log \pi_{\theta}(a_t|s_t) \big) \Big(Q^{\pi_\theta}_\tau(s_t,a_t)-\tau\log \plcy_\prm(a_t|s_t)\Big)\right] \notag\\
	&= \frac{1}{1-\gamma} \mathop\Exs\limits_{s\sim d^{\pi_\theta}_{s_0}} \mathop\Exs\limits_{ a\sim \pi_{\theta}(\cdot|s)}
	\Big[ \big( \nabla_{\theta} \log \pi_{\theta}(a|s)\big)  \Big( Q^{\pi_\theta}_\tau(s,a)-\tau\log \plcy_\prm(a|s) \Big) \Big] \notag\\
	&= \frac{1}{1-\gamma} \mathop\Exs\limits_{s\sim d^{\pi_\theta}_{s_0}} \mathop\Exs\limits_{ a\sim \pi_{\theta}(\cdot|s)}
	\Big[ \big( \nabla_{\theta} \log \pi_{\theta}(a|s)\big)  \Big( A^{\pi_\theta}_\tau(s,a) + V^{\pi_\theta}_\tau(s) \Big) \Big] \notag\\
	&= \frac{1}{1-\gamma} \mathop\Exs\limits_{s\sim d^{\pi_\theta}_{s_0}} \mathop\Exs\limits_{ a\sim \pi_{\theta}(\cdot|s)}
	\Big[ \big( \nabla_{\theta} \log \pi_{\theta}(a|s)\big)  A^{\pi_\theta}_\tau(s,a)  \Big] ,
	\label{eq:V-identity-123}
\end{align}
	where the second line follows by aggregating the terms corresponding to the same state-action pair, and  the third line invokes the definition \eqref{eqn:asofttau} of $A^{\pi_\theta}_\tau$. To see why the last line holds, invoke \eqref{eq:sum-pi-1} to reach
\begin{align*}
	\mathbb{E}_{a\sim \pi_{\theta}(\cdot|s)} \Big[ \soft{\V}^{\plcy_\prm}(s)\nabla_{\theta} \log \pi_{\theta}(a|s) \Big]
	&=
	\sum_{a} \soft{\V}^{\plcy_\prm}(s)\pi_{\theta}(a|s)\nabla_{\theta} \log \pi_{\theta}(a|s) \\
	&=
	\soft{\V}^{\plcy_\prm}(s) \sum_{a} \pi_{\theta}(a|s)\nabla_{\theta} \log \pi_{\theta}(a|s) = 0.
\end{align*}

%

Further, it is easily seen that under the softmax parametrization in \eqref{eq:definition_softmax},
\begin{align}
	\frac{\partial \log \pi_{\theta}(a'|s')}{\partial \theta(s,a)} = \ind[s'=s] \big(\ind[a'=a] - \pi_\theta(a|s)\big)
	\label{eq:derivative-log-pi}
\end{align}
for any $(s,a),(s',a')\in \cS\times \cA$. 
Combining with~\eqref{eq:V-identity-123}, it further implies that
\begin{align*}
	\frac{\partial V^{\pi_\theta}_\tau(s_0)}{\partial \theta(s,a)} 
	&=
	\frac{1}{1-\gamma} \Exs_{s' \sim d^{\pi_\theta}_{s_0}} \mathbb{E}_{a' \sim \pi_{\theta}(\cdot|s')} \Bigg[  \frac{\partial \log \pi_{\theta}(a'|s')}{\partial \theta(s,a)} \soft{\A}^{\plcy_\prm}(\s',\ac') \Bigg] \\
& = \frac{1}{1-\gamma} \Exs_{s' \sim d^{\pi_\theta}_{s_0}} \mathbb{E}_{a' \sim \pi_{\theta}(\cdot|s')} \Bigg[ \Big( \ind[s'=s] \big(\ind[a'=a] - \pi_\theta(a|s)\big)\Big) \soft{\A}^{\plcy_\prm}(\s',\ac') \Bigg] \\
	&\overset{(\mathrm{i})}{=} \frac{1}{1-\gamma} \Exs_{s'\sim d^{\pi_\theta}_{s_0}} \mathbb{E}_{a'\sim \pi_{\theta}(\cdot|s')}
	\Big[\ind\big[(s',a') = (s,a)\big]\, \soft{\A}^{\plcy_\prm}(\s',\ac') \Big] \\
	&= \frac{1}{1-\gamma}d_{s_0}^{\pi_\theta}(s)\pi_{\theta}(a|s)\soft{\A}^{\plcy_\prm}(\s,\ac).
\end{align*}
where $(\mathrm{i})$ follows from $ \mathbb{E}_{a' \sim \pi_{\theta}(\cdot|s')} \soft{\A}^{\plcy_\prm}(\s',\ac') = \sum_{a'}  \pi_\theta(a'|s') \soft{\A}^{\plcy_\prm}(\s',\ac')=0$ due to the definition \eqref{eqn:asofttau}. The proof regarding $V^{\pi_\theta}_\tau(\rho)$ can be obtained by averaging  the initial state $s_0$ over the distribution $\rho$.

%

\paragraph{Proof of Eqn.~\eqref{eq:search-direction-NPG}.}

In order to establish  \eqref{eq:search-direction-NPG}, a crucial observation is that $w_{\theta} :=  \big(\mathcal{F}_\rho^{\prm}\big)^\dagger \nabla_{\theta} {\V}_{\tau}^{\pi_{\theta}} (\rho)$
is exactly the solution to the following least-squares problem
\begin{align}
	\mathrm{minimize}_{w\in \mathbb{R}^{|\cS||\cA|}} ~~ \big\|\mathcal{F}_\rho^{\prm} w - \nabla_{\theta} {\V}_{\tau}^{\pi_{\theta}} (\rho) \big\|_2^2.
	\label{eq:least-squares-F}
\end{align}
%
%
From the definition \eqref{eq:defn-fisher-information} of the Fisher information matrix,  we have
\begin{align*}
\mathcal{F}_{\rho}^{\theta}w & = \Exs_{s \sim d^{\pi_\theta}_{\rho}} \mathbb{E}_{a \sim \pi_{\theta}(\cdot|s)} \left[\big(\nabla_{\theta}\log\pi_{\theta}(a|s)\big)\big(\nabla_{\theta}\log\pi_{\theta}(a|s)\big)^{\top}w\right].
\end{align*}
for any fixed vector $w=[w_{s,a}]_{(s,a)\in \cS\times \cA}$. 
As a result, for any $(s,a)\in \cS\times \cA$ one has
\begin{align*}
\left(\mathcal{F}_{\rho}^{\theta}w\right)_{s,a} & =\Exs_{s' \sim d^{\pi_\theta}_{\rho}} \mathbb{E}_{a' \sim \pi_{\theta}(\cdot|s')}  \Bigg[\frac{\partial\log\pi_{\theta}(a'|s')}{\partial\theta(s,a)}
	\Bigg(\sum_{\tilde{s},\tilde{a}}\frac{\partial\log\pi_{\theta}(a'|s')}{\partial\theta(\tilde{s},\tilde{a})}w_{\tilde{s},\tilde{a}}\Bigg)\Bigg]\\
	& \overset{(\mathrm{i})}{=}\Exs_{s' \sim d^{\pi_\theta}_{\rho}} \mathbb{E}_{a' \sim \pi_{\theta}(\cdot|s')}  \Bigg[\ind[s'=s]\Big(\ind[a'=a]-\pi_{\theta}(a|s)\Big)\Bigg(\sum_{\tilde{s},\tilde{a}}\ind[\tilde{s}=s']\Big(\ind[\tilde{a}=a']-\pi_{\theta}(\tilde{a}|\tilde{s})\Big)w_{\tilde{s},\tilde{a}}\Bigg)\Bigg]\\
 & =\Exs_{s' \sim d^{\pi_\theta}_{\rho}} \mathbb{E}_{a' \sim \pi_{\theta}(\cdot|s')} \left[\ind[s'=s]\Big( \ind[a'=a]-\pi_{\theta}(a|s)\Big)\Big(w_{s',a'}-\sum_{\tilde{a}}\pi_{\theta}(\tilde{a}|s')w_{s',\tilde{a}}\Big)\right]\\
 & =d_{\rho}^{\pi_{\theta}}(s) \mathbb{E}_{a' \sim \pi_{\theta}(\cdot|s')} \left[\Big(\ind[a'=a]-\pi_{\theta}(a|s)\Big)\big(w_{s,a'}-c(s)\big)\right]\\
 & =d_{\rho}^{\pi_{\theta}}(s) \mathbb{E}_{a' \sim \pi_{\theta}(\cdot|s')}\Big[\ind[a'=a]w_{s,a'}-\pi_{\theta}(a|s)w_{s,a'}-\ind[a'=a]c(s)+\pi_{\theta}(a|s)c(s)\Big]\\
 & =d_{\rho}^{\pi_{\theta}}(s)\Big[\pi_{\theta}(a|s)w_{s,a}-\pi_{\theta}(a|s)c(s)-\pi_{\theta}(a|s)c(s)+\pi_{\theta}(a|s)c(s)\Big]\\
	& =d_{\rho}^{\pi_{\theta}}(s)\pi_{\theta}(a|s)\big[w_{s,a}-c(s)\big],
\end{align*}
where $(\mathrm{i})$ makes use of the derivative calculation (\ref{eq:derivative-log-pi}), and we define $c(s):=\sum_{a}\pi_{\theta}(a|s)w_{s,a}$. Consequently, the objective function of \eqref{eq:least-squares-F} can be written as
\begin{align*}
\big\|\mathcal{F}_{\rho}^{\prm}w-\nabla_{\theta}{\V}_{\tau}^{\pi_{\theta}}(\rho)\big\|_{2}^{2} & =\sum_{s,a}\left(d_{\rho}^{\pi_{\theta}}(s)\pi_{\theta}(a|s)\left[w_{s,a}-c(s)\right]-\frac{1}{1-\gamma}d_{\rho}^{\pi_{\theta}}(s)\pi_{\theta}(a|s)A_{\tau}^{\pi_{\theta}}(s,a)\right)^{2}\\
 & =\sum_{s,a}\left(d_{\rho}^{\pi_{\theta}}(s)\pi_{\theta}(a|s) \left( w_{s,a}- c(s)-\frac{1}{1-\gamma} A_{\tau}^{\pi_{\theta}}(s,a) \right) \right)^{2} ,
\end{align*}
which is  minimized by choosing $w_{s,a}=\frac{1}{1-\gamma}A_{\tau}^{\pi_{\theta}}(s,a) + c(s)$
for all $(s,a)\in\cS\times\cA$. This concludes the proof.

	\section{Convergence guarantees for CPI-style policy updates} \label{sec:cpi}

Employing the SPI update as the improved policy, we arrive at the following CPI-style update
\begin{subequations}
	\label{eq:CPI_spi}
	\begin{equation}	\label{eq:CPI_update}
		\plcy^{(t+1)} = (1-\beta) \plcy^{(t)} + \beta \overline{\plcy}^{(t+1)}.
	\end{equation}
	Here,  $ \overline{\plcy}^{(t+1)}$ corresponds to a one-step SPI update from $\plcy^{(t)}$, namely,
	\begin{equation}
		\forall (s,a)\in \mathcal{S} \times \cA,  \qquad \overline{\plcy}^{(t+1)}(\ac|\s) =\frac{1}{\overline{Z}^{(t)}(\s)} \exp\prn{\soft{\Q}^{(t)}(\s, \ac)/\tau} ,
	\end{equation}
\end{subequations}
where we denote 
\[
	\overline{Z}^{(t)}(\s) = \sum_{\ac \in \cA} \exp\prnbig{\soft{\Q}^{(t)}(\s, \ac)/\tau} \qquad \text{and} \qquad \soft{\Q}^{(t)}=\soft{\Q}^{\pi^{(t)}}
\]
as usual.
Here, $\beta \in (0,1]$ is a parameter that controls the ``conservatism'' of the updates. We  characterize the convergence rate of this update rule \eqref{eq:CPI_spi} in the following theorem.
\begin{theorem}[Linear convergence of CPI-style updates]
	\label{cor:cpi}
	For any $0 < \beta  \le 1$, the update rule \eqref{eq:CPI_spi} satisfies
	\begin{align}
		{\soft{V}^\star(\rho) - \soft{V}^{(t)}(\rho)} & \le \norm{\frac{\rho}{\soft{\mu}^\star}}_\infty  \prnbig{1 - \beta(1-\gamma)}^{t} \prn{\soft{V}^\star(\soft{\mu}^\star) - \soft{V}^{(0)}(\soft{\mu}^\star)}, \qquad \forall t\geq 0,
	\end{align}
	where $\soft{\mu}^\star$ is the stationary distribution defined in \eqref{eq:stationary_dist_disc}. 
\end{theorem} 
According to Theorem~\ref{cor:cpi}, it takes the CPI-style policy update \eqref{eq:CPI_spi} at most
$$\frac{1}{\beta(1-\gamma)} \log \prn{\norm{\frac{\rho}{\soft{\mu}^\star}}_\infty \frac{\soft{V}^\star(\soft{\mu}^\star) - \soft{V}^{(0)}(\soft{\mu}^\star)}{\epsilon}}$$
iterations to reach $\soft{V}^\star(\rho) - \soft{V}^{(t)}(\rho) \leq \epsilon$. 
As it turns out, the CPI-style update rule can be analyzed using our framework through the following performance improvement lemma, which is an adaptation of Lemma~\ref{lemma:dalaopo}. 
	In what follows, we use $\overline{\Q}_\tau^{(t+1)}$ and $\overline{\V}_\tau^{(t+1)}$ to abbreviate $\soft{\Q}^{ \overline{\plcy}^{(t+1)}}$ and $\soft{\V}^{ \overline{\plcy}^{(t+1)}}$, respectively.
	\begin{lemma}[Performance improvement of CPI-style updates] \label{lemma:cpi}
	Consider the policy update rule	\eqref{eq:CPI_update} with any $\beta \in (0,1]$. For any distribution $\rho$, one has  
		\begin{align*}
		\soft{\V}^{(t+1)}(\rho) -  \soft{\V}^{(t)}(\rho) 
		\ge &  \frac{\beta \tau}{1-\disct} \exlim{\s \sim d_{\rho}^{(t+1)}}{ \KL\prn{\plcy^{(t)}(\cdot | \s)\,\|\,\overline{\plcy}^{(t+1)}(\cdot | \s)}}.
		\end{align*}
	\end{lemma}
	\begin{proof} See Appendix~\ref{proof:lemma_cpi}. \end{proof}
	Combining the above result with Lemma \ref{lemma:xiaoqie} and following a similar approach to \eqref{eq:core} give
	\begin{align}
		\soft{\V}^{\star}(\rho) - \soft{\V}^{(t+1)}(\rho) &= \soft{\V}^{\star}(\rho) - \soft{\V}^{(t)}(\rho) + \Big( \soft{\V}^{(t)}(\rho) - \soft{\V}^{(t+1)}(\rho) \Big) \notag\\
		& \overset{\mathrm{(i)}}{\le} \soft{\V}^{\star}(\rho) - \soft{\V}^{(t)}(\rho)  - \frac{\beta\tau}{1-\disct} \exlim{\s \sim d_{\rho}^{(t+1)}}{  \KL\prn{\plcy^{(t)}(\cdot |\s) \,\big\|\, \overline{\plcy}^{(t+1)}(\cdot | \s)}}\notag\\
		& \overset{\mathrm{(ii)}}{\leq} \soft{\V}^{\star}(\rho) - \soft{\V}^{(t)}(\rho)  -  \frac{\beta\tau}{1-\disct} \norm{\frac{d_\rho^{\plcy^\star_\tau}}{d_{\rho}^{(t+1)}}}^{-1}_\infty \exlim{\s \sim d_{\rho}^{\plcy^\star_\tau}}{ {\KL \Big( \plcy^{(t)}(\cdot|\s) \,\big\|\, \overline{\plcy}^{(t+1)}(\cdot|\s) \Big)}}  \notag\\
		& \overset{\mathrm{(iii)}}{\le}   \soft{\V}^{\star}(\rho) - \soft{\V}^{(t)}(\rho) - \beta\norm{\frac{d_\rho^{\plcy^\star_\tau}}{d_{\rho}^{(t+1)}}}^{-1}_\infty \Big( \soft{\V}^{\star}(\rho) - \soft{\V}^{(t)}(\rho) \Big) \notag\\
		& =   \left(1 - \beta\norm{\frac{d_\rho^{\plcy^\star_\tau}}{d_{\rho}^{(t+1)}}}^{-1}_\infty  \right) \prn{\soft{\V}^{\star}(\rho) - \soft{\V}^{(t)}(\rho)} .
	\end{align}
	Here, (i) arises from Lemma~\ref{lemma:cpi},  (ii) employs the pre-factor $\big\| {d_\rho^{\plcy^\star_\tau}}/{d_{\rho}^{(t+1)}} \big\|^{-1}_\infty $ to accommodate the change of distributions, 
	whereas (iii) follows from Lemma \ref{lemma:xiaoqie} and the constraint that $0\leq \eta\leq \frac{1-\gamma}{\tau}$. By taking $\rho$ to be the stationary distribution $\soft{\mu}^\star$ (cf.~\eqref{eq:stationary_dist_disc}), one has
		\begin{align*}
		\soft{\V}^{\star}(\soft{\mu}^\star) - \soft{\V}^{(t+1)}(\soft{\mu}^\star) &\le   \left(1 - \beta\norm{\frac{d_{\soft{\mu}^\star}^{\plcy^\star_\tau}}{d_{\soft{\mu}^\star}^{(t+1)}}}^{-1}_\infty  \right) \prn{\soft{\V}^{\star}(\soft{\mu}^\star) - \soft{\V}^{(t)}(\soft{\mu}^\star)} \\
		&\le \left(1 - \beta\norm{\frac{{\soft{\mu}^\star}}{(1-\disct){\soft{\mu}^\star}}}^{-1}_\infty  \right) \prn{\soft{\V}^{\star}(\soft{\mu}^\star) - \soft{\V}^{(t)}(\soft{\mu}^\star)}\\
		&= \left(1 - \beta(1-\disct) \right) \prn{\soft{\V}^{\star}(\soft{\mu}^\star) - \soft{\V}^{(t)}(\soft{\mu}^\star)},
	\end{align*}
	where we have used $d_{\soft{\mu}^\star}^{\plcy^\star_\tau} = \soft{\mu}^\star$ (cf.~\eqref{eq:stationary_dist_disc}) and $d_{\soft{\mu}^\star}^{(t+1)} \ge (1-\disct)\soft{\mu}^\star$ in the second step. This immediately concludes the proof.

\subsection{Proof of Lemma~\ref{lemma:cpi}} \label{proof:lemma_cpi}

First of all, we claim that
		\begin{align} \label{eq:key_cpi_relation}
		\soft{\V}^{(t+1)}(\rho) -  \soft{\V}^{(t)}(\rho) 
		=&  \frac{\tau}{1-\disct} \exlim{\s \sim d_{\rho}^{(t+1)}}{ \KL\prn{\plcy^{(t)}(\cdot | \s)\,\|\,\overline{\plcy}^{(t+1)}(\cdot | \s)} - \KL\prn{\plcy^{(t+1)}(\cdot | \s )\,\|\,\overline{\plcy}^{(t+1)}(\cdot | \s)}},
		\end{align}
which we shall establish momentarily. Since the KL divergence $\KL\prn{\plcy(\cdot | \s )\,\|\,\overline{\plcy}^{(t+1)}(\cdot | \s)}$ is convex in $\plcy(\cdot|\s)$ \citep{cover1999elements}, the update rule \eqref{eq:CPI_update} together with Jensen's inequality necessarily implies that
		\begin{align*}
			\KL\prn{\plcy^{(t+1)}(\cdot | \s )\,\|\,\overline{\plcy}^{(t+1)}(\cdot | \s)} 
			&\le \beta \KL\prn{\overline{\plcy}^{(t+1)}(\cdot | \s)\,\|\,\overline{\plcy}^{(t+1)}(\cdot | \s)} 
			+ (1-\beta)\KL\prn{\plcy^{(t)}(\cdot | \s)\,\|\,\overline{\plcy}^{(t+1)}(\cdot | \s)} \\
			&= (1-\beta)\KL\prn{\plcy^{(t)}(\cdot | \s)\,\|\,\overline{\plcy}^{(t+1)}(\cdot | \s)}.
		\end{align*}
		Substituting the above inequality into \eqref{eq:key_cpi_relation} allows us to conclude that
		\begin{align*}
		\soft{\V}^{(t+1)}(\rho) -  \soft{\V}^{(t)}(\rho) 
		\ge &  \frac{\beta \tau}{1-\disct} \exlim{\s \sim d_{\rho}^{(t+1)}}{ \KL\prn{\plcy^{(t)}(\cdot | \s)\,\|\,\overline{\plcy}^{(t+1)}(\cdot | \s)}}.
		\end{align*}

The rest of this proof is then dedicated to establishing the claim \eqref{eq:key_cpi_relation}, which is similar to the proof of Lemma \ref{lemma:dalaopo}. To begin with, we express $\soft{\V}^{(t)}(\s _0) $ as follows
		\begin{align*}
		\soft{\V}^{(t)}(\s _0) 
		& = \exlim{{\ac_0 \sim \plcy^{(t)}(\cdot | \s _0)}}{-\tau \log \plcy^{(t)}(\ac_0|\s _0) + \soft{Q}^{(t)}(\s _0, \ac_0)} \notag\\
		& = \exlim{{\ac_0 \sim \plcy^{(t)}(\cdot | \s _0)}}{-\tau \log \plcy^{(t)}(\ac_0|\s _0) + \tau \log \overline{\plcy}^{(t+1)}(\ac_0|\s_0)} + \tau\log \overline{Z}^{(t)}(\s_0) \notag\\
		& =  \tau \log \overline{Z}^{(t)}(\s_0) -\tau \KL\prn{\plcy^{(t)}(\cdot | \s _0)\,\|\,\overline{\plcy}^{(t+1)}(\cdot | \s _0)} \\
		& =  \tau \exlim{{\ac_0 \sim \plcy^{(t+1)}(\cdot | \s _0)}}{\log \overline{Z}^{(t)}(\s_0)} -\tau \KL\prn{\plcy^{(t)}(\cdot | \s _0)\,\|\,\overline{\plcy}^{(t+1)}(\cdot | \s _0)},
		\end{align*}
	where the first line makes use of the definitions \eqref{defn:V-tau} and \eqref{eq:defn-regularized-Q}, the second line follows from \eqref{eq:CPI_spi}, the third line uses the definition of the KL divergence, and the last line follows since $ \overline{Z}^{(t)}(\s_0)$ does not depend on $a$. To continue, we subtract and add $\tau \KL\prn{\plcy^{(t+1)}(\cdot | \s _0)\,\|\,\overline{\plcy}^{(t+1)}(\cdot | \s _0)}$ to obtain 
		\begin{align*}
	\soft{\V}^{(t)}(\s _0) 	& = \exlim{{\ac_0 \sim \plcy^{(t+1)}(\cdot | \s _0)}}{\tau\log \overline{Z}^{(t)}(\s_0)-\tau \log \plcy^{(t+1)}(\ac_0|\s _0) + \tau \log \overline{\plcy}^{(t+1)}(\ac_0|\s_0)}   \\
		&\qquad\qquad\quad + \tau \KL\prn{\plcy^{(t+1)}(\cdot | \s _0)\,\|\,\overline{\plcy}^{(t+1)}(\cdot | \s _0)}-\tau \KL\prn{\plcy^{(t)}(\cdot | \s _0)\,\|\,\overline{\plcy}^{(t+1)}(\cdot | \s _0)} \notag\\
		&= \exlim{{\ac_0 \sim \plcy^{(t+1)}(\cdot | \s _0)}}{-\tau \log \plcy^{(t+1)}(\ac_0|\s _0) + \soft{Q}^{(t)}(\s _0, \ac_0)}  + \tau \KL\prn{\plcy^{(t+1)}(\cdot | \s _0)\,\|\,\overline{\plcy}^{(t+1)}(\cdot | \s _0)} \\
		&\qquad\qquad\quad -\tau \KL\prn{\plcy^{(t)}(\cdot | \s _0)\,\|\,\overline{\plcy}^{(t+1)}(\cdot | \s _0)} \notag\\
		&=  \soft{\V}^{(t+1)}(\s _0)  + \frac{\tau}{1-\disct} \exlim{\s \sim d_{\s_0}^{(t+1)}}{\KL\prn{\plcy^{(t+1)}(\cdot | \s )\,\|\,\overline{\plcy}^{(t+1)}(\cdot | \s)} - \KL\prn{\plcy^{(t)}(\cdot | \s)\,\|\,\overline{\plcy}^{(t+1)}(\cdot | \s)}}\notag.
		\end{align*}
	Here, the first step relies on the definition of KL divergence, the second step comes from \eqref{eq:CPI_spi}, while the last step is obtained by using the relation $\soft{Q}^{(t)}(\s_0, \ac_0) = r(s_0, a_0) + \gamma \mathbb{E}_{\s _1\sim \prob(\cdot | \s _0, \ac_0)} 
	\big[\soft{\V}^{(t)}(\s_1)\big]$ and then invoking the above equality recursively as in the expression~\eqref{eqn:useful} (see Lemma~\ref{lemma:dalaopo}). Averaging the equality over the initial state distribution $\s_0 \sim \rho$ thus establishes the claim \eqref{eq:key_cpi_relation}.

	\section{Proof for approximate entropy-regularized NPG (Theorem~\ref{thm:npg_inexact})} \label{sec:proof_approximate_npg}

In this section, we complete the proofs of  Theorem~\ref{thm:npg_inexact} in Section~\ref{sec:Pf-inexact}, which consists of (i) establishing the linear system in \eqref{eq:inexact_ls} and (ii) extracting the convergence rate from \eqref{eq:inexact_ls}.

\paragraph{Step 1: establishing the linear system \eqref{eq:inexact_ls}.} In what follows, we shall justify the linear system relation by checking each row separately.

\bigskip

\noindent {(1)} {\bf Bounding $\| \soft{Q}^\star - \tau \log \widehat{\xi}^{(t+1)} \|_{\infty}$.} From the construction \eqref{defn:hat-xi-t-sequence-2} of $\widehat{\xi}^{(t+1)}$,  we have
\begin{align*}
	\soft{Q}^\star - \tau \log \widehat{\xi}^{(t+1)} = \alpha \prnbig{\soft{Q}^\star - \tau \log \widehat{\xi}^{(t)}} + (1-\alpha) \prnbig{\soft{Q}^\star - \soft{Q}^{(t)}} + (1-\alpha)\prnbig{\soft{Q}^{(t)} - \estsoftQ}.
\end{align*}
Taken together with the triangle inequality and  the assumption $\big\|\soft{Q}^{(t)} - \estsoftQ\big\|_\infty\leq \delta$, this gives
\begin{equation} \label{eq:inexact_ls_1}
	\normbig{\soft{Q}^\star - \tau \log \widehat{\xi}^{(t+1)}}_\infty \le \alpha \normbig{\soft{Q}^\star - \tau \log \widehat{\xi}^{(t)}}_\infty + (1-\alpha) \normbig{\soft{Q}^\star - \soft{Q}^{(t)}}_\infty + \prn{1-\alpha}\delta. 
\end{equation}	

\medskip
\noindent {(2)} {\bf Bounding $-\min_{s,a}  \prnbig{  \soft{\Q}^{(t+1)}(\s, \ac) - \tau \log \widehat{\xi}^{(t+1)}(\s, \ac) }$.} Invoking the definition \eqref{defn:hat-xi-t-sequence-2} of $\widehat{\xi}^{(t+1)}$ again implies that for any $(s,a)\in\mathcal{S}\times\mathcal{A}$,
\begin{align*}
 - &\prn{ \soft{\Q}^{(t+1)}(\s, \ac) -\tau \log \widehat{\xi}^{(t+1)}(\s, \ac)} \\
&=  -\prn{ \soft{\Q}^{(t+1)}(\s, \ac) -\tau \prn{\alpha  \log \widehat{\xi}^{(t)}(\s, \ac) + (1-\alpha) {\estsoftQ(\s,\ac)}/{\tau} }} \\
& = - \alpha\prn{ \soft{\Q}^{(t)}(\s, \ac) - \tau   \log \widehat{\xi}^{(t)}(\s, \ac)} + \prn{1-\alpha}\prn{\estsoftQ(\s,\ac) - \soft{\Q}^{(t)}(\s, \ac)}  + \prn{ \soft{\Q}^{(t)}(\s, \ac) -  \soft{\Q}^{(t+1)}(\s, \ac)} \\
&\le - \alpha\prn{ \soft{\Q}^{(t)}(\s, \ac) - \tau   \log \widehat{\xi}^{(t)}(\s, \ac)} + \prn{1-\alpha}\delta + \frac{2\gamma \delta}{1 - \disct} ,
\end{align*}
where the last inequality follows from $\big\|\soft{Q}^{(t)} - \estsoftQ\big\|_\infty\leq \delta$ and \eqref{eq:almost_mono_Q}. Taking the maximum over $(s,a)\in \cS\times \cA$ on both sides  and using the definition $\alpha = 1 - \frac{\eta \tau}{1-\gamma}$ yield
\begin{align} 
\label{eq:inexact_ls_2}
-\min_{s,a}  & \prn{  \soft{\Q}^{(t+1)}(\s, \ac) - \tau \log \widehat{\xi}^{(t+1)}(\s, \ac) } 
\leq  - \alpha\min_{s,a}  \prn{\soft{Q}^{(t)}(s,a) - \tau\log \widehat{\xi}^{(t)}(s,a)}  + \prn{1-\alpha}\delta \prn{1 + \frac{2\disct }{\eta\tau}} . 
\end{align}

\medskip
\noindent {(3)} {\bf Bounding $\big\|\soft{Q}^\star - \soft{Q}^{(t+1)} \big\|_{\infty}$.} Following the same arguments as for \eqref{eq:linsys_1_penguin}, we obtain
\begin{align*}
	\soft{Q}^\star(s,a) - \soft{Q}^{(t+1)}(s,a) &= \disct \exlim{\s' \sim \prob(\cdot|\s, \ac)}{\tau \log\prnbig{ \normbig{\exp\prn{{\soft{Q}^{\star}(\s', \cdot)}/{\tau}}}_1} -\tau \log \prnbig{ \normbig{\widehat{\xi}^{(t+1)}(s',\cdot)}_1}} \notag \\ 
& \qquad\qquad - \disct \exlim{\substack{\s' \sim \prob(\cdot|\s, \ac),\\\ac' \sim \plcy^{(t+1)}(\cdot|\s')}}{ \soft{\Q}^{(t+1)}(\s', \ac') -\tau \log \widehat{\xi}^{(t+1)}(\s', \ac')} \\
& \leq \disct \normbig{\soft{Q}^\star - \tau \log \widehat{\xi}^{(t+1)}}_\infty - \gamma \min_{s,a}  \prn{  \soft{\Q}^{(t+1)}(\s, \ac) - \tau \log \widehat{\xi}^{(t+1)}(\s, \ac) } ,
\end{align*}
where the last line follows from \eqref{eq:bound_logexp_diff}. By plugging \eqref{eq:inexact_ls_1} and \eqref{eq:inexact_ls_2} into the above inequality, we arrive at the claimed bound regarding this term.


\paragraph{Step 2: deducing convergence guarantees from the linear system \eqref{eq:inexact_ls}.} We start by pinning down the eigenvalues and eigenvectors of the matrix $B$. Specifically, the three eigenvalues can be calculated as
\begin{equation} \label{eq:eigenvalues_B}
\lambda_1 = \alpha + \disct(1-\alpha) = 1-\eta\tau,\qquad \lambda_2 = \alpha \qquad\mbox{and} \qquad \lambda_3 = 0,
\end{equation} 
whose corresponding eigenvectors are given respectively by
\begin{equation}\label{eq:eigenvectors_B}
v_1 = \begin{bmatrix}
\gamma \\
1 \\
0
\end{bmatrix}, \qquad v_2 = \begin{bmatrix}
0 \\
-1 \\
1
\end{bmatrix}, \qquad \mbox{and}\qquad v_3 = \begin{bmatrix}
\alpha \\
\alpha - 1 \\
0
\end{bmatrix}.
\end{equation}
With some elementary computation, one can show that $z_0$ and $b$ introduced in \eqref{eq:expressions_inexact_ls} can be related to the eigenvectors of $B$ in the following way:
\begin{align}
z_0 & \leq \begin{bmatrix}
\normbig{\soft{Q}^\star - \soft{Q}^{(0)}}_\infty \\[1ex]
\normbig{\soft{Q}^\star - \tau \log \widehat{\xi}^{(0)}}_\infty \\[1ex]
\normbig{\soft{Q}^{(0)}- \tau\log \widehat{\xi}^{(0)}}_\infty
\end{bmatrix} \notag\\
& = \frac{1}{1-\eta\tau}\brk{(1-\alpha)\normbig{\soft{Q}^\star - \soft{Q}^{(0)}}_\infty + \alpha \prn{\normbig{\soft{Q}^\star - \tau \log \widehat{\xi}^{(0)}}_\infty +\normbig{\soft{Q}^{(0)}- \tau\log \widehat{\xi}^{(0)}}_\infty} } v_1 \notag\\
&\quad\quad\quad + \normbig{\soft{Q}^{(0)}- \tau\log \widehat{\xi}^{(0)}}_\infty v_2 + c_{z}  v_3 \notag\\
& \le \frac{1}{1-\eta\tau}\prn{\normbig{\soft{Q}^\star - \soft{Q}^{(0)}}_\infty + 2\alpha \tau \normbig{\log \soft{\plcy}^\star - \log \plcy^{(0)}}_\infty  } v_1 + \normbig{\soft{Q}^{(0)}- \tau\log \widehat{\xi}^{(0)}}_\infty  v_2 + c_z  v_3,
\label{eq:z-expression-decomposition}
\end{align}
where $c_z$ is some scalar whose value is immaterial since the eigenvalue corresponding to $v_3$ is $\lambda_3=0$, and the last line follows from the same reasoning for \eqref{eq:initial_reform}. Another userful identity is:
\begin{align}
b= (1-\alpha)\delta\begin{bmatrix}
\disct\prn{2+\frac{2\disct}{\eta\tau}} \\
1 \\
1 + \frac{2\disct }{\eta\tau}
\end{bmatrix}  = (1-\alpha)\delta\brk{\prn{2+\frac{2\disct }{\eta\tau}} v_1 + \prn{1 + \frac{2\disct }{\eta\tau}}  v_2}.
	\label{eq:b-expression-decomposition}
\end{align}
  
With these preparations in place, we can now invoke the recursion relationship \eqref{eq:inexact_ls} and the non-negativity of $B$ to obtain
\begin{align*}
z_{t+1} & \le B^{t+1} z_0 + \sum_{s=0}^{t} B^{t-s} b \\
& \le B^{t+1}\brk{\frac{1}{1-\eta\tau}\prn{\normbig{\soft{Q}^\star - \soft{Q}^{(0)}}_\infty + 2\alpha \tau \normbig{\log \soft{\plcy}^\star - \log \plcy^{(0)}}_\infty  } v_1 + \normbig{\soft{Q}^{(0)}- \tau\log \widehat{\xi}^{(0)}}_\infty v_2 + c_z v_3}\\
& \qquad\qquad + (1-\alpha)\delta\sum_{s=0}^{t} B^{t-s} \brk{\prn{2+\frac{2\disct }{\eta\tau}} v_1 + \prn{1 + \frac{2\disct }{\eta\tau}}  v_2}\\
&= \brk{ \lambda_1^t \prn{\normbig{\soft{Q}^\star - \soft{Q}^{(0)}}_\infty + 2\alpha \tau \normbig{\log \soft{\plcy}^\star - \log \plcy^{(0)}}_\infty  } + (1-\alpha)\delta\prn{2+\frac{2\disct }{\eta\tau}}\frac{1-\lambda_1^{t+1}}{1-\lambda_1} }  v_1 \\
&\qquad  + \brk{\lambda_2^{t+1} \normbig{\soft{Q}^{(0)}- \tau\log \widehat{\xi}^{(0)}}_\infty + (1-\alpha)\delta\prn{1 + \frac{2\disct }{\eta\tau}}\frac{1-\lambda_2^{t+1}}{1 - \lambda_2} }  v_2,
\end{align*}
where the eigenvalues and eigenvectors of $B$ are given in \eqref{eq:eigenvalues_B} and \eqref{eq:eigenvectors_B}, respectively, and the second inequality relies on \eqref{eq:z-expression-decomposition} and \eqref{eq:b-expression-decomposition}.
Note that we are only interested in the first two entries of the vector $z_t$. Since the first two entries of the eigenvector $v_2$ are non-positive, we can safely drop the term involving $v_2$ in the above inequality to obtain
\begin{align}
&\begin{bmatrix}
\normbig{\soft{Q}^\star - \soft{Q}^{(t+1)}}_\infty \\[1ex]
\normbig{\soft{Q}^\star - \tau \log \widehat{\xi}^{(t+1)}}_\infty
\end{bmatrix}    \notag\\
	& \le  \Bigg\{ \lambda_1^t \prn{\normbig{\soft{Q}^\star - \soft{Q}^{(0)}}_\infty + 2\alpha \tau \normbig{\log \soft{\plcy}^\star - \log \plcy^{(0)}}_\infty  } + (1-\alpha)\delta\prn{2+\frac{2\disct }{\eta\tau}}\frac{1-\lambda_1^{t+1}}{1-\lambda_1} \Bigg\}   \begin{bmatrix}
\disct\\
1
\end{bmatrix}  \notag\\
&  \le
	\Bigg\{ \prn{1-\eta\tau}^t \prn{\normbig{\soft{Q}^\star - \soft{Q}^{(0)}}_\infty + 2\prn{1-\frac{\eta\tau}{1-\disct}} \tau \normbig{\log \soft{\plcy}^\star - \log \plcy^{(0)}}_\infty  } 
	+ \frac{2\delta}{1-\disct}\prn{1+\frac{\disct }{\eta\tau}} \Bigg\}  \begin{bmatrix}
\disct\\
1
\end{bmatrix} .
	\label{eq:upper-bound-1234}
\end{align}

When it comes to the log policies, we recall again the fact that $\pi^{(t)}$ is related to $\widehat{\xi}^{(t)}$ as
\begin{align}
	\forall s\in \cS: \qquad \pi^{(t)}(\cdot | s) =  \frac{1}{\normbig{\widehat{\xi}^{(t)} (s,\cdot)}_1} \widehat{\xi}^{(t)} (s,\cdot). 
\end{align}
Invoking the elementary property \eqref{eq:log_pi_diff}, we reach
$$\normbig{\log \soft{\plcy}^\star - \log \plcy^{(t+1)}}_\infty \le 2\normbig{\soft{\Q}^\star/\tau - \log \widehat{\xi}^{(t+1)}}_\infty.$$ 
This together with the bound on $\normbig{\soft{Q}^\star - \tau \log \widehat{\xi}^{(t+1)}}_\infty$ in \eqref{eq:upper-bound-1234} establishes our claim for $\normbig{\log \soft{\plcy}^\star - \log \plcy^{(t+1)}}_\infty$.

	\section{Proof for local quadratic convergence (Theorem~\ref{thm:spi_sp})}	
\label{sec:superlinear}

	Assuming that the policy $\pi^{(t)}$ obeys Condition \eqref{eqn:pit-condition}, 
	we can control the difference of the corresponding discounted state visitation probabilities in terms of the sub-optimality gap w.r.t.~the log policy. 
	This is stated in the following lemma, whose proof is deferred to Section~\ref{sec:proof_lemma_drho_diff}. 
	%
		\begin{lemma} 
			\label{lemma:d_rho-difference-bound} 
		Consider any policy $\plcy$ satisfying $\norm{\log \plcy - \log \soft{\plcy}^\star}_\infty \le   1$.  It follows that
		\begin{align*}
			\norm{1 - \frac{d_\rho^{\soft{\plcy}^\star}}{d_\rho^\plcy}}_\infty \le 2 \prn{ \frac{1}{1-\disct} \norm{\frac{d_\rho^{\soft{\plcy}^\star}}{\rho}}_\infty  - 1} 
			\big\| \log \plcy - \log \soft{\plcy}^\star \big\|_\infty.
		\end{align*}
		In particular, by taking $\rho = \soft{\mu}^\star$ one has
		\begin{align*}
		   \Bigg\| 1 - \frac{\soft{\mu}^\star }{d_{\soft{\mu}^\star}^\plcy} \Bigg\|_\infty = \Bigg\|1 - \frac{d_{\soft{\mu}^\star}^{\soft{\plcy}^\star}}{d_{\soft{\mu}^\star}^\plcy} \Bigg\|_\infty 
			\le \frac{2\disct}{1-\disct} \big\| \log \plcy - \log \soft{\plcy}^\star \big\|_\infty.
		\end{align*}
	\end{lemma}

First, by virtue of the SPI update rule \eqref{eq:SPI-update} and the inequality \eqref{eq:log_pi_diff}, it is guaranteed that
	\begin{align}
		\normbig{\log \soft{\plcy}^\star - \log \plcy^{(t+1)} }_\infty 
		& \leq \frac{2}{\tau} \normbig{\soft{Q}^\star - \soft{Q}^{(t)} }_\infty \leq \frac{2\gamma}{\tau} \normbig{\soft{V}^\star - \soft{V}^{(t)} }_\infty 
		 \leq  \frac{2\gamma}{\tau} \norm{\frac{1}{\soft{\mu}^\star}}_\infty \prn{\soft{V}^\star(\soft{\mu}^\star) - \soft{V}^{(t)}(\soft{\mu}^\star)},
		\label{eq:log-policy-UB2}
	\end{align}
where the last inequality comes from a change of distributions argument. Armed with Lemma~\ref{lemma:d_rho-difference-bound} and the inequality~\eqref{eq:log-policy-UB2}, we arrive at 
	\begin{align}
		\Bigg\| 1 - \frac{d_{\soft{\mu}^\star}^{\soft{\plcy}^\star}}{d_{\soft{\mu}^\star}^{(t+1)}} \Bigg\|_\infty  
		\le  \frac{2\disct}{1-\disct} \big\| \log \plcy^{(t+1)} - \log \soft{\plcy}^\star \big\|_\infty
		\le \frac{4\disct^2}{(1 - \disct)\tau} \norm{\frac{1}{\soft{\mu}^\star}}_\infty \prn{\soft{V}^\star(\soft{\mu}^\star) - \soft{V}^{(t)}(\soft{\mu}^\star)}.
		\label{eq:diff-d-UB1}
	\end{align}
	%
Substitution into \eqref{eq:core} gives
	\begin{align*}
		\soft{\V}^{\star}({\soft{\mu}^\star}) - \soft{\V}^{(t+1)}({\soft{\mu}^\star}) \le & \left( 1 - \norm{\frac{d_{\soft{\mu}^\star}^{\plcy^\star_\tau}}{d_{{\soft{\mu}^\star}}^{(t+1)}}}^{-1}_\infty\right) \prn{\soft{\V}^{\star}({\soft{\mu}^\star}) - \soft{\V}^{(t)}({\soft{\mu}^\star})} \\
		= & \left( 1 - {\norm{1 + \frac{d_{\soft{\mu}^\star}^{\plcy^\star_\tau} - d_{{\soft{\mu}^\star}}^{(t+1)}}{d_{{\soft{\mu}^\star}}^{(t+1)}}}_\infty^{-1}}\right) \prn{\soft{\V}^{\star}({\soft{\mu}^\star}) - \soft{\V}^{(t)}({\soft{\mu}^\star})} \\
		\le & \left( 1 - \frac{1}{1 + \frac{4\disct^2}{(1 - \disct)\tau} \norm{\frac{1}{{\soft{\mu}^\star}}}_\infty  \prn{\soft{V}^\star({\soft{\mu}^\star}) - \soft{V}^{(t)}({\soft{\mu}^\star})}}\right) \prn{\soft{\V}^{\star}({\soft{\mu}^\star}) - \soft{\V}^{(t)}({\soft{\mu}^\star})} \\
	= & \frac{\frac{4\disct^2}{(1 - \disct)\tau} \norm{\frac{1}{{\soft{\mu}^\star}}}_\infty  \prn{\soft{V}^\star({\soft{\mu}^\star}) - \soft{V}^{(t)}({\soft{\mu}^\star})}^2}{1 + \frac{4\disct^2}{(1 - \disct)\tau} \norm{\frac{1}{{\soft{\mu}^\star}}}_\infty  \prn{\soft{V}^\star({\soft{\mu}^\star}) - \soft{V}^{(t)}({\soft{\mu}^\star})}} 
	\le \frac{4\disct^2}{(1 - \disct)\tau} \norm{\frac{1}{{\soft{\mu}^\star}}}_\infty \prn{\soft{V}^\star({\soft{\mu}^\star}) - \soft{V}^{(t)}({\soft{\mu}^\star})}^2, 
	\end{align*}
	where the second inequality makes use of the bound \eqref{eq:diff-d-UB1}.
	This in turn reveals that
	\begin{align*}
		\frac{4\disct^2}{(1 - \disct)\tau} \norm{\frac{1}{\soft{\mu}^\star}}_\infty \prn{\soft{V}^\star(\soft{\mu}^\star) - \soft{V}^{(t+1)}(\soft{\mu}^\star)} 
		\le \prn{\frac{4\disct^2}{(1 - \disct)\tau} \norm{\frac{1}{\soft{\mu}^\star}}_\infty  \prn{\soft{V}^\star(\soft{\mu}^\star) - \soft{V}^{(t)}(\soft{\mu}^\star)}}^2,
	\end{align*}
	which leads to our claimed result by a standard change of distributions.

	%
	


\subsection{Proof of Lemma~\ref{lemma:d_rho-difference-bound} } 
\label{sec:proof_lemma_drho_diff}

	For any policy $\plcy$, denote by 
	 $\sP_\plcy \in \real^{|\ssp|\times|\ssp|}$ the state transition matrix induced by $\plcy$ as follows
	 \begin{align}
		 \forall s,s'\in \cS:\qquad [\sP_\plcy]_{\s, \s '} := \mathbb{E}_{\ac\sim \plcy(\cdot | \s)} \big[ \prob(\s '|\s, \ac) \big].
		 \label{eq:defn-P-pi}
	\end{align}
	For any policy $\pi$ satisfying $\norm{\log \plcy - \log \soft{\plcy}^\star}_\infty \le  1$ , we develop an upper bound on $\abs{\brk{{\sP_\plcy -  \sP_{\soft{\plcy}^\star}}}_{\s, \s'}}$ as follows
	\begin{align*}
		\abs{\brk{{\sP_\plcy -  \sP_{\soft{\plcy}^\star}}}_{\s, \s'}} = \abs{\sum_{\ac} \prob(\s'|\s,\ac) \big( \plcy(\ac|\s) - \soft{\plcy}^\star(\ac|\s) \big) } 
	& \le\sum_{\ac} \prob(\s'|\s,\ac)\soft{\plcy}^\star(\ac|\s) \abs{{\frac{\plcy(\ac|\s)}{\soft{\plcy}^\star(\ac|\s)} - 1}} \\
	& \overset{\mathrm{(i)}}{\le} (e-1)\sum_{\ac} \prob(\s'|\s,\ac)\soft{\plcy}^\star(\ac|\s)  \Big| \log{\plcy(\ac|\s)} - \log{\soft{\plcy}^\star(\ac|\s)} \Big| \\
	& \le \norm{\log \plcy - \log \soft{\plcy}^\star}_\infty (e-1) \sum_{\ac} \prob(\s'|\s,\ac)\soft{\plcy}^\star(\ac|\s) \\
	& \le 2\brk{\sP_{\soft{\plcy}^\star}}_{\s, \s'}\norm{\log \plcy - \log \soft{\plcy}^\star}_\infty  ,
	\end{align*}
	%
	where (i) uses the assumption $\| \log \pi^{\star} - \log \pi\|_{\infty} \leq 1$
	together with the elementary inequality $\abs{x} \le (e-1) \abs{\log(1+x)}$ when $-1 < x \le e-1$. With the preceding bound in mind, we can demonstrate that
	\begin{align}
		\Big| \big( d_\rho^{\soft{\plcy}^\star} \big)^\top \prn{\sP_\plcy -  \sP_{\soft{\plcy}^\star}} \Big| 
		\leq 2\norm{\log \plcy - \log \soft{\plcy}^\star}_\infty \big( d_\rho^{\soft{\plcy}^\star} \big)^\top {\sP_{\soft{\plcy}^\star}}.
		\label{eq:drho-UB123}
	\end{align}
%
%
	Here and throughout, we overload the notation $|z|$ for any vector $z\in \mathbb{R}^{|\cS|}$ to denote $[|z_i|]_{1\leq i\leq |\cS|}$.  

	In addition, the definitions of ${d_\rho^\plcy}$ and ${d_\rho^{\soft{\plcy}^{\star}}}$ admit the following matrix-vector representation:
	\begin{align}
		\big( d_\rho^\plcy \big)^\top =& (1-\disct)\rho^\top \prn{\idn - \disct \sP_{\plcy}}^{-1}, \label{eq:d-rho-matrix}\\
		\big( d_\rho^{\soft{\plcy}^\star} \big)^\top =& (1-\disct) \rho^\top \prn{\idn - \disct \sP_{\plcy^\star}}^{-1},
	\end{align}
	thus allowing one to derive
	\begin{align*}
		\big(d_\rho^\plcy - d_\rho^{\soft{\plcy}^\star}\big)^\top =& (1-\disct)\rho^\top \prn{\idn - \disct \sP_{\plcy^\star}}^{-1} \big[ \prn{\idn - \disct \sP_{\plcy^\star}} - \prn{\idn - \disct \sP_{\plcy}} \big] \prn{\idn - \disct \sP_{\plcy}}^{-1} \\
	=&  \disct \big( d_\rho^{\soft{\plcy}^\star} \big)^\top \prn{\sP_\plcy -  \sP_{\soft{\plcy}^\star}}( \idn -\disct\sP_{\plcy})^{-1} .
	\end{align*}
	This together with the non-negativity of the matrix $( \idn -\disct\sP_{\plcy})^{-1}$ \citep[Lemma 7]{li2020sample}  enables the following bound
	\begin{align}
		\abs{ \big(d_\rho^\plcy - d_\rho^{\soft{\plcy}^\star}\big)^\top } 
		&\leq \gamma \abs{\big( d_\rho^{\soft{\plcy}^\star} \big)^\top  \prn{\sP_\plcy -  \sP_{\soft{\plcy}^\star}}}(\idn-\disct\sP_{\plcy})^{-1} \notag \\
		&\leq  2 \big\| \log \plcy - \log \soft{\plcy}^\star \big\|_\infty  \disct \big( d_\rho^{\soft{\plcy}^\star} \big)^\top {\sP_{\soft{\plcy}^\star}}(\idn-\disct\sP_{\plcy})^{-1},
		\label{eq:d-rho-124}
	\end{align}
	where the last inequality results from \eqref{eq:drho-UB123}.

	Furthermore, we make the observation that
	\begin{align*}
		\gamma\big(d_{\rho}^{\soft{\plcy}^{\star}}\big)^{\top}\sP_{\soft{\plcy}^{\star}}  &=(1-\gamma)\gamma\rho^{\top}\big(I-\gamma P_{\soft{\plcy}^{\star}}\big)^{-1}\sP_{\soft{\plcy}^{\star}}
  =(1-\gamma)\gamma\rho^{\top}\left[ \sum_{i=0}^{\infty}\big(\gamma P_{\soft{\plcy}^{\star}}\big)^{i}\right] P_{\soft{\plcy}^{\star}}\\
 & =(1-\gamma)\rho^{\top}\left[ \sum_{i=1}^{\infty}\big(\gamma P_{\soft{\plcy}^{\star}}\big)^{i}\right] \\
 & =(1-\gamma)\rho^{\top}\left[ \big(I-\gamma P_{\soft{\plcy}^{\star}}\big)^{-1}-I\right] =\big(d_{\rho}^{\soft{\plcy}^{\star}}\big)^{\top}-(1-\gamma)\rho^{\top} 
 \leq \prn{\norm{\frac{d_\rho^{\soft{\plcy}^\star}}{\rho}}_\infty  - (1-\disct)} \rho^\top,
\end{align*}
%
where the last line comes from a change of distributions argument. 
Combining this bound with \eqref{eq:d-rho-124} gives
	\begin{align*}
		\abs{\big(d_\rho^\plcy - d_\rho^{\soft{\plcy}^\star}\big)^\top} & \leq  2\big\|\log \plcy - \log \soft{\plcy}^\star \big\|_\infty  \prn{ \frac{1}{1-\disct} \norm{\frac{d_\rho^{\soft{\plcy}^\star}}{\rho}}_\infty  - 1} (1-\disct)\rho^\top \prn{\idn - \disct \sP_{\plcy}}^{-1}\\
		&= 2\big\|\log \plcy - \log \soft{\plcy}^\star \big\|_\infty  \prn{\frac{1}{1-\disct}  \norm{\frac{d_\rho^{\soft{\plcy}^\star}}{\rho}}_\infty  - 1} \big( d_\rho^\plcy \big)^\top,
	\end{align*}
	where the last line arises from the expression \eqref{eq:d-rho-matrix}. 
	As a result, we establish the claimed bound
	\[
	\norm{1 - \frac{d_\rho^{\soft{\plcy}^\star}}{d_\rho^\plcy}}_\infty \le 2 \big\| \log \plcy - \log \soft{\plcy}^\star \big\|_\infty  \prn{ \frac{1}{1-\disct} \norm{\frac{d_\rho^{\soft{\plcy}^\star}}{\rho}}_\infty  - 1}.
	\]

\end{document}